\PassOptionsToPackage{unicode}{hyperref}
\PassOptionsToPackage{hyphens}{url}
\PassOptionsToPackage{dvipsnames,svgnames,x11names}{xcolor}
\documentclass[12pt]{article}

\RequirePackage{amsthm,amsmath,amsfonts,amssymb}
\RequirePackage[authoryear]{natbib}
\RequirePackage[colorlinks,citecolor=blue,urlcolor=blue]{hyperref}
\RequirePackage{graphicx} 

\usepackage{setspace}
\usepackage{iftex}
\ifPDFTeX
  \usepackage[T1]{fontenc}
  \usepackage[utf8]{inputenc}
  \usepackage{textcomp} 
\else 
  \usepackage{unicode-math}
  \defaultfontfeatures{Scale=MatchLowercase}
  \defaultfontfeatures[\rmfamily]{Ligatures=TeX,Scale=1}
\fi
\usepackage{lmodern}
\ifPDFTeX\else  
\fi
\IfFileExists{upquote.sty}{\usepackage{upquote}}{}
\IfFileExists{microtype.sty}{
  \usepackage[]{microtype}
  \UseMicrotypeSet[protrusion]{basicmath} 
}{}
\makeatletter
\@ifundefined{KOMAClassName}{
  \IfFileExists{parskip.sty}{%
    \usepackage{parskip}
  }{
    \setlength{\parindent}{0pt}
    \setlength{\parskip}{6pt plus 2pt minus 1pt}}
}{
  \KOMAoptions{parskip=half}}
\makeatother
\usepackage{xcolor}
\makeatletter
\ifx\paragraph\undefined\else
  \let\oldparagraph\paragraph
  \renewcommand{\paragraph}{
    \@ifstar
      \xxxParagraphStar
      \xxxParagraphNoStar
  }
  \newcommand{\xxxParagraphStar}[1]{\oldparagraph*{#1}\mbox{}}
  \newcommand{\xxxParagraphNoStar}[1]{\oldparagraph{#1}\mbox{}}
\fi
\ifx\subparagraph\undefined\else
  \let\oldsubparagraph\subparagraph
  \renewcommand{\subparagraph}{
    \@ifstar
      \xxxSubParagraphStar
      \xxxSubParagraphNoStar
  }
  \newcommand{\xxxSubParagraphStar}[1]{\oldsubparagraph*{#1}\mbox{}}
  \newcommand{\xxxSubParagraphNoStar}[1]{\oldsubparagraph{#1}\mbox{}}
\fi
\makeatother

\usepackage{longtable,booktabs,array}
\usepackage{calc} 
\usepackage{etoolbox}
\makeatletter
\patchcmd\longtable{\par}{\if@noskipsec\mbox{}\fi\par}{}{}
\makeatother
\IfFileExists{footnotehyper.sty}{\usepackage{footnotehyper}}{\usepackage{footnote}}
\makesavenoteenv{longtable}
\usepackage{graphicx}
\makeatletter
\def\maxwidth{\ifdim\Gin@nat@width>\linewidth\linewidth\else\Gin@nat@width\fi}
\def\maxheight{\ifdim\Gin@nat@height>\textheight\textheight\else\Gin@nat@height\fi}
\makeatother
\setkeys{Gin}{width=\maxwidth,height=\maxheight,keepaspectratio}
\makeatletter
\def\fps@figure{htbp}
\makeatother

\makeatletter
\@ifpackageloaded{caption}{}{\usepackage{caption}}
\AtBeginDocument{%
\ifdefined\contentsname
  \renewcommand*\contentsname{Table of contents}
\else
  \newcommand\contentsname{Table of contents}
\fi
\ifdefined\listfigurename
  \renewcommand*\listfigurename{List of Figures}
\else
  \newcommand\listfigurename{List of Figures}
\fi
\ifdefined\listtablename
  \renewcommand*\listtablename{List of Tables}
\else
  \newcommand\listtablename{List of Tables}
\fi
\ifdefined\figurename
  \renewcommand*\figurename{Figure}
\else
  \newcommand\figurename{Figure}
\fi
\ifdefined\tablename
  \renewcommand*\tablename{Table}
\else
  \newcommand\tablename{Table}
\fi
}
\@ifpackageloaded{float}{}{\usepackage{float}}
\floatstyle{ruled}
\@ifundefined{c@chapter}{\newfloat{codelisting}{h}{lop}}{\newfloat{codelisting}{h}{lop}[chapter]}
\floatname{codelisting}{Listing}

\makeatother
\makeatletter
\@ifpackageloaded{caption}{}{\usepackage{caption}}
\@ifpackageloaded{subcaption}{}{\usepackage{subcaption}}
\makeatother

\ifLuaTeX
  \usepackage{selnolig}  
\fi
\usepackage[]{natbib}
\usepackage{bookmark}

\IfFileExists{xurl.sty}{\usepackage{xurl}}{} 
\hypersetup{
  pdftitle={Title},
  pdfauthor={Author 1; Author 2},
  pdfkeywords={3 to 6 keywords, that do not appear in the title},
  colorlinks=true,
  linkcolor={blue},
  filecolor={Maroon},
  citecolor={Blue},
  urlcolor={Blue},
  pdfcreator={LaTeX via pandoc}}

\newcommand{\anon}{1}

\usepackage{url}
\numberwithin{equation}{section}

\usepackage{cleveref}
\usepackage{thmtools,thm-restate}

\newtheorem{lemma}{Lemma}
\newcommand{\EE}{\mathbb{E}}
\def\bzeta{\mbox{\boldmath${\zeta}$}}
\def\bbeta{\mbox{\boldmath${\beta}$}}
\def\balpha{\mbox{\boldmath${\alpha}$}}
\def\btheta{\mbox{\boldmath${\theta}$}}
\def\bX{{\bf X}}
\def\bZ{{\bf Z}}
\def\bg{{\bf g}}
\def\bx{{\bf x}}
\def\bz{{\bf z}}
\def\bW{{\bf W}}
\def\ba{{\bf a}}
\def\P{P}
\def\PP{\mathbb{P}}
\def\GG{\mathbb{G}}
\def\Var{{\mathbb V}}

\newtheorem{assumption}{}

\begin{document}
\def\spacingset#1{\renewcommand{\baselinestretch}%
{#1}\small\normalsize} \spacingset{1}


\if1\anon
{
  \title{\bf FLEXI--Haz: Semiparametric Deep Survival Inference
Beyond Proportional Hazards}
  \author{Asaf Ben Arie and 
    Malka Gorfine\thanks{
    The authors gratefully acknowledge \textit{to the Israel Science Foundation (ISF) under Grant No. 767/21}}\hspace{.2cm}\\
    Department of Statistics and Operations Research, Tel Aviv University, Israel\\
    }
  \maketitle
} \fi

\if0\anon
{
  \bigskip
  \bigskip
  \bigskip
  \begin{center}
    {\LARGE\bf FLEXI--Haz: Semiparametric Deep Survival Inference
Beyond Proportional Hazards}
\end{center}
  \medskip
} \fi

\bigskip
\begin{abstract}
Flexible survival prediction does not by itself ensure valid
inference for covariate effects or subject-specific survival
probabilities. We develop FLEXI--Haz, a semiparametric deep neural
network model for right-censored data that combines an interpretable
linear component in the log-hazard with a flexible function of time
and adjustment covariates. This formulation permits nonproportional
nuisance effects while retaining interpretable regression parameters.
Under regularity conditions, we establish nuisance convergence rates
that are minimax optimal up to logarithmic factors over composite
H\"older classes, together with $\sqrt{n}$-consistent, asymptotically
normal, and semiparametrically efficient estimation of the regression
parameter. For adjustment covariates with finite support, we derive
the explicit influence function for cumulative hazard and construct
a cross-fitted one-step estimator yielding asymptotically valid
pointwise confidence intervals for subject-specific survival.
Simulations demonstrate reduced bias and improved coverage relative
to a proportional hazards comparator under time-varying nuisance
effects. An analysis of the Rotterdam breast cancer data illustrates
the consequences of flexible adjustment for inference on treatment
associations.
\end{abstract}

\noindent%
{\it Keywords:} 
Interpretability, Partial linear regression, Survival analysis,
Time to event data.
\vfill

\newpage
\spacingset{1.8} 

\section{Introduction}

Deep neural networks (DNNs) have become widely used in survival analysis because of their ability to capture complex nonlinear and high-dimensional structure. Much of this literature  extends the Cox proportional hazards (PH) model \citep{cox1972regression} by replacing the linear risk score with a neural
network and training by the negative partial likelihood. Early work includes \citet{faraggi1995neural}, \citet{katzman2018deepsurv} and numerous extensions to
genomics, medical imaging, recurrent and discrete-time settings, and hypothesis testing
\citep{yousefi2017predicting,ching2018coxnnet,haarburger2019image,li2019deep,giunchiglia2018rnnsurv,gensheimer2019nnet,zhong2026variable}.
While these approaches can model complex nonlinear relationships, they inherit the PH restriction, which
may be too rigid when covariate effects vary over time.

Several DNN-based  methods move beyond PH.  \cite{kvamme2019time} augment the network input with time while retaining an unspecified baseline hazard, whereas \citet{hu2023conditional} and \citet{su2026nonparametric} consider baseline-free continuous-time hazard formulations. \citet{hu2023conditional} do not provide asymptotic theory, while \citet{su2026nonparametric} focus on one- and two-sample hypothesis testing.
Another direction is the DeepEH model of \citet{zhong2021deep}, which replaces the two regression components of the classical extended-hazard model by neural networks and provides consistency and convergence-rate results. However, these approaches generally do not provide an interpretable low-dimensional regression component with classical $\sqrt{n}$ inference.

A semiparametric partially linear structure is  particularly attractive when a small set of covariates is of primary scientific interest and the remaining covariates are included for flexible adjustment.  
In survival analysis, 
\cite{zhong2022deep} proposed the partially linear Cox model
$h_0(t)\exp\{ \zeta^\top \bZ + f(\bX)\}$,
where \(\bZ\) is the low-dimensional covariate vector of primary interest, \(\bX\) is a nuisance
covariate vector, and \(f\) is approximated by a DNN. Their approach yields optimal nuisance convergence and  \(\sqrt{n}\)-consistent estimation of \(\zeta\), but remains tied to the PH and baseline hazard estimation,
with its DNN approximation theory building on  \citet{schmidthieber2020nonparametric}.

The restriction $f(\bX)$, rather than $f(t,\bX)$, is consequential. It requires proportional effects over time for all nuisance covariates, a substantially stronger assumption than requiring proportionality only for the covariates of primary scientific interest $\bZ$.  In many applications, $\bZ$ contains  treatment indicators or selected biomarkers for which an interpretable constant effect is desired, whereas $\bX$ contains demographic, clinical, genomic, or other adjustment variables whose effects need not be proportional. Hence, misspecifying the nuisance time structure can bias the scientifically relevant effects. 

\subsection{Contribution 1: Efficient Semiparametric Estimation}
Our first contribution is a flexible semiparametric framework together with an asymptotic theory for efficient estimation of interpretable parameters under nonproportional nuisance effects. Specifically, we model the log-hazard as
$\btheta_o^\top \bZ + g_o(t,\bX)$,
where \(\btheta_o \in \mathbb{R}^p\) is the finite-dimensional parameter of interest and
\(g_o(t,\bX)\) is an unknown nuisance function approximated by a DNN. This baseline-free formulation allows flexible time-covariate interactions in the nuisance component while preserving the interpretable multiplicative effect $\exp(\theta_{oj})$ for each component of $\bZ$.  Our Rotterdam breast-cancer analysis (Section 6.2) illustrates how flexible adjustment can
materially affect inference on treatment associations.
For example, the confidence interval for the chemotherapy coefficient in the low-estrogen-receptor subgroup excludes
zero under FLEXI--Haz but includes zero under the partially linear PH model of  \cite{zhong2022deep}.

This extension is methodologically nontrivial: our full-likelihood approach requires numerical integration of the cumulative hazard, and the asymptotic analysis must accommodate a nuisance function that depends jointly on time and covariates.
Under the stated regularity conditions, we
show that the DNN estimator of $g_o$ attains the minimax-optimal
convergence rate up to logarithmic factors over composite
H\"older classes, while the estimator of $\btheta_o$ is
$\sqrt{n}$-consistent, asymptotically normal, and
semiparametrically efficient.

\subsection{Contribution 2: Inference for
Subject-Specific Survival}

Our second contribution is asymptotically valid inference for a
new subject's survival probabilities. This requires additional
theory beyond efficient estimation of $\btheta_o$: the survival
function depends directly on the unknown nuisance function
$g_o$, and its estimation error can affect the validity of
confidence intervals. For adjustment covariates $\bX$ with finite
support, we derive an explicit influence function and
construct cross-fitted one-step estimators of the cumulative
hazard and survival function. The correction removes the
first-order effect of nuisance-estimation error, while the
variance estimator accounts for uncertainty in both the
regression and nuisance components. Under the stated regularity
conditions, we establish $\sqrt{n}$-consistency, asymptotic
normality, and valid Wald-type pointwise confidence intervals,
conditional on a new subject's covariates.

The finite-support formulation accommodates multivariate
categorical or grouped adjustment covariates, while the primary
covariates $\bZ$ may remain continuous. It imposes neither
proportionality nor additivity on the nuisance effects.
For example, our Rotterdam analysis retains 590 distinct
observed profiles of the seven-dimensional adjustment vector,
illustrating the substantial covariate diversity that this
formulation can accommodate. 
Importantly, finite support of $\bX$ does not reduce the proposed procedure to stratified nonparametric survival estimation. Rather than estimating separate survival curves within each covariate profile, FLEXI–Haz exploits the common partially linear hazard structure and shared DNN representation to borrow information across profiles. This can be particularly advantageous when many profiles contain few subjects or events.
When continuous covariates are
grouped, the resulting intervals concern survival conditional
on those groups and on $\bZ$. The finite-support assumption is
specific to the explicit survival-inference procedure; it is
not required for the nuisance convergence results or efficient
inference on $\btheta_o$. A high-level extension beyond finite
support, requiring additional representer and approximation
conditions, is presented in the Supplementary File.

\section{Model Formulation}
We consider right-censored survival data, where for each individual the true event time is denoted by $U$ and the censoring time by $C$. We observe only
$T=\min(U,C)$ and $\Delta=\mathbf{1}(U\le C)$,
indicating the observed time and whether the event occurred.
Let the covariate vector be $({\bZ}^{\sf T},{\bX}^{\sf T})^{\sf T}$, where ${\bZ}\in\mathbb{R}^p$ contains the variables of primary scientific interest and ${\bX}\in\mathbb{R}^d$ denotes nuisance covariates. Our goal is to characterize the effect of $\bZ$ on the event time while appropriately adjusting for $\bX$. 
 For simplicity, we focus here on time-independent covariates; the extension to time-dependent covariates will be discussed in Section 7. We assume that $U$ and $C$ are conditionally independent given $(\mathbf{X},\mathbf{Z})$. The observed data consist of $n$ independent replicates,  $\{(T_i,\Delta_i,\mathbf{X}_i,\mathbf{Z}_i)\}_{i=1}^n$.

Under the above setting, the PH model \citep{cox1972regression} specifies the instantaneous  hazard
function at time $t>0$ as $h_1(t\mid \bX,\bZ)=h_0(t)\exp(\balpha_1^\top X+\balpha_2^\top Z)$, where $h_0$ is an unspecified baseline hazard. DeepSurv \citep{katzman2018deepsurv}  replaces the linear predictor by a DNN, while Cox-Time \citep{kvamme2019time}  further allows time-dependent effects through a function $f_1(t,\bX,\bZ)$. \citet{hu2023conditional} proposed a baseline-free hazard model $h_2(t\mid \bX,\bZ)=\exp\{f_2(t,\bX,\bZ)\}$, and \citet{zhong2021deep} proposed the DeepEH model, which generalizes the classical extended hazard model by replacing its two regression components with DNNs. Finally, \cite{zhong2022deep}  introduced the deep partially linear Cox model
\begin{eqnarray}\label{eq:modelzhong}
h_3(t \mid \mathbf{X},\mathbf{Z}) = h_0(t)\exp\{\bzeta^\top \mathbf{Z} + f(\mathbf{X})\} \, ,    
\end{eqnarray}
which combines an interpretable linear component with flexible adjustment through a DNN. 

In this work, we extend the deep partially linear model (\ref{eq:modelzhong}) to the non-PH setting by considering a baseline-free model of the form
\begin{equation}\label{eq:model1}
h_o(t \mid \mathbf{X},{\bZ}) = \exp\{ \btheta_o^\top \mathbf{Z} + g_o(t,\mathbf{X}) \} \, .   
\end{equation}
Thus,  proportionality is retained for $\bZ$ conditional
on $\bX$, while the effects of $\bX$ may vary with time, and $g_o(t,\mathbf{X})$ is estimated by a DNN. Model (\ref{eq:model1}) is ``baseline-free'' because all  time dependence  is absorbed into $g_o(t,\bX)$. It strictly generalizes the  partially linear Cox model (\ref{eq:modelzhong}),  recovered when
$g_o(t,\bX)=\log h_0(t)+f(\bX)$, while allowing unrestricted time-covariate interactions in the nuisance component. It therefore removes the PH restriction on $\bX$ without sacrificing the partially linear interpretation of $\btheta_o$. 

The covariates $\bZ$ and $\bX$ need not be independent. For
identification, $\bZ$ excludes an intercept and retains
nondegenerate residual variation after conditioning on
$(T,\bX)$ among observed failures, as formalized in
Assumption~A5 (Section~4.1). Our estimation procedure
uses the full likelihood, with numerical approximation
of the cumulative-hazard integral. Sections~3--5 develop
the corresponding estimation and theory.

\subsection{Composite \text{H\"older} Class of \texorpdfstring{$g_o$}{go}}
To formalize the complexity assumptions on the nuisance component $g_o$,  we follow the compositional function framework  of \citet{juditsky2009} and \citet{schmidthieber2020nonparametric}, 
 and use the notation of \citet{zhong2022deep}. We assume that the nuisance function $g_o$ admits a hierarchical, low-dimensional representation of the form
$g_o = g_q \circ g_{q-1} \circ \cdots \circ g_1 \circ g_0$,
with the convention that $(g\circ f)(x)=g(f(x))$.  
This composition is associated with a dimension vector ${\bf d}^*=( d^*_0,\ldots,d^*_{q+1}) \in \mathbb{N}_+^{q+2}$ 
and an intrinsic dimension vector $\tilde{\bf d}=(\tilde d_0,\ldots,\tilde d_q)\in\mathbb{N}_+^{q+1}$, $\tilde d_\ell \leq  d^*_\ell$ for each $\ell$. 
The outermost input is $ d^*_0=r=1+d$, corresponding to observed times and covariates $\mathbf{X}$, and the final output is scalar $ d^*_{q+1}=1$. 
For each layer $l$, $l=0,\ldots,q$, the vector-valued mapping $g_\ell=(g_{\ell 1},\ldots,g_{\ell\, d^*_{\ell+1}})^\top$ maps the bounded set $[a_\ell,b_\ell]^{\, d^*_\ell}$ into $[a_{\ell+1},b_{\ell+1}]^{\,d^*_{\ell+1}}$, where $a_l,b_l \in \mathbb{R}$ are fixed finite constants ensuring each intermediate representation remains bounded. Additionally, each component $g_{\ell j}$ depends on at most $\tilde{d}_\ell$ coordinates of its input. 


We also assume that each component function satisfies a standard Hölder–smoothness condition. 
Let $\alpha>0$ and $M>0$. 
The ball of $\alpha$-Hölder functions with radius $M$ consists of all such functions whose derivatives and Hölder constants are bounded by $M$. Formally, for a domain $\mathbb{D}\subset\mathbb{R}^r$, the Hölder ball is defined as
\[
\mathcal{H}^{\alpha}_{r}(\mathbb{D},M)
=
\Big\{ g:\mathbb{D}\to\mathbb{R} :
\sum_{\boldsymbol{\beta}:|\boldsymbol{\beta}|<\alpha}\|\partial^{\boldsymbol{\beta}} g\|_{\infty}
+
\sum_{\boldsymbol{\beta}:|\boldsymbol{\beta}|=\lfloor\alpha\rfloor}
\sup_{u,v\in \mathbb{D}, u\neq v}
\frac{|\partial^{\boldsymbol{\beta}} g(u)-\partial^{\boldsymbol{\beta}} g(v)|}{\|u-v\|^{\alpha-\lfloor\alpha\rfloor}}
\le M \Big\},
\]
where $\boldsymbol{\beta}^{\sf T}=(\beta_1,\ldots,\beta_r)\in\mathbb{N}^r$ is a multi-index,
$\partial^{\boldsymbol{\beta}}=\partial^{\beta_1}\cdots\partial^{\beta_r}$,
$|\boldsymbol{\beta}|=\sum_{j=1}^r\beta_j$, $\|g\|_\infty$ is the sup--norm over its domain,
and $\lfloor\alpha\rfloor$ is the largest integer strictly smaller than $\alpha$.

To formalize the function class used for the nuisance component,
let $\boldsymbol{\alpha}^{\sf T}=(\alpha_0,\ldots,\alpha_q)\in\mathbb{R}_+^{q+1}$ denote the layerwise smoothness parameters. 
The composite Hölder class is defined as
\[
\mathcal{H}(q,\boldsymbol{\alpha},{\bf d}^*,\tilde{{\bf d}},M)
=
\Big\{ g=g_q\circ\cdots\circ g_0 : g_{\ell j}\in \mathcal{H}^{\alpha_\ell}_{\tilde d_\ell}([a_\ell,b_\ell]^{\tilde d_\ell},M), \ |a_\ell|,|b_\ell|\le M \Big\}.
\]
This class accommodates a broad range of structured functions, such as additive models, single- and multiple-index models, and hierarchical interactions, through the intrinsic dimensions $\tilde d_\ell$. Finally, we assume that the true nuisance function $g_o$ belongs to the composite class,
$g_o \in \mathcal{H}(q,\boldsymbol{\alpha},{\bf d}^*,\tilde{{\bf d}},M)$.

For an $r$-variate non-composite H\"older class, the
classical minimax rate in $L_2$ norm is
$n^{-\alpha/(2\alpha+r)}$, where here $r=d+1$ \citep{tsybakov2009}. 
For multilayer composite Hölder classes, \citet{schmidthieber2020nonparametric} established matching minimax lower and upper bounds (up to polylogarithmic factors), showing that  the optimal rate is governed by
$\gamma_n \;=\; \max_{\ell=0,\ldots,q} \, n^{-\tilde\alpha_\ell/(2\tilde\alpha_\ell+\tilde d_\ell)}$,
where the effective smoothness indices are
$
\tilde\alpha_\ell \;=\; \alpha_\ell \prod_{k=\ell+1}^{q}(\alpha_k \wedge 1),
\ \ a\wedge b=\min\{a,b\}
$. Thus, each layer contributes to the overall difficulty of the estimation problem through its intrinsic dimension $\tilde{d}_{\ell}$ and effective smoothness $\tilde{\alpha}_{\ell}$. 

\subsection{Neural Network Class}
For the theoretical analysis, we approximate $g_o(t,\bX)$
within a class of sparse deep ReLU networks with bounded
parameters and output, following \cite{schmidthieber2020nonparametric}.
This choice is motivated by two considerations: (i) ReLU networks are universal approximators capable of representing complex, high-dimensional functions; and (ii) their structure lends itself to sharp theoretical control through entropy bounds and sparsity constraints.
We define a general $(K+1)$-layer feedforward ReLU network with depth $K \in \mathbb{N}_+$ and width vector $\boldsymbol{p}^{\top}=(p_0,\ldots,p_{K+1})$, where $p_0=r=d+1$ denotes the input dimension and $p_{K+1}=1$ the scalar output. Specifically, the network is defined by
$
g(\mathbf{u}) \;=\; \mathbf{W}_K \,\sigma\!\big(\mathbf{W}_{K-1}\,\sigma(\cdots \sigma(\mathbf{W}_0 \mathbf{u}+\mathbf{v}_0)\cdots)+\mathbf{v}_{K-1}\big)+\mathbf{v}_K$,
where $\mathbf{u}\in\mathbb{R}^r$, $\sigma(x)=\max\{x,0\}$ is applied componentwise, $\mathbf{W}_\ell\in\mathbb{R}^{p_{\ell+1}\times p_\ell}$ are weight matrices, and $\mathbf{v}_\ell\in\mathbb{R}^{p_{\ell+1}}$ are bias vectors, $\ell=0,\ldots,K$. 

For theoretical analysis we first focus on networks with uniformly bounded parameters,
$
\mathcal{G}(K,\boldsymbol{p})
\;=\;
\Big\{\, g:\ \max_{\ell=0,\ldots,K}\big(\|\mathbf{W}_\ell\|_\infty \vee \|\mathbf{v}_\ell\|_\infty\big)\le 1 \,\Big\}$,
which ensures controlled Lipschitz behavior of the class, where for a matrix ${\bf A}$ and a vector ${\bf a}$ the entrywise sup–norms are adopted,
$\|\mathbf{A}\|_\infty=\max_{i,j}|A_{ij}|$, $\|\mathbf{a}\|_\infty=\max_j |a_j|$.
Following \citet{schmidthieber2020nonparametric} we further impose sparsity. For sparsity level $s\in\mathbb{N}_+$ and range bound $D>0$, the sparse deep ReLU class is defined as
$$
\mathcal{G}(K,s,\boldsymbol{p},D)
\;=\;
\Big\{
g\in\mathcal{G}(K,\boldsymbol{p})\ :\ 
\sum_{\ell=0}^K \big(\|\mathbf{W}_\ell\|_0 + \|\mathbf{v}_\ell\|_0\big)\le s,\ \ \|g\|_\infty \le D
\Big\},
$$
where $\|\cdot\|_0$ counts nonzero entries. Sparsity serves to control network complexity and is essential for obtaining optimal convergence rates. In the subsequent analysis we set the depth, width, and sparsity levels to balance the bias-variance tradeoff. The bracketing entropy \citep[Chs.~2–3]{vdvwellner1996} of this class is then bounded using results of \citet{schmidthieber2020nonparametric}, and Dudley’s entropy integral \citep[Thm.~2.14.1]{vdvwellner1996} is used to establish the required uniform empirical-process bounds.

\section{FLEXI--Haz Procedure}

\subsection{Estimation of \texorpdfstring{$(\btheta_o,g_o)$}{(thetao, go)}}\label{subsec:estimation}
Estimation is based on maximizing the full likelihood induced by the hazard model (\ref{eq:model1}). Under noninformative and conditionally independent censoring given $(\bX,\bZ)$, the corresponding full log–likelihood (after omitting terms independent of \((\btheta,g)\)) is 
\begin{equation}
\ell_n(\btheta,g)
\;=\;
\frac{1}{n}\sum_{i=1}^n
\Big[
\Delta_i\big\{ g(T_i,\mathbf{X}_i)+\btheta^\top\mathbf{Z}_i\big\}
-
\int_{0}^{\tau} Y_i(t)\,\exp\!\big\{g(t,\mathbf{X}_i)+\btheta^\top\mathbf{Z}_i\big\} \,dt
\Big].
\label{loss2}
\end{equation}
where
$Y_i(t)\;=\;\mathbf{1}(T_i\ge t)$, $t\in[0,\tau]$, $i=1,\ldots,n$, are at-risk processes, and $\tau$ is a prespecified constant. Because $g$ is modeled by a DNN, the predictable cumulative hazard  in (\ref{loss2}) contains a time-dependent integral that lacks a closed-form expression and must therefore be approximated numerically during optimization. We therefore introduce 
 a grid $0=t_0<t_1<\cdots<t_{m_n}=\tau$ and approximate the integral by the Riemann–type sum
$$
\int_{0}^{\tau} Y_i(t)\,\exp\!\{g(t,\mathbf{X}_i)+\btheta^\top\mathbf{Z}_i\}\,dt
\;\approx\;
\sum_{j:\,Y_i(t_{j-1})=1}
\big(\min\{T_i,t_j\}-t_{j-1}\big)\,
\exp\!\{g(t_j,\mathbf{X}_i)+\btheta^\top\mathbf{Z}_i\}.
$$
Defining $\Delta_{ij} = \mathbf{1}\{T_i\in(t_{j-1},t_j],\Delta_i=1\}$, we obtain the discretized likelihood
\begin{equation}
\widetilde\ell_n(\btheta,g)
\;=\;
\frac{1}{n}\sum_{i=1}^n
\sum_{j:\,Y_i(t_{j-1})=1}
\Big[
\Delta_{ij}\,\big\{g(t_j,\mathbf{X}_i)+\btheta^\top\mathbf{Z}_i\big\}
-
\big\{\min(T_i,t_j)-t_{j-1}\big\}\,\exp\!\{g(t_j,\mathbf{X}_i)+\btheta^\top\mathbf{Z}_i\}
\Big].
\label{approximated-loss}
\end{equation}
For computation, we discretize the time axis on a fine grid based on all observed times (event and censoring times). Namely, the data are expanded in counting-process form:
each subject contributes one row for each grid interval during which
they remain at risk, carrying the covariates $(\bX_i,\bZ_i)$, the observed
time at risk in that interval, and the event indicator $\Delta_{ij}$.
This organization is analogous to the expanded-data representation
used for Cox models with time-dependent covariates \citep{therneau2000cox, hu2023conditional}.

We consider two estimators. The oracle maximum–likelihood estimator is defined as
\begin{equation}
(\widehat\btheta, \widehat g)
\; \in \;
\arg\max_{\,g\in\mathcal{G}(K,s,\boldsymbol{p},D),\;\theta\in\Theta\subseteq\mathbb{R}^p}
\;\ell_n(\btheta, g) \, .
\label{full-estimator}
\end{equation} 
The estimator that is actually implemented replaces the at-risk integral with its discretized version $\widetilde\ell_n$. We therefore define the numerically integrated  likelihood estimator as
\begin{equation}
(\widetilde\btheta, \widetilde g)
\; \in \;
\arg\max_{\,g\in\mathcal{G}(K,s,\boldsymbol{p},D),\;\theta\in\Theta \subseteq\mathbb{R}^p}
\;\widetilde\ell_n(\btheta,g) \, .
\label{approximated-estimator}
\end{equation}
If the maximizer is not unique, owing to the non-uniqueness of neural-network representations, we take an arbitrary element of the argmax set.
Theorems~1--4 establish the asymptotic properties of the
exact-integral sieve likelihood maximizer in \eqref{full-estimator}.
Under its grid and regularity conditions, Theorem~5
establishes asymptotic equivalence with the discretized
likelihood maximizer in \eqref{approximated-estimator}. 

In practice, we optimize the discretized likelihood using neural-network
optimization with early stopping. At each learning-rate phase boundary,
we refine $\btheta$ while holding the nuisance network fixed. For analyses
using multiple-start selection, we retain the fit with the smallest norm
of the empirical efficient score. Experiment-specific architectures,
tuning procedures, initialization rules, and optimization settings are
provided in Supplementary Section~A.

\subsection{Estimation and Inference for the Survival of a New Subject}
 Let $(\bX_0,\bZ_0)$ be an independent draw from the covariate
distribution, independent of the training sample, and define
\[
H_o(t \mid \bX_0,\bZ_0)
=
\int_0^t h_o(u \mid \bX_0,\bZ_0)\,du \, ,
\qquad
S_o(t \mid \bX_0,\bZ_0)
=
\exp\{-H_o(t \mid \bX_0,\bZ_0)\} \, .
\]
For numerical survival estimation, we augment the integration grid to include each requested evaluation time $t$. All grid-based quantities for that evaluation are computed on the augmented grid.
Our goal is to construct valid pointwise inference for these quantities.
A natural estimator of $H_o(t \mid \bX_0,\bZ_0)$ is obtained by plugging the estimators
$(\widetilde {\btheta},\widetilde g)$ 
into the hazard and then numerically integrating, yielding
$\widehat H(t\mid \bX_0,\bZ_0)
=\sum_{j:t_{j-1}<t}
\{\min(t,t_j)-t_{j-1}\}
\exp\{\widetilde g(t_j,\bX_0)+\widetilde\btheta^\top \bZ_0\}$,
where $\Delta t_j=t_j-t_{j-1}$, and
$\widehat S(t \mid \bX_0,\bZ_0)=\exp\{-\widehat H(t \mid \bX_0,\bZ_0)\}$.
However, since the target depends on the infinite-dimensional nuisance function $g_o$,  the first-order effect of nuisance estimation error need not vanish automatically. In particular, the nuisance-rate theory established in Section 4 controls $\widetilde g-g_o$ in an integrated norm, which is sufficient for estimation of $\btheta_o$ but does not by itself guarantee that the plug-in survival estimator admits an asymptotically linear expansion with an $o_p(n^{-1/2})$ remainder.

Hence, we use a one-step estimator to remove the first-order nuisance bias
\citep{bickel1993efficient,van2000asymptotic}, augmenting the plug-in estimator
by an empirical average of an estimated influence function (IF). Under the
rate conditions below, the resulting remainder is second order, yielding
Wald-type inference for the cumulative hazard and, by the delta method, for
the survival function.
We implement the correction using cross-fitting
\citep{zheng2011cross,chernozhukov2018double}, which separates nuisance
estimation from influence-function evaluation, thereby weakening
empirical-process requirements and reducing sensitivity to overfitting.

Two additional points are worth emphasizing. First, 
 we consider inference for the conditional survival probability
evaluated at an independent draw $(\bX_0,\bZ_0)$, with coverage assessed
marginally over the training sample and this new covariate vector.
Pointwise inference at a fixed covariate value $(\bx_0,\bz_0)$
can be nonregular in unrestricted models with continuously distributed
adjustment covariates, requiring additional structure, smoothness,
or localization assumptions. Random evaluation alone does not remove
this difficulty. Our construction relies on the regularity conditions
stated below, with finite support of $\bX$ providing the setting for
the explicit estimator. Under the conditions of Theorem~6, the
conditional asymptotic result given $(\bX_0,\bZ_0)$ implies the
stated marginal coverage guarantee.
Second, the construction also applies without a separate linear
component when the entire covariate vector has finite
support and the corresponding regularity conditions hold.

We first present the estimator, with technical details deferred to Section 5.
To define the one-step estimator, let $\phi_t^0(\bW;\bX_0,\bZ_0)$ denote the  IF used for $H_o(t\mid \bX_0,\bZ_0)$, where $\bW=(T,\Delta,\bX,\bZ)$, and let
$\widehat\phi_t^{(-k)}(\bW;\bX_0,\bZ_0)$ be its estimator computed using nuisance estimators
trained without fold $k$. Let $\{\mathcal I_k\}_{k=1}^{K_{\mathrm{cf}}}$ be a partition of $\{1,\dots,n\}$ into folds. For $i\in \mathcal I_k$, the observation $\bW_i$ is evaluated using nuisance estimators fitted on the complement of $\mathcal I_k$.
The cross-fitted one-step estimator of the cumulative hazard is
\begin{equation}
\widehat H^{1,\mathrm{cf}}(t\mid \bX_0,\bZ_0)
=
\widehat H^{\mathrm{cf}}(t\mid \bX_0,\bZ_0)
+
\frac{1}{n}\sum_{k=1}^{K_{\mathrm{cf}}} \sum_{i\in \mathcal I_k}
\widehat\phi_t^{(-k)}(\bW_i;\bX_0,\bZ_0),
\label{eq:H1cf}
\end{equation}
where
$
\widehat H^{\mathrm{cf}}(t\mid \bX_0,\bZ_0)
=
\sum_{k=1}^{K_{\mathrm{cf}}} \frac{|\mathcal I_k|}{n}
\widehat H^{(-k)}(t\mid \bX_0,\bZ_0)
$
and the corresponding one-step estimator of the survival function is
\begin{equation}
\widehat S^{1,\mathrm{cf}}(t\mid \bX_0,\bZ_0)
=
\exp\bigl\{-\widehat H^{1,\mathrm{cf}}(t\mid \bX_0,\bZ_0)\bigr\}.
\label{eq:S1cf}
\end{equation}

The one-step correction is not shape-constrained. Hence, in finite samples, \(\widehat H^{1,\mathrm{cf}}(t\mid \bX_0,\bZ_0)\) need not be monotone nondecreasing in \(t\), and
\(\widehat S^{1,\mathrm{cf}}(t\mid \bX_0,\bZ_0)=\exp\{-\widehat H^{1,\mathrm{cf}}(t\mid \bX_0, \bZ_0)\}\) need not be
monotone nonincreasing. For display, monotonicity can be enforced by rearrangement or isotonic projection.

Next, we give the explicit estimator for finite-support \(\bX\), where the Riesz representer has a closed form; a high-level extension to general covariate distributions is given in Supplementary Section D.

\subsubsection{Explicit form under finite discrete covariates}
\label{subsec:discrete-estimator}

We now specialize to the case where $\bX$ has finite support
$\mathcal X=\{\bx^{(1)},\dots,\bx^{(J)}\}$ with
${\Pr}(\bX=\bx^{(j)})=\pi_j>0$, $j=1,\dots,J$.
For $u\in[0,\tau]$ and $j=1,\dots,J$, define
$
A_j(u)
=
\EE_0\!\left\{
Y(u)h_o(u\mid \bX,\bZ)\mid \bX=\bx^{(j)}
\right\}$. For the unique index $j_0$ satisfying
$\bX_0=\mathbf x^{(j_0)}$, write
$A_{\bX_0}(u)=A_{j_0}(u)$.
Then the nuisance part of the IF is represented by a weight function
of the form
\begin{equation}
\psi_t^0(u,\bX;\bX_0,\bZ_0)
=
\mathbf 1(u\le t)\mathbf 1(\bX=\bX_0)
\frac{h_o(u\mid \bX_0,\bZ_0)}
{\Pr(\bX=\bX_0)\,A_{\bX_0}(u)}.
\label{eq:psi-discrete-skeleton}
\end{equation}
Accordingly, a valid IF takes the form
\begin{equation}
\phi_t^0(\bW;\bX_0,\bZ_0)
=
\int_0^t
\mathbf 1(\bX=\bX_0)
\frac{h_o(u\mid \bX_0,\bZ_0)}
{\Pr(\bX=\bX_0)\,A_{\bX_0}(u)}
\,dM(u)
+
c_{0,t}(\bX_0,\bZ_0)^{\sf T}
{\bf I}(\btheta_o)^{-1}
\ell_{\theta_o}^\ast(\bW),
\label{eq:eif-discrete-skeleton}
\end{equation}
where $M(t)=N(t)-\int_0^t Y(u)h_o(u\mid \bX,\bZ)\,du$ is the counting-process martingale under $\PP_0$ with $N(t)=\Delta \mathbf 1(T\le t)$.
The quantities $c_{0,t}(\bX_0,\bZ_0)$, $I(\btheta_o)$, and $\ell^*_{\theta_o}$ are defined formally in Section 5; for now, it suffices to note that they represent the sensitivity term, efficient information, and efficient score, respectively.
Replacing the unknown quantities in \eqref{eq:eif-discrete-skeleton} by their fold-specific estimators yields
an explicit estimator $\widehat\phi_t^{(-k)}(\bW_i;\bX_0,\bZ_0)$.

We estimate the asymptotic variance of
$\sqrt{n}\{\widehat H^{1,\mathrm{cf}}(t\mid \bX_0,\bZ_0)
-H_o(t\mid \bX_0,\bZ_0)\}$ by the empirical second moment
of the cross-fitted influence-function estimates:
\begin{equation}
\widehat\sigma_{H}^2(t ; \bX_0,\bZ_0)
=
\frac{1}{n}\sum_{k=1}^{K_{\mathrm{cf}}}\sum_{i\in \mathcal I_k}
\left\{
\widehat\phi_t^{(-k)}(\bW_i;\bX_0,\bZ_0)
\right\}^2.
\label{eq:var-est-skeleton}
\end{equation}

An asymptotic $(1-\alpha)$ confidence interval for $H_o(t\mid \bX_0,\bZ_0)$ is therefore
\begin{equation}
\widehat H^{1,\mathrm{cf}}(t\mid \bX_0,\bZ_0)
\pm
z_{1-\alpha/2}
\frac{\widehat\sigma_{H}(t ; \bX_0,\bZ_0)}{\sqrt n},
\label{eq:CI-H-skeleton}
\end{equation}
and, by the delta method, an asymptotic $(1-\alpha)$ confidence interval for
$S_o(t\mid \bX_0,\bZ_0)$ is
\begin{equation}
\widehat S^{1,\mathrm{cf}}(t\mid \bX_0,\bZ_0)
\pm
z_{1-\alpha/2}
\frac{
\widehat S^{1,\mathrm{cf}}(t\mid \bX_0,\bZ_0)
\widehat\sigma_{H}(t ; \bX_0,\bZ_0)
}{\sqrt n}.
\label{eq:CI-S-skeleton}
\end{equation}
The asymptotic validity of these intervals is established in Section 5.

\section{Asymptotic Theory for \texorpdfstring{$(\widetilde{\btheta},\widetilde{g})$}{(theta-tilde, g-tilde)}}

\subsection{Assumptions}\label{sec:assumptions}
Some additional notation will be used throughout. We write $a_n \lesssim b_n$ if there exists a constant $c>0$ such that $a_n \le c\,b_n$ for all $n$, and $a_n \asymp b_n$ if both $a_n \lesssim b_n$
and $b_n \lesssim a_n$ hold. For a column vector ${\bf a}$, ${\bf a}^{\otimes 2}={\bf a}{\bf a}^{\sf T}$. Let $\PP_0$ denote the true distribution of a generic observation
$\bW=(T,\Delta,\bX,\bZ)$, and let $\PP_{\bX,\bZ}$ denote
the marginal distribution of $(\bX,\bZ)$ under $\PP_0$.
Unless otherwise stated, write $\PP=\PP_0$.
For an integrable function $f$, we use
$\PP f=\EE_0\{f(\bW)\}$ and
$\PP_n f=n^{-1}\sum_{i=1}^n f(\bW_i)$.
For any measurable function $\mu=\mu(T,\bX,\bZ)$, define the unweighted $L^2_{\PP_0}$ norm 
$\|\mu\|^2_{L^2({\PP_0})} = \EE_0\{\mu(T,\bX,\bZ)^2\}$
and the event–weighted norm 
$\|\mu\|^2_{L^2(\PP_0,\Delta)}= \EE_0\{\Delta\,\mu(T,\bX,\bZ)^2\big\}$.
Let $\eta=(\btheta,g)$, write
$\chi_\eta(t,\bX,\bZ) = g(t,\bX)\ +\ \theta^\top \bZ$,
and denote the true parameter by $\eta_o=(\btheta_o,g_o)$. For two parameters $\eta_1,\eta_2$, we measure
their respective event-weighted score distance by 
\[
d(\eta_1,\eta_2)
\ =\
\Big(\EE_0\big[\Delta\big\{\chi_{\eta_1}(T,\bX,\bZ)-\chi_{\eta_2}(T,\bX,\bZ)\big\}^2\big]\Big)^{1/2}
\ =\ \|\chi_{\eta_1}-\chi_{\eta_2}\|_{L^2({\PP_0,\Delta})} \, .
\]

Finally, the population loss is defined as 
\begin{equation}\label{loss}
\ell_0(\eta)
\;=\;
\EE_0\Big[
\Delta\,\chi_\eta(T,\bX,\bZ)
-
\int_{0}^{\tau} Y(t)\,\exp\!\big\{\chi_\eta(t,\bX,\bZ)\big\}\,dt
\Big] \, .
\end{equation}

\begin{assumption}\label{ass:A1}
The failure  and censoring times  are conditionally independent given 
$(\bX^\top,\bZ^\top)$, and the censoring mechanism is noninformative for $(\btheta,g)$. 
The true nuisance function satisfies 
$g_o \in \mathcal{H}(q,\boldsymbol{\alpha},{\bf d}^*,\tilde{{\bf d}},M)$.
\end{assumption}

\begin{assumption}\label{ass:A2}
The covariate vector  $(\bX^\top, \mathbf{Z}^\top)$ takes values in a bounded subset of $\mathbb{R}^{d+p}$, denoted $\mathcal{X} \times \mathcal{Z}$, where both sets  are compact.  Without loss of generality, assume $\mathcal{X} \subset [0,1]^d$ and that $\| \mathbf{Z} \|_2 \leq B_{\bZ}$ almost surely. The joint  distribution  of $({\bX}^\top,{\bZ}^\top)$ has a density with respect to an appropriate dominating measure that is bounded away from zero on $\mathcal{X} \times \mathcal{Z}$. The parameter space for the finite-dimensional component is restricted to the compact set
$\Theta \;=\; \{\vartheta\in\mathbb{R}^p:\ \|\vartheta\|\le M_\Theta\}$ and \(\btheta_o\) lies in the interior of \(\Theta\).
Finally, the study follows participants over a fixed, finite time horizon  $\tau>0$.
\end{assumption}

\begin{assumption}\label{ass:A3}
 The network architecture satisfies 
$K = O(\log n)$,  $s = O\!\big(n\gamma_n^2 \log n\big)$, and 
$ n\gamma_n^2 \lesssim  \;\min_{1\le \ell\le K} p_\ell \lesssim \max_{1\le \ell\le K} p_\ell \lesssim n$. 
In addition, the sieve is uniformly bounded, $\|g\|_\infty\le D_n$ for a known sieve radius $D_n=O(1)$.
\end{assumption}

\begin{assumption}\label{ass:A4}
There exist $\delta_1, \delta_2 , \delta_3\in(0,1]$ such that, almost surely,
$\Pr(\Delta=1\mid T,\bX,\bZ) \ge \delta_1$, $\Pr(U\ge \tau\mid \bX,\bZ) \ge \delta_2$, and $\Pr(Y(t)=1 \mid \bX, \bZ) \geq \delta_3$ for all $t\in [0,\tau)$.
\end{assumption}

\begin{assumption}\label{ass:A5}
The covariate vector $\mathbf{Z}$ does not include a constant (intercept) term.  
Furthermore, the residual covariance matrix of $\mathbf{Z}$ given $(T,\bX)$ among failures, is positive definite, that is,
\[
\Sigma_{\mathbf Z}
=
\mathbb{E}\!\left[
\big\{\mathbf Z - \mathbb{E}(\mathbf Z \mid T, \mathbf X, \Delta=1)\big\}
\big\{\mathbf Z - \mathbb{E}(\mathbf Z \mid T, \mathbf X, \Delta=1)\big\}^\top \mid \Delta = 1
\right]
\succ 0 \, .
\]
\end{assumption}

\begin{assumption}\label{ass:Anew}
Let
$r_{\theta,g}(T,\bX,\bZ,\Delta)
=\Delta\{\bZ-\bg^*(T,\bX)\}-
\int_0^\tau \{\bZ-\bg^*(t,\bX)\}Y(t)
\exp\{g(t,\bX)+\theta^\top\bZ\}\,dt$,
where $\bg^*$ is the least-favorable projection defined in
Theorem~\ref{thm:eff-score}. The fitted exact-likelihood estimator satisfies
$\PP_n r_{\widehat\theta,\widehat g}=o_p(n^{-1/2})$.
In addition, for every fixed $C<\infty$, uniformly over
$d((\btheta,g),(\btheta_o,g_o))\le C\tau_n$,
$\PP\{r_{\theta,g}-r_{\theta_o,g_o}\} =
-I(\theta_o)(\btheta-\btheta_o)+
O\!\left\{d^2((\btheta,g),(\btheta_o,g_o))\right\}$.
For the numerically integrated estimator, the first condition is
imposed on the corresponding discretized projected score. Finally, the least-favorable projection $\bg^*$ satisfies
$\sup_{(t,\bx)\in[0,\tau]\times\mathcal{X}}
\|\bg^*(t,\bx)\|<\infty$.
\end{assumption}

\ref{ass:A1}--\ref{ass:A3} 
are analogous to those of \cite{zhong2022deep} and imply uniform upper and lower bounds on the hazard. \ref{ass:A4} is stronger than in the partially linear PH model because $g_o(t,\bX)$ 
jointly models time and nuisance covariates, requiring adequate event-risk support throughout $[0,\tau]$.
 \ref{ass:A5} is the  identifiability condition: after conditioning on $(T,\bX)$ among failures, sufficient residual variation in $\bZ$ must remain.

Assumption \ref{ass:Anew} imposes two requirements for inference: approximate stationarity in the least-favourable score direction and a local population expansion with a quadratic nuisance remainder. Neither follows from nuisance convergence rates alone, and neither is required for Theorems~\ref{thm:rate}--\ref{thm:minimax-delta}.
\cite{zhong2022deep} (see Section 6.5 and supplementary Lemma 5) invoke related properties, including exact nuisance-score stationarity along all population nuisance tangent directions. For a constrained neural-network sieve, however, maximization guarantees stationarity only along admissible two-sided differentiable perturbations. Extending this to the least-favourable direction requires a feasibility and approximation argument that is not explicitly supplied in their proof. Likewise, eliminating the first-order nuisance term through projection requires compatibility between the projection space and the nuisance estimation errors.
\ref{ass:Anew} states these requirements explicitly, imposing approximate stationarity only for the projected score used in inference. Our asymptotic normality and efficiency conclusions are conditional on \ref{ass:Anew}.

\subsection{Asymptotic Results}
Proofs are provided in Supplementary Appendix B.

\begin{restatable}{theorem}{RateTheorem}\label{thm:rate}
Let $(\widehat{\btheta},\widehat{g})$ be the estimator defined in (\ref{full-estimator}).
Under Assumptions \ref{ass:A1}--\ref{ass:A5},
\begin{equation}\label{eq:theor1}
\left( \EE\!\Bigl[\Delta\bigl\{ \widehat g(T,\bX)-g_o(T,\bX)\bigr\}^2\Bigr] \right)^{1/2}
     =O_p(\tau_n)    
\end{equation}
where the expectation is taken with respect to an independent draw $(T,\Delta,{\bX},{\bZ}) \sim \PP_0$ from the true distribution, and $\tau_n=\gamma_n\log^{2}n$.  
The quantity in \emph{(\ref{eq:theor1})} is random through its dependence on the training sample used to construct $\widehat g$.
\end{restatable}

By Supplementary Lemma 5, Theorem~\ref{thm:rate} also implies convergence in \(L_2(dt\otimes \PP_{\bX,\bZ})\) over the entire time domain. Theorem \ref{thm:rate}  uses \cite{schmidthieber2020nonparametric}  only for deterministic sparse-ReLU approximation results and the associated network-size and sparsity scaling. All stochastic arguments, including  estimation-error control, empirical-process bounds, and convergence-rate derivation, are developed  specifically for the present survival setting and differ substantially from \cite{zhong2022deep}. The following theorem  establishes the corresponding minimax lower bound.

\begin{restatable}{theorem}{LowerBoundTheorem}\label{thm:minimax-delta}
Under \ref{ass:A1}--\ref{ass:A2} and \ref{ass:A4} there exists $c>0$ such that
\[
\inf_{\widehat g} \sup_{\theta_o \in \Theta,\, g_o \in \mathcal{H}(q,\boldsymbol{\alpha},{\bf d}^*,\tilde{{\bf d}},M)} 
\EE_{\PP_{g_o,\theta_o}}\!\left[\Delta\{\widehat g(T,{\bf X}) - g_o(T,{\bf X})\}^2\right] 
\;\;\geq\; c\,\gamma_n^2 \, .
\]
\end{restatable}

We next characterize the efficient score for $\btheta_o$. To define the nuisance tangent space associated with  $g_o$, let
$\mathcal{H}_{g_o}$ denote the collection of 
differentiable one-dimensional submodels
 $\{g_b: b \in (-1,1) \} \subset \mathcal{H}(q,\boldsymbol{\alpha},{\bf d}^*,\tilde{{\bf d}},M)$
passing through $g_o$, for which there exists
$g\in L^2([0,\tau]\times[0,1]^d)$ such that
$
\lim_{b \to 0} \left\| b^{-1}(g_b - g_o) - g \right\|_{L^2([0,\tau]\times [0,1]^d)} = 0$.
The set of  attainable first–order directions at $g_o$ is 
$$
\mathcal{T}_{g_o} = \Big\{ g \in L^2([0,\tau]\times [0,1]^d) : \exists \, \{g_b\} \in \mathcal{H}_{g_o} \text{ s.t. } 
\lim_{b \to 0} \left\| b^{-1}(g_b - g_o) - g \right\|_{L^2([0,\tau]\times [0,1]^d)} = 0 \Big\} \, .
$$ 
Let $\overline{\mathcal{T}}_{g_o}$ denote the closure of the linear span of
$\mathcal{T}_{g_o}$ under the norm
$
g\mapsto
\left\{\EE\!\left[\Delta g(T,\bX)^2\right]\right\}^{1/2}$. This is the nuisance tangent space. 
Let
$
(\overline{\mathcal{T}}_{g_o})^p
=
\left\{
\bg=(g_1,\ldots,g_p)^{\sf T}:
g_j\in\overline{\mathcal{T}}_{g_o},\ j=1,\ldots,p
\right\}$,
and define $\bg^*=(g_1^*,\ldots,g_p^*)^{\sf T}\in(\overline{\mathcal{T}}_{g_o})^p$
componentwise by
\[
g_j^*
\in
\arg\min_{g\in \overline{\mathcal{T}}_{g_o}}
\EE\!\left[\Delta\{Z_j-g(T,\bX)\}^2\right],
\qquad j=1,\ldots,p \, .
\]
Namely, $\bg^*(t,\bX)$ is the $\Delta$-weighted $L_2$-projection of $\bZ$ onto $(\overline{\mathcal{T}}_{g_o})^p$.
%
Finally, under the counting-process representation of the model \citep{andersen1982cox}, the counting process $N(t)=\Delta I(T \leq t)$ has compensator $\int_0^t Y(s)\exp\{\btheta_o^\top \bZ+g_o(s,\bX)\}\,ds$  so the associated martingale is
$
M(t) = N(t) - \int_0^t Y(s) 
\exp\{\btheta_o^{\sf T} {\bZ} + g_o(s,{\bX})\}\, ds$.

\begin{restatable}{theorem}{EfficientScoreAndInformationBound}\label{thm:eff-score}
Under Assumptions~\ref{ass:A1}--\ref{ass:A5}, the efficient score for $\btheta_o$ in  model \eqref{eq:model1} is
$$
\ell_{\theta_o}^*(T,\bX,\bZ,\Delta)
\;=\;
\int_0^\tau \big\{ \bZ - \bg^*(t,\bX) \big\}\, dM(t) \, .
$$
The semiparametric information bound for $\btheta_o$ is
$$
I(\btheta_o)
\;=\;
\mathbb{E}\!\left\{\,\ell_{\theta_o}^*(T,\bX,\bZ,\Delta)^{\otimes 2}\right\}
=
\mathbb{E}\!\left[\,\Delta\,\{\bZ - \bg^*(T,\bX)\}^{\otimes 2}\right] \, .
$$
Consequently, the asymptotic variance of any regular asymptotically linear (RAL)  estimator of $\btheta_o$ is bounded below by $I(\btheta_o)^{-1}$ in the Loewner sense.
\end{restatable}


These ingredients yield the following asymptotic expansion of \(\widehat{\btheta}\).

\begin{restatable}{theorem}{ThetaHatAsmp}\label{thm:asym-theta}
Under Assumptions \ref{ass:A1}--\ref{ass:Anew}, suppose that the efficient
information matrix $I(\btheta_o)$ is nonsingular and that
$\sqrt{n}\,\tau_n^2 \;\to\; 0$ as $n\to\infty$, where $\tau_n=\gamma_n\log^{2}n$. 
Then,
\[
\sqrt{n}(\widehat{\btheta} - \btheta_o)
=
I(\btheta_o)^{-1}\,\frac{1}{\sqrt{n}}
\sum_{i=1}^n \ell^*_{\theta_o}(T_i,\bX_i,\bZ_i,\Delta_i)
+ o_p(1) \, ,
\]
and hence
$
\sqrt{n}(\widehat{\btheta} - \btheta_o)
\;\overset{D}{\longrightarrow}\;
N({\bf 0}, I(\btheta_o)^{-1})$.
\end{restatable}
The population expansion in \ref{ass:Anew} removes the first-order nuisance contribution. Together with the convergence rate and \(\sqrt{n}\tau_n^2\to0\), the resulting quadratic remainder is \(o_p(n^{-1/2})\).
Finally, we show that numerical integration does not affect first-order asymptotics.

\begin{restatable}{theorem}{ThetaAsmp}\label{thm:quad-equiv}
Let  $(\widetilde\btheta,\widetilde g)$ be the numerically integrated likelihood estimator
defined in~(\ref{approximated-estimator}), with a time grid 
$0 = t_0 < t_1 < \cdots < t_{m_n} = \tau$ and mesh size
$
\delta_n \;=\; \max_{1\le j\le m_n}(t_j - t_{j-1})
$.
Suppose that the grid contains all observed times \(T_1,\dots,T_n\). 
For any fixed $M>0$, let $\mathcal N_n(M)=\{\eta:d(\eta,\eta_o)\le M\tau_n\}$,
and define
$\omega_{n,M}(\delta)=\sup_{\eta\in\mathcal N_n(M)}
\sup_{X\in\mathcal X}
\sup_{\substack{s,t\in[0,\tau]\\ |t-s|\le\delta}}
|g(t,\bX)-g(s,\bX)|$.
Assume that, for every sufficiently large fixed $M$,
$\omega_{n,M}(\delta_n)=o(\tau_n^2)$, 
$\sqrt n\,\omega_{n,M}(\delta_n)\to0$,
that $\bg^*(t,\bX)$ is uniformly Lipschitz in $t$, that
$\sqrt n\,\delta_n\to0$,
and that Assumptions \ref{ass:A1}--\ref{ass:Anew} hold. Then,
$
d(\widetilde\eta,\eta_o) = O_p(\tau_n)
$ and
$\|\widetilde\btheta - \btheta_o\| = O_p(\tau_n)$.
If, in addition, $I(\btheta_o)$ is nonsingular and $\sqrt{n} \tau_n^2 \rightarrow 0$, then 
$\sqrt n(\widetilde\btheta-\btheta_o)
=
I(\btheta_o)^{-1}
\frac{1}{\sqrt n}\sum_{i=1}^n
\ell_{\theta_o}^*(T_i,\bX_i,\bZ_i,\Delta_i)
+o_p(1)$,
and hence
$\sqrt{n}(\widetilde\btheta - \btheta_o)
\;\xrightarrow{\mathcal{D}}\;
N\bigl(\mathbf{0},\,I(\btheta_o)^{-1}\bigr)$.
\end{restatable}



\subsection{Covariance Estimation}
\label{sec:covariance}

Theorem~\ref{thm:quad-equiv} establishes the asymptotic distribution of the implemented
estimator $\widetilde\btheta$. To construct confidence intervals,
we estimate the efficient information matrix $I(\btheta_o)$ as follows. We estimate $\bg^\ast$ coordinatewise using additional neural networks. For each
$j=1,\ldots,p$, $\widehat g_j^\ast(T,\bX)$ is fitted by least squares among subjects with
$\Delta=1$, with $(T,\bX)$ as input and $Z_j$ as the response. 
To reduce overfitting, the projection step is implemented using five-fold cross-fitting. For each fold, the projection networks are trained on the remaining four folds and evaluated on the held-out observations. Combining the resulting out-of-fold predictions gives
$
\widehat{\bg}^{\ast}(T_i,\bX_i)
=
\left(
\widehat g_1^\ast(T_i,\bX_i),\ldots,
\widehat g_p^\ast(T_i,\bX_i)
\right)^\top$.
Define
$
{\bf R}_i
=
{\Delta_i}
\left\{
\bZ_i-\widehat{\bg}^{\ast}(T_i,\bX_i)
\right\}$,
$i=1,\ldots,n$,
and let $\mathcal R$ be the matrix with rows ${\bf R}_i^\top$. We estimate the efficient information and asymptotic covariance by
$
\widehat I
=
\frac{1}{n}\mathcal R^\top\mathcal R$ and
$\widehat{\mathrm{Var}}(\widetilde{\btheta})
=
\frac{1}{n} \widehat I^{-1}$.

\section{Inference for Subject-Specific Survival}

\subsection{Additional assumptions for survival inference}
Fix $t \in [0,\tau]$, and let $(\bX_0,\bZ_0)$ be an independent draw from the covariate distribution, independent of the training sample.
The assumptions below are used to establish pathwise differentiability of the target, derive its IF, and justify the cross-fitted one-step estimator introduced in Section 3.2.
To state these assumptions, we equip nuisance directions with the predictable-variation inner product
\begin{equation}
	\label{eq:inner_product}
	\langle a,b\rangle_0
	=
	\EE_0\!\left\{
	\int_0^\tau 
	a(u,\bX)b(u,\bX)\,
	Y(u)h_o(u\mid \bX,\bZ)\,du
	\right\} \, ,
	\qquad
	\|b\|_{L^2(\PP_0,\Delta)}^2=\langle b,b\rangle_0 \, .
\end{equation}

\begin{assumption}\label{ass:A6}
For the fixed time $t$, $\EE_0\!\int_0^t h_o(u\mid \bX_0,\bZ_0)\,du <\infty$, and there exists $\kappa_t>0$ such that
		$\PP_0( Y(u)=1\mid \bX_0,\bZ_0) \ge \kappa_t$
		a.s. for a.e. $u\in[0,t]$.        
\end{assumption}

\begin{assumption}\label{ass:A7}
Define the linear operator
$Q : \overline{\mathcal{T}}_{g_o} \to L^2(\PP_{\bX,\bZ})$, 
$$(Qb)(\bX_0,\bZ_0)=\int_0^t h_o(u \mid \bX_0,\bZ_0)  \, b(u,\bX_0)\, du ,
\qquad b \in \overline{\mathcal{T}}_{g_o} \, .
$$
Assume that \(Q\) is bounded, i.e., there exists a constant \(C_t < \infty\) such that
$
\|Qb\|_{L^2(\PP_{\bX,\bZ})}^2
=
E_{\bX_0,\bZ_0}\!\left\{ (Qb)(\bX_0,\bZ_0)^2 \right\}
\le
C_t \|b\|_{L^2(\PP_0,\Delta)}^2$,
for all $b \in \overline{\mathcal{T}}_{g_o}$.
Assume further that there exists a square-integrable random function
$\psi_t^0(\cdot,\cdot;\bX_0,\bZ_0)$,
such that for every \(b \in \overline{\mathcal{T}}_{g_o}\),
\begin{equation}\label{eq:qblinear}
(Qb)(\bX_0,\bZ_0)
=
\left\langle \psi_t^0(\cdot,\cdot;\bX_0,\bZ_0),\, b \right\rangle_0
\quad\text{in }L^2(\PP_{\bX,\bZ}).
\end{equation}
\end{assumption}

\begin{assumption}\label{ass:A8}
The nuisance covariate $\bX$ takes values in a finite set
$\mathcal X=\{\bx^{(1)},\dots,\bx^{(J)}\}$,
with $\Pr(\bX=\bx^{(j)})=\pi_j>0$, $j=1,\dots,J$.   
\end{assumption}

\begin{assumption}\label{ass:A9}
Let $\{\mathcal I_k\}_{k=1}^{K_{\mathrm{cf}}}$ be a fixed $K_{\mathrm{cf}}$-fold partition of $\{1,\dots,n\}$, and for each $k$ let
$(\widetilde\btheta^{(-k)},\widetilde g^{(-k)})$ and $\widehat \bg^{*,(-k)}$
be estimated using only the training sample $\mathcal I_k^c$. Assume these estimators are measurable with respect to the training sigma-field and, conditional on the training sample, are independent of the validation observations $\{\bW_i:i\in \mathcal I_k\}$. Assume moreover that, uniformly in $k$,
$\|\widetilde g^{(-k)}-g_o\|_{L^2(\PP_0,\Delta)} = O_p(\tau_n)$,
$\|\widehat \bg^{*,(-k)}-\bg^*\|_{L^2(\PP_0,\Delta)} = O_p(\tau_n)$,
$\|\widetilde \btheta^{(-k)}-\btheta_o\| = O_p(n^{-1/2})$,
where
$\tau_n=\gamma_n \log^2 n$, and $\sqrt n\,\tau_n^2\to 0$.   
Finally, the projection estimators $\widehat{\bg}^{*,(-k)}$ are
constructed over uniformly bounded sparse-ReLU classes whose
depth, width, and sparsity satisfy the same order conditions
as in \ref{ass:A3}.
\end{assumption}

\begin{assumption}\label{ass:A10}
    The efficient information matrix $I(\btheta_o)$ is nonsingular
    and the smallest eigenvalue satisfies $\lambda_{min}{I(\theta_o)} \geq c_0$ for some constant $c_0 > 0$. In addition,
$E_0\bigl\|\ell^*_{\theta_o}(\bW)\bigr\|^2<\infty$,
and 
$E_0\!\left[\bigl\{\phi^0_t(\bW;\bX_0,\bZ_0)\bigr\}^2\right]<\infty $.
\end{assumption}

\begin{assumption}\label{ass:A11}
Let $0 = t_0 < t_1 < \cdots < t_{m_n} = \tau$ be the numerical grid used in the estimation procedure, with mesh size
$\delta_n \;=\; \max_{1\le j\le m_n}(t_j - t_{j-1})$.
Assume the grid contains all observed event and censoring times and the fixed evaluation time $t$, that
$\sqrt n\,\delta_n\to 0$,
and that the corresponding Riemann-approximation error is $O(\delta_n)$ uniformly over a neighborhood of $(\btheta_o,g_o)$. 
\end{assumption}


Under finite-support $\bX$, \ref{ass:A6} and \ref{ass:A8} yield the explicit  Riesz representer below.  \ref{ass:A8}  is required only for the explicit survival-inference procedure, not for estimation or inference on
$\btheta_o$.
A high-level extension to general covariate distributions is given in Supplementary File D.

\subsection{Influence function and one-step representation}
Proofs are provided in Supplementary Appendix C. We now provide the semiparametric ingredients underlying the estimator introduced in Section 3.2.  
Consider nuisance perturbations of the form $g_\varepsilon=g_o+\varepsilon b$,
with predictable $b=b(u,\bX)$. The corresponding nuisance score equals
\begin{equation}
	\label{eq:nuisance_score}
	S_{g,b}(\bW)
	=
	\left.\frac{\partial}{\partial \varepsilon}
	\ell(\btheta_o,g_\varepsilon; \bW)\right|_{\varepsilon=0}
	=
	\int_0^\tau b(u,\bX)\,dM(u) \, ,
\end{equation}
where $\ell(\btheta,g;\bW)$ denotes the per-observation log-likelihood contribution corresponding to (3.1). 
For $\ba\in\mathbb R^p$, consider the parametric path
$\btheta_\varepsilon=\btheta_o+\varepsilon \ba$.
To construct a least favorable submodel, we simultaneously perturb the nuisance function as
\begin{equation}
	\label{eq:LFS}
	g_{\varepsilon,\ba}(u,\bX)
	=
	g_o(u,\bX)-\varepsilon\,\ba^\top \bg^*(u,\bX) \, .
\end{equation}
More generally, allowing an additional nuisance direction $b\in \overline{\mathcal{T}}_{g_o}$, we consider
\[
g_\varepsilon(u,\bX)
=
g_o(u,\bX)
-\varepsilon\,\ba^\top \bg^*(u,\bX)
+\varepsilon b(u,\bX) \, ,
\]
while keeping $\btheta_\varepsilon=\btheta_o+\varepsilon\ba$. The score decomposes as
\begin{equation}
	\label{eq:score_decomp}
	S(\bW;\ba,b)
	=
	\ba^\top \ell_{\theta_o}^*(\bW)
	+
	\int_0^\tau b(u,\bX)\,dM(u) \, .
\end{equation}
Fix $t\in(0,\tau]$ and define the linear functional
\begin{equation}
	\label{eq:L_def}
	L_t(b)
	=
	\int_0^t
	h_o(u\mid \bX_0,\bZ_0)\,
	b(u,\bX_0)\,du \, ,
	\qquad b\in  \overline{\mathcal{T}}_{g_o} \, .
\end{equation}
For a vector-valued function 
${\bf \mu}(u,\bX)=(\mu_1,\dots,\mu_p)^\top$,
we define
$L_t({\bf \mu})=\big(L_t(\mu_1),\dots,L_t(\mu_p)\big)^\top$.

Define
\begin{equation}
	\label{eq:c_def}
	c_{0,t}(\bX_0,\bZ_0)
	=
	\int_0^t
	h_o(u\mid \bX_0,\bZ_0)
	\{\bZ_0-\bg^*(u,\bX_0)\}\,du
	=
	\bZ_0\,H_o(t\mid \bX_0,\bZ_0)
	-
	L_t(\bg^*) \, .
\end{equation}
The pathwise derivative of the cumulative-hazard target is given below.
\begin{restatable}{lemma}{Hderivative}
	\label{lem:pathwise}
	Let $\PP_\varepsilon$ be a regular submodel with score \eqref{eq:score_decomp}. Then
	\begin{equation}\label{eq:Gateaux_deriv_MS2}
	\left.
\frac{d}{d\varepsilon}
H_\varepsilon(t\mid \bX_0,\bZ_0)
\right|_{\varepsilon=0}
=
L_t(b)
+
\ba^\top c_{0,t}(\bX_0,\bZ_0) \, .
	\end{equation}
\end{restatable}


The resulting IF is given next.
\begin{restatable}{lemma}{EIFms}
	\label{lem:EIF_existence_MS}
	Fix $t\in[0,\tau]$ and let $(\bX_0,\bZ_0)$ be an independent draw from the covariate distribution,
	independent of the training sample. 
	Assume \ref{ass:A6}--\ref{ass:A7}. Then the map $\PP\mapsto H_{\PP}(t\mid \bX_0,\bZ_0)$ is pathwise differentiable at $\PP_0$
	as an element of $L^2(\PP_{\bX,\bZ})$; that is, for every regular submodel
	$\PP_\varepsilon$ through $\PP_0$ with score $S(\bW;\ba,b)$,
	\[
	\left.\frac{d}{d\varepsilon}H_\varepsilon(t\mid \bX_0,\bZ_0)\right|_{\varepsilon=0}
	=
	\EE_0\!\big\{\phi_t^0(\bW;\bX_0,\bZ_0)\,S(\bW;\ba,b)\ \big|\ \bX_0,\bZ_0\big\}
	\quad\text{in }L^2(\PP_{\bX,\bZ}),
	\]
	where
	\begin{equation}\label{eq:EIF_MS}
		\phi_t^0(\bW;\bX_0,\bZ_0)
		=
		\int_0^\tau
		\psi^0_{t}(u,\bX;\bX_0,\bZ_0)\,dM(u)
		+
		c_{0,t}(\bX_0,\bZ_0)^\top
		I(\btheta_o)^{-1}
		\ell_{\theta_o}^*(\bW) \, .   
	\end{equation}
    Thus, $\phi_{0t}$ is a valid influence function for $H_o(t\mid \bX_0,\bZ_0)$. 
\end{restatable}

\subsection{Finite-support \texorpdfstring{$\bX$}{X}}
Under finite-support $\bX$, the Riesz representer is explicit.
\begin{restatable}{lemma}{Discrete}
	\label{prop:closed_form_discrete}
	Under Assumptions \ref{ass:A6}--\ref{ass:A8}, 
	$\psi^0_{t}(\cdot,\cdot;\bX_0,\bZ_0)$ defined by (\ref{eq:psi_closed}) satisfies the Riesz identity (\ref{eq:qblinear}) for every $b\in\overline{\mathcal T}_{g_o}$, where
	\begin{equation}
		\label{eq:psi_closed}
		\psi^0_{t}(u,\bX;\bX_0,\bZ_0)
		=
		\mathbf 1(u\le t)
		\mathbf 1(\bX=\bX_0)
		\frac{h_o(u\mid \bX_0,\bZ_0)}
		{\Pr(\bX=\bX_0)\,A_{\bX_0}(u)}.
	\end{equation}
\end{restatable}

For implementation, 
define the fold-specific hazard estimator
$
\widehat h ^{(-k)}(u\mid \bX,\bZ)
=
\exp\!\left\{
\widetilde g^{(-k)}(u,\bX)
+
(\widetilde\btheta^{(-k)})^\top \bZ
\right\}$. For an evaluation time $t$ included in the grid, define
the plug-in cumulative hazard at $(\bX_0,\bZ_0)$  by
$
\widehat H^{(-k)}(t\mid \bX_0,\bZ_0)
=
\sum_{j: t_j\le t}
(t_j-t_{j-1})
\widehat h^{(-k)}(t_j\mid \bX_0,\bZ_0)$.
For $i=1,\ldots,n$, define
$
\Delta N_{ij}
=
1\{T_i\in (t_{j-1},t_j],\ \Delta_i=1\}$,
$
Y_{ij} = Y_i(t_{j})
=
1\{T_i\ge t_{j}\}.
$
For $i \in I_k$ set
$
\Delta A_{ij}^{(-k)}
=
\Delta t_j
Y_{ij}
\widehat h^{(-k)}(t_j\mid \bX_i,\bZ_i)$.
The discretized residual increment is
$
\Delta\widehat M_{ij}^{(-k)}
=\Delta N_{ij}-\Delta A_{ij}^{(-k)}$.
Let 
\[
\widehat\psi_{t}^{(-k)}(t_j,\bX_i;\bX_0,\bZ_0)
=
\mathbf 1(t_j\le t)\,
\mathbf 1(\bX_i=\bX_0)\,
\frac{
	\widehat h^{(-k)}(t_j\mid \bX_0,\bZ_0)
}{
	\widehat\pi_{j_0}^{(-k)}
	\widehat A_{j_0}^{(-k)}(t_j)
},
\]
where $j_0$ satisfies $\bX_0=\bx^{(j_0)}$,
$
\widehat\pi_{j_0}^{(-k)}
=
|\mathcal I_k^c|^{-1}
\sum_{i\in  \mathcal I_k^c}
\mathbf 1(\bX_i=\bx^{(j_0)})$,
and
\[
\widehat A_{j_0}^{(-k)}(t_j)
=
\frac{
	\sum_{i\in \mathcal I_k^c}
	Y_{ij}
	\widehat h^{(-k)}(t_j\mid \bX_i,\bZ_i)
	1(\bX_i=\bx^{(j_0)})
}{
	\sum_{i\in \mathcal I_k^c}
	\mathbf 1(\bX_i=\bx^{(j_0)})
}.
\]
The nuisance influence contribution is
$
\widehat\phi^{(-k)}_{t,g}(\bW_i;\bX_0,\bZ_0)
=
\sum_{j=1}^{m_n}
\widehat\psi_{t}^{(-k)}(t_j,\bX_i;\bX_0,\bZ_0)
\Delta\widehat M_{ij}^{(-k)}$.
Let $\widehat \bg^{*,(-k)}$ denote the estimator of the projection
$\bg^*$ and define
$$
\widehat c_t^{(-k)}(\bX_0,\bZ_0)
=
\sum_{j: t_j\le t}
(t_j-t_{j-1})
\widehat h^{(-k)}(t_j\mid \bX_0,\bZ_0)
\Big\{
\bZ_0-\widehat \bg^{*,(-k)}(t_j,\bX_0)
\Big\}.
$$
Let $\widehat I^{(-k)}$ and
$\widehat\ell_{\theta_o}^{*,(-k)}$
be the out-of-fold efficient information and score estimators.
Define
\[
\widehat\phi^{(-k)}_{t,\theta}(\bW_i;\bX_0,\bZ_0)
=
\widehat c_t^{(-k)}(\bX_0,\bZ_0)^\top
\big\{\widehat I^{(-k)}\big\}^{-1}
\widehat\ell_{\theta_o}^{*,(-k)}(\bW_i) \, .
\]
The fold-specific influence estimator is
$
\widehat\phi_t^{(-k)}
=
\widehat\phi^{(-k)}_{t,g}
+
\widehat\phi^{(-k)}_{t,\theta}$.
This completes the specification of
$\widehat H^{1,{\sf cf}}(t\mid \bX_0,\bZ_0)$
and
$\widehat S^{1,{\sf cf}}(t\mid \bX_0,\bZ_0)$
in (\ref{eq:H1cf})--(\ref{eq:S1cf}).

Supplementary Lemmas 13–15 control the centered cross-fitted score replacement and the estimated information, and establish $L_2(\PP_0)$-consistency of the estimated IF.

\begin{restatable}{theorem}{DiscreteCLT}
	\label{thm:cf_discrete_CLT}
    Fix $t\in[0,\tau]$ and let $(\bX_0,\bZ_0)$ be an independent draw from the covariate distribution. Under Assumptions \ref{ass:A1}--\ref{ass:A11}, and assuming moreover that for each fold $k$,
	$
	\|\widetilde g^{(-k)}-g_o\|_{L^2(\PP_0,\Delta)}
	=
	O_p(\gamma_n\log^2 n)$ where
	$
	\gamma_n\log^2 n=o(n^{-1/4})$,
	uniformly in $k$, 
	and that $\widetilde\btheta^{(-k)}$ satisfies the same first-order expansion as established in Theorem 4,
	namely,
	\[
	\sqrt{| \mathcal I_k^c|}\,(\widetilde\btheta^{(-k)}-\btheta_o)
	=
	I(\btheta_o)^{-1}\frac{1}{\sqrt{|\mathcal I_k^c|}}\sum_{i\in \mathcal I_k^c}\ell_{\theta_o}^*(\bW_i)
	+o_p(1),
	\]
	uniformly in k. Then,  conditional on the independent draw $(\bX_0,\bZ_0)$,
	\[
	\sqrt n\Big\{\widehat H^{\,1,{\sf cf}}(t\mid\bX_0,\bZ_0)-H_o(t\mid\bX_0,\bZ_0)\Big\} \mid (\bX_0,\bZ_0)
	\rightarrow^{\mathcal D}
	\mathcal N\!\big(0,\sigma_H^2(t;\bX_0,\bZ_0)\big),
	\]
	where
	$
	\sigma_H^2(t;\bX_0,\bZ_0)
	=
	\Var_0\!\Big(\phi_t^0(\bW;\bX_0,\bZ_0)\ \big|\ \bX_0,\bZ_0\Big).
	$
	Moreover, with
	\[
	\widehat\sigma_H^2(t;\bX_0,\bZ_0)
	=
	\frac1n\sum_{k=1}^{K_{\mathrm{cf}}}\ \sum_{i\in \mathcal I_k}
	\Big\{\widehat\phi_t^{(-k)}(\bW_i;\bX_0,\bZ_0)\Big\}^2,
	\]
	we have $\widehat\sigma_H^2(t;\bX_0,\bZ_0)\xrightarrow{p}\sigma_H^2(t;\bX_0,\bZ_0)$.
 Finally, 
	\[
	\sqrt n\Big\{\widehat S^{\,1,{\sf cf}}(t\mid\bX_0,\bZ_0)-S_o(t\mid\bX_0,\bZ_0)\Big\} \mid (\bX_0,\bZ_0)
	\rightarrow{^\mathcal D}
	\mathcal N\!\Big(0,\sigma_S^2(t;\bX_0,\bZ_0)\Big)
	\]
	where
	$
\sigma_S^2(t;\bX_0,\bZ_0)=S_o(t\mid\bX_0,\bZ_0)^2\,\sigma_H^2(t;\bX_0,\bZ_0)
	$
	and a consistent variance estimator is
	$\widehat\sigma_S^2(t;\bX_0,\bZ_0)
	=
	\big\{\widehat S^{\,1,{\sf cf}}(t\mid\bX_0,\bZ_0)\big\}^2\,
	\widehat\sigma_H^2(t;\bX_0,\bZ_0)$.
\end{restatable}
Theorem~\ref{thm:cf_discrete_CLT} yields a Wald-type pointwise confidence interval for the survival of a new subject at any fixed time $t\in[0,\tau]$, that is
$
\widehat S^{1,\mathrm{cf}}(t\mid \bX_0, \bZ_0)
\pm
z_{1-\alpha/2}{\widehat\sigma_S(t; \bX_0, \bZ_0)}/{\sqrt n}$. 
The result is conditional on the independent draw $(\bX_0,\bZ_0)$, rather than uniform over fixed covariate values. Consequently, the interval has asymptotic conditional coverage $1-\alpha$ given $(\bX_0,\bZ_0)$, and hence asymptotic marginal coverage
$1-\alpha$ under the joint distribution of the training sample and $(\bX_0,\bZ_0)$, namely,
$
\Pr \!\left(
S_o(t\mid \bX_0, \bZ_0)\in
\left[
\widehat S^{1,\mathrm{cf}}(t\mid \bX_0,\bZ_0)
\pm
z_{1-\alpha/2}\frac{\widehat\sigma_S(t;\bX_0,\bZ_0)}{\sqrt n}
\right]
\right)
\to 1-\alpha$.

The condition on $\gamma_n \log^2 n = o(n^{-1/4})$ is the usual second-order requirement for one-step inference,  ensuring $\sqrt n\,\|\widetilde g-g_o\|_{L^2(\PP_0,\Delta)}^2 = o_p(1)$, so that nuisance-estimation error enters only through $o_p(n^{-1/2})$ terms. For the slowest layer of composite H\"older representation, up to logarithmic factors, it requires
$
\frac{\tilde\alpha_\ell}{2\tilde\alpha_\ell+\tilde d_\ell}>\frac14$, equivalently,
$2\tilde\alpha_\ell>\tilde d_\ell$.
Thus the condition is mild when the intrinsic dimensions are small relative to the effective smoothness.

\section{Simulation Study and Real-data application}
\label{sec:simulation}

\subsection{Simulation study}
\label{sec:simulation-revised}

We evaluated estimation and inference for $\btheta_o$ under two designs. Because our primary simulation target is estimation and inference for the interpretable coefficient $\btheta_o$, we compare with \cite{zhong2022deep}, the existing DNN method that targets the corresponding partially linear parameter under PH. The experiments were designed for 200
Monte Carlo replications per configuration. The first design separates
non-proportional nuisance effects from dependence between the primary and
nuisance covariates; the second examines a two-dimensional primary effect
under non-PH with independent covariates. Code for all simulations and data analyses is available at \url{https://github.com/AsafBanana/Flexi_Haz/tree/main}.\\
{\bf Setting 1: dependence and non-proportionality.} 
This setting uses the Gompertz-based extended-hazard construction
considered by \cite{zhong2021deep}, augmented with a linear
primary-covariate effect and parameters controlling covariate
dependence and departures from PH.
Let $\bX=(X_1,\ldots,X_6)^\top$, with independent
$X_1,\ldots,X_5\sim U(0,1)$ and $X_6\sim\operatorname{Bernoulli}(1/2)$.
We generated
\[
 Z=\sqrt{\rho}\,s(\bX)+\sqrt{1-\rho}\,\varepsilon,
 \qquad
 s(\bX)=\sqrt{\frac{12}{5}}\sum_{j=1}^{5}(X_j-1/2),
 \qquad \varepsilon\sim N(0,1),
\]
where $\varepsilon$ is independent of $\bX$. Thus $\operatorname{Var}(Z)=1$,
$\operatorname{Var}(Z\mid \bX)=1-\rho$, and
$\operatorname{Corr}\{Z,s(\bX)\}=\sqrt{\rho}$. The choices
$\rho\in\{0,0.5,0.75\}$ correspond to correlations of approximately
$0$, $0.71$, and $0.87$. Increasing $\rho$ reduces the variation in $Z$
that remains after adjustment for $\bX$, making the primary effect harder
to separate from the nuisance component. Thus, varying $\rho$ also examines deterioration in information for $\theta_o$ as $\bZ$ becomes increasingly predictable from $\bX$. Event times followed
\[
 h_o(t\mid \bX,Z)=0.05\exp\{t e^{\gamma a(\bX)}+b(\bX)+\theta_o Z\},
 \qquad \theta_o=\log(1.5),
\]
where $a$ and $b$ are nonlinear interaction functions specified in
Supplementary File, Section A.
When $\gamma=0$, the time-dependent term is common to all subjects and
the model satisfies PH. When $\gamma=0.5$, the effect of $\bX$ varies with
time, while the coefficient of $Z$ remains constant. Crossing the two
values of $\gamma$ with the three values of $\rho$ therefore distinguishes
the consequences of PH misspecification from those of covariate
dependence. We considered $n\in\{4000,8000\}$, with independent censoring
calibrated to approximately $30\%$.\\
{\bf Setting 2: two primary covariates.}
We independently generated the components of $\bX\in\mathbb{R}^3$ and
$\bZ\in\mathbb{R}^2$ from $U[-1,1]$, set
$\btheta_o=(1,-1)^\top$, and used
\[
 h_o(t\mid \bX,\bZ)=0.1\exp\{[0.1+f(\bX)^2]t+\btheta_o^\top Z\},
 \qquad f(\bX)=0.2(X_1+X_2)+0.5X_1X_2+X_3^2.
\]
This design assesses estimation of two effects of opposite signs when
the nuisance effects are non-proportional even though $\bX$ and $\bZ$ are
independent. Sample sizes were $n\in\{2000,4000,8000\}$.
Independent exponential censoring with mean 30 and administrative
censoring at time 10 produced approximately $21\%$ censoring.

Table~\ref{tab:sim-main} reports mean estimates, empirical standard deviations, mean
estimated standard errors, and coverage of the $90\%$ and $95\%$ Wald
intervals; parentheses give $95\%$ Wald intervals for the coverage
probabilities. Tuning procedures, selected architectures, optimization
settings, and initialization rules are reported in Supplementary
File, Section A. Table~\ref{tab:sim-main} demonstrates the advantage of FLEXI--Haz when nuisance
covariate effects vary over time. In Setting~1, both methods generally
performed well under PH. Under non-PH, however, the proportional comparator exhibited increasing
bias as dependence between the primary and nuisance covariates
strengthened, whereas FLEXI--Haz maintained substantially smaller
bias and more stable coverage. For example, at $n=8000$ and
$\gamma=0.5$, coverage of the nominal $95\%$ intervals ranged from
$93.0\%$ to $95.0\%$ for FLEXI--Haz, compared with $74.5\%$ to
$85.5\%$ for the proportional comparator. The contrast was even stronger in Setting 2, where the proportional comparator showed bias and severe undercoverage. Increasing sample size reduced the comparator's
variability without eliminating its bias, resulting in increasingly
severe undercoverage. Although some finite-sample coverage
shortfalls remained for FLEXI--Haz, its estimated standard errors
generally tracked empirical variability, supporting the value of
flexible nuisance modeling for inference on the primary effects.

We also assessed pointwise survival inference in a separate experiment
similar to Setting 2, with $p=2$, $n=8000$ and independent Poisson(2) nuisance covariates truncated at 6,
giving finite covariate support. Results are based on  an independent evaluation sample of size 200.
For each $s\in\{0.1,0.2,\ldots,0.9\}$, a time $t_s$ was selected from the
training data so that mean estimated survival was approximately $s$.
Table~\ref{tab:survival-coverage} reports coverage of the proposed pointwise intervals for
$S_o(t_s\mid \bX_0,\bZ_0)$, which was generally close to the nominal level.

\begin{table}[!t]
\centering
\caption{Simulation results for estimation and inference on $\theta_o=0.4055$ for $p=1$, and $\btheta_o^\top=(1,-1)$ for $p=2$. Parentheses contain 95\% Wald confidence intervals for the empirical coverage probabilities, based on 200 Monte Carlo replications.}
\label{tab:sim-main}
\small
\setlength{\tabcolsep}{4.6pt}
\begin{tabular}{ccclccccc}
\toprule
$p$ & $n$ & $\gamma,\rho$ & Parameter/Method &
Est. & Emp.\ SD & Est.\ SE & 90\%CI  & 95\% CI \\
\midrule
1 & 4000 & $0,0$ & FLEXI--Haz
  & 0.404 & 0.026 & 0.024 & {\bf 87.5} (82.9, 92.1) & {\bf 94.0} (90.7, 97.3) \\
  &&& Zhong et al. & 0.412 & 0.029 & 0.024 & 82.0 (76.7, 87.3) & 90.5 (86.4, 94.6) \\
1 & 4000 & $0,0.5$ & FLEXI--Haz
  & 0.404 & 0.035 & 0.033 & {\bf 88.0} (83.5, 92.5) & {\bf 94.5} (91.3, 97.7) \\
  &&& Zhong et al. & 0.393 & 0.034 & 0.033 & {\bf 85.5} (80.6, 90.4) & {\bf 93.0} (89.5, 96.5) \\
1 & 4000 & $0,0.75$ & FLEXI--Haz
  & 0.405 & 0.050 & 0.046 & {\bf 86.5} (81.8, 91.2) & {\bf 95.0} (92.0, 98.0) \\
  &&& Zhong et al. & 0.399 & 0.044 & 0.046 & {\bf 91.0} (87.0, 95.0) & {\bf 96.0} (93.3, 98.7) \\
1 & 4000 & $0.5,0$ & FLEXI--Haz
  & 0.398 & 0.026 & 0.023 & {\bf 85.5} (80.6, 90.4) & {\bf 91.0} (87.0, 95.0) \\
  &&& Zhong et al. & 0.391 & 0.026 & 0.024 & 80.5 (75.0, 86.0) & 90.0 (85.8, 94.2) \\
1 & 4000 & $0.5,0.5$ & FLEXI--Haz
  & 0.397 & 0.036 & 0.033 & {\bf 86.0} (81.2, 90.8) & {\bf 91.5} (87.6, 95.4) \\
  &&& Zhong et al. & 0.369 & 0.034 & 0.033 & 67.5 (61.0, 74.0) & 77.0 (71.2, 82.8) \\
1 & 4000 & $0.5,0.75$ & FLEXI--Haz
  & 0.397 & 0.051 & 0.046 & {\bf 86.5} (81.8, 91.2) & {\bf 92.0} (88.2, 95.8) \\
  &&& Zhong et al. & 0.354 & 0.043 & 0.046 & 71.0 (64.7, 77.3) & 82.0 (76.7, 87.3) \\
\midrule
1 & 8000 & $0,0$ & FLEXI--Haz
  & 0.400 & 0.017 & 0.017 & {\bf 88.0} (83.5, 92.5) & {\bf 96.0} (93.3, 98.7) \\
  &&& Zhong et al. & 0.410 & 0.018 & 0.017 & {\bf 88.5} (84.1, 92.9) & {\bf 94.0} (90.7, 97.3) \\
1 & 8000 & $0,0.5$ & FLEXI--Haz
  & 0.397 & 0.023 & 0.023 & {\bf 90.0} (85.8, 94.2) & {\bf 95.5} (92.6, 98.4) \\
  &&& Zhong et al. & 0.400 & 0.023 & 0.024 & {\bf 89.5} (85.3, 93.7) & {\bf 95.0} (92.0, 98.0) \\
1 & 8000 & $0,0.75$ & FLEXI--Haz
  & 0.400 & 0.028 & 0.033 & {\bf 94.0} (90.7, 97.3) & {\bf 98.0} (96.1, 99.9) \\
  &&& Zhong et al. & 0.400 & 0.039 & 0.032 & 80.0 (74.5, 85.5) & 86.5 (81.8, 91.2) \\
1 & 8000 & $0.5,0$ & FLEXI--Haz
  & 0.402 & 0.018 & 0.017 & {\bf 87.0} (82.3, 91.7) & {\bf 95.0} (92.0, 98.0) \\
  &&& Zhong et al. & 0.389 & 0.017 & 0.016 & 75.0 (69.0, 81.0) & 85.5 (80.6, 90.4) \\
1 & 8000 & $0.5,0.5$ & FLEXI--Haz
  & 0.400 & 0.026 & 0.024 & {\bf 86.0} (81.2, 90.8) & {\bf 94.0} (90.7, 97.3) \\
  &&& Zhong et al. & 0.374 & 0.020 & 0.023 & 60.5 (53.7, 67.3) & 75.0 (69.0, 81.0) \\
1 & 8000 & $0.5,0.75$ & FLEXI--Haz
  & 0.392 & 0.032 & 0.032 & {\bf 88.5} (84.1, 92.9) & {\bf 93.0} (89.5, 96.5) \\
  &&& Zhong et al. & 0.361 & 0.028 & 0.032 & 61.5 (54.8, 68.2) & 74.5 (68.5, 80.5) \\
\midrule

2 & 2000 & & $\theta_{o1}$, FLEXI--Haz & 0.987 & 0.052 & 0.057 & \textbf{89.5} (85.3, 93.7) & \textbf{95.5} (92.6, 98.4) \\
  &&& $\theta_{o1}$, Zhong et al. & 0.912 & 0.057 & 0.057 & 50.5 (43.6,57.4) & 63.0 (56.3,69.7) \\
  &&& $\theta_{o2}$, FLEXI--Haz & -0.979 & 0.056 & 0.057 & \textbf{87.5} (82.9, 92.1) & \textbf{95.0} (92.0, 98.0) \\
  &&& $\theta_{o2}$, Zhong et al. & -0.906 & 0.063 & 0.056 & 46.5 (39.6,53.4) & 58.5 (51.7,65.3) \\
\midrule
2 & 4000 & & $\theta_{o1}$, FLEXI--Haz & 1.001 & 0.046 & 0.041 & {\bf 86.5} (81.8, 91.2) & {\bf 92.0} (88.2, 95.8) \\
  &&& $\theta_{o1}$, Zhong et al. & 0.907 & 0.041 & 0.040 & 25.0 (19.0,31.0) & 38.5 (31.8,45.2) \\
  &&& $\theta_{o2}$, FLEXI--Haz & -0.997 & 0.042 & 0.041 & {\bf 87.5} (82.9, 92.1) & {\bf 95.0} (92.0, 98.0) \\
  &&& $\theta_{o2}$, Zhong et al. & -0.904 & 0.037 & 0.040 & 20.5 (14.9,26.1) & 31.5 (25.1,37.9) \\
\midrule
2 &  8000 & & $\theta_{o1}$, FLEXI--Haz
  & 0.993 & 0.029 & 0.029
  & \textbf{89.5} (85.3, 93.7)
  & \textbf{94.0} (90.7, 97.3) \\
  &&& $\theta_{o1}$, Zhong et al.
  & 0.900 & 0.026 & 0.029
  & 2.0 (0.1, 3.9)
  & 6.5 (3.1, 9.9) \\
  &&& $\theta_{o2}$, FLEXI--Haz
  & -0.992 & 0.029 & 0.029
  & \textbf{88.0} (83.5, 92.5)
  & \textbf{93.5} (90.1, 96.9) \\
  &&& $\theta_{o2}$, Zhong et al.
  & -0.900 & 0.026 & 0.029
  & 1.5 ($-0.2$, 3.2)
  & 5.5 (2.3, 8.7) \\
\bottomrule
\end{tabular}
\end{table}

\begin{table}[!htbp]
\centering
\caption{Empirical pointwise coverage of the proposed confidence intervals for
the survival function, based on 200 replications with training sample size
$n=8000$.}
\label{tab:survival-coverage}
\small
\setlength{\tabcolsep}{3.2pt}
\renewcommand{\arraystretch}{0.98}
\begin{tabular}{cccccc}
\toprule
Marginal survival & $t_s$ & 90\% CI & SD & 95\% CI & SD \\
\midrule
0.1 & 6.7126 & 0.9377 & 0.0741 & 0.9676 & 0.0531 \\
0.2 & 4.2934 & 0.9120 & 0.0861 & 0.9548 & 0.0638 \\
0.3 & 3.0192 & 0.9167 & 0.0769 & 0.9565 & 0.0595 \\
0.4 & 2.2003 & 0.9121 & 0.0755 & 0.9595 & 0.0523 \\
0.5 & 1.6035 & 0.8903 & 0.0889 & 0.9488 & 0.0589 \\
0.6 & 1.2330 & 0.8997 & 0.0912 & 0.9520 & 0.0611 \\
0.7 & 0.9476 & 0.8995 & 0.0824 & 0.9480 & 0.0614 \\
0.8 & 0.6812 & 0.9123 & 0.0815 & 0.9577 & 0.0559 \\
0.9 & 0.3900 & 0.9044 & 0.0987 & 0.9567 & 0.0638 \\
\bottomrule
\end{tabular}
\end{table}

\subsection{Real-data application: Rotterdam breast cancer study}

We analyzed the Rotterdam breast cancer data, available in the R package
\texttt{survival}, comprising 2982 patients followed from surgery for
recurrence or death, with a 42.56\% right-censoring rate. Our primary question concerned the adjusted
associations of hormonal therapy (Ho) and chemotherapy (Ch) with the
event hazard. We set $\bZ=(\mathrm{Ho},\mathrm{Ch})^\top$ and included
age, menopausal status, tumor size, tumor grade, number of positive lymph
nodes, progesterone receptor level, and estrogen receptor (ER) level in
$\bX$. Additional network-fitting details are available in the Supplementary File, Section A.

In Table~\ref{tab:rotterdam_treatments1}, FLEXI--Haz estimated negative associations for both hormonal
therapy ($-0.364$; $95\%$ CI: $-0.58$ to $-0.15$) and chemotherapy
($-0.321$; $95\%$ CI: $-0.55$ to $-0.10$). The hormonal-therapy estimate
was similar under \cite{zhong2022deep}, whereas their chemotherapy estimate was closer to zero and its interval included zero. The Cox PH estimates were closer to zero for both treatments, and both
$95\%$ confidence intervals included zero.

\begin{table}[!htbp]
\centering
\caption{Estimated treatment effects and \(95\%\) confidence intervals for the Rotterdam  data.}
\label{tab:rotterdam_treatments1}
\small
\setlength{\tabcolsep}{3.2pt}
\renewcommand{\arraystretch}{0.98}
\begin{tabular}{lccc}
\hline
Treatment  & FLEXI--Haz & \cite{zhong2022deep} & Cox PH \\
\hline
Hormonal 
& $\mathbf{-0.364\,(-0.58,\ -0.15)}$
& $\mathbf{-0.327\,(-0.54,\ -0.11)}$
& $-0.132\,(-0.29,\ 0.02)$ \\

Chemo
& $\mathbf{-0.321\,(-0.55,\ -0.10)}$
& $-0.205\,(-0.43,\ 0.02)$
&  $-0.091\,(-0.23,\ 0.05)$\\

\hline
\end{tabular}
\end{table}

We next examined whether treatment associations differed across ER
levels, motivated by the established role of ER in predicting benefit
from endocrine therapy \citep{EBCTCG2011}. Defining
$L=\mathbf{1}(\mathrm{ER}<61)$ using the sample median, we replaced $\bZ$ by
$
 \bZ=\bigl(\mathrm{Ho}L,\mathrm{Ho}(1-L),
         \mathrm{Ch}L,\mathrm{Ch}(1-L)\bigr)^\top
$.
This exploratory split does not represent a clinical receptor-positivity
threshold. ER remained continuous in $\bX$; thus the treatment interactions
and nuisance covariates are dependent by construction, providing a
practical example in which independence between $\bZ$ and $\bX$ is untenable.

Table~\ref{tab:rotterdam_treatments} shows more negative hormonal-therapy estimates in the high-ER
group, with intervals excluding zero in that group under all three
methods. FLEXI--Haz chemotherapy estimates were similar across ER groups
($-0.322$ and $-0.307$), and both intervals excluded zero, whereas the
corresponding intervals from the two PH models included zero. Formal
evidence of interaction requires direct contrasts between ER-specific
coefficients. These observational results concern adjusted associations;
the methodological contribution is the ability to examine clinically
motivated treatment contrasts while accommodating flexible temporal
nuisance effects and dependence between the primary and nuisance
covariates.

\begin{table}[!htbp]
\centering
\caption{Estimated ER-specific treatment effects and \(95\%\) confidence intervals for the Rotterdam breast cancer data.}
\label{tab:rotterdam_treatments}
\small
\setlength{\tabcolsep}{3.2pt}
\renewcommand{\arraystretch}{0.98}
\begin{tabular}{llccc}
\hline
Treatment & ER & FLEXI--Haz & \cite{zhong2022deep} & Cox PH \\
\hline
Hormonal & Low
& $-0.197\,(-0.49,\ 0.10)$
& $-0.198\,(-0.50,\ 0.10)$
& $0.093\,(-0.12,\ 0.30)$ \\

Hormonal & High
& $\mathbf{-0.569\,(-0.87,\ -0.27)}$
& $\mathbf{-0.493\,(-0.78,\ -0.21)}$
& $\mathbf{-0.318\,(-0.53,\ -0.11)}$ \\

Chemo & Low
& $\mathbf{-0.322\,(-0.60,\ -0.05)}$
& $-0.223\,(-0.48,\ 0.03)$
& $-0.094\,(-0.26,\ 0.07)$ \\

Chemo & High
& $\mathbf{-0.307\,(-0.59,\ -0.02)}$
& ${-0.241\,(-0.51,\ 0.03)}$
& $-0.093\,(-0.29,\ 0.10)$ \\
\hline
\end{tabular}
\end{table}

Finally, Figure~\ref{fig:survival-curvesRotData} presents estimated survival curves and pointwise
confidence intervals for the 16 most frequent profiles of the grouped adjustment covariates $\bX$, conditional also on
 $\bZ=(\mathrm{Ho},\mathrm{Ch})^\top$. Grouping 4 continuous or count-valued nuisance covariates and retaining 3 original categorical variables yielded 590 distinct observed
profiles among 2982 patients. The displayed profiles contain 505 patients (17\% of the cohort), with 23--46 patients and 4--26 events per profile. Grouping rules and profile counts are provided in Supplementary File Section A. 

\begin{figure}[!htbp]
\centering
\includegraphics[width=0.75\linewidth]{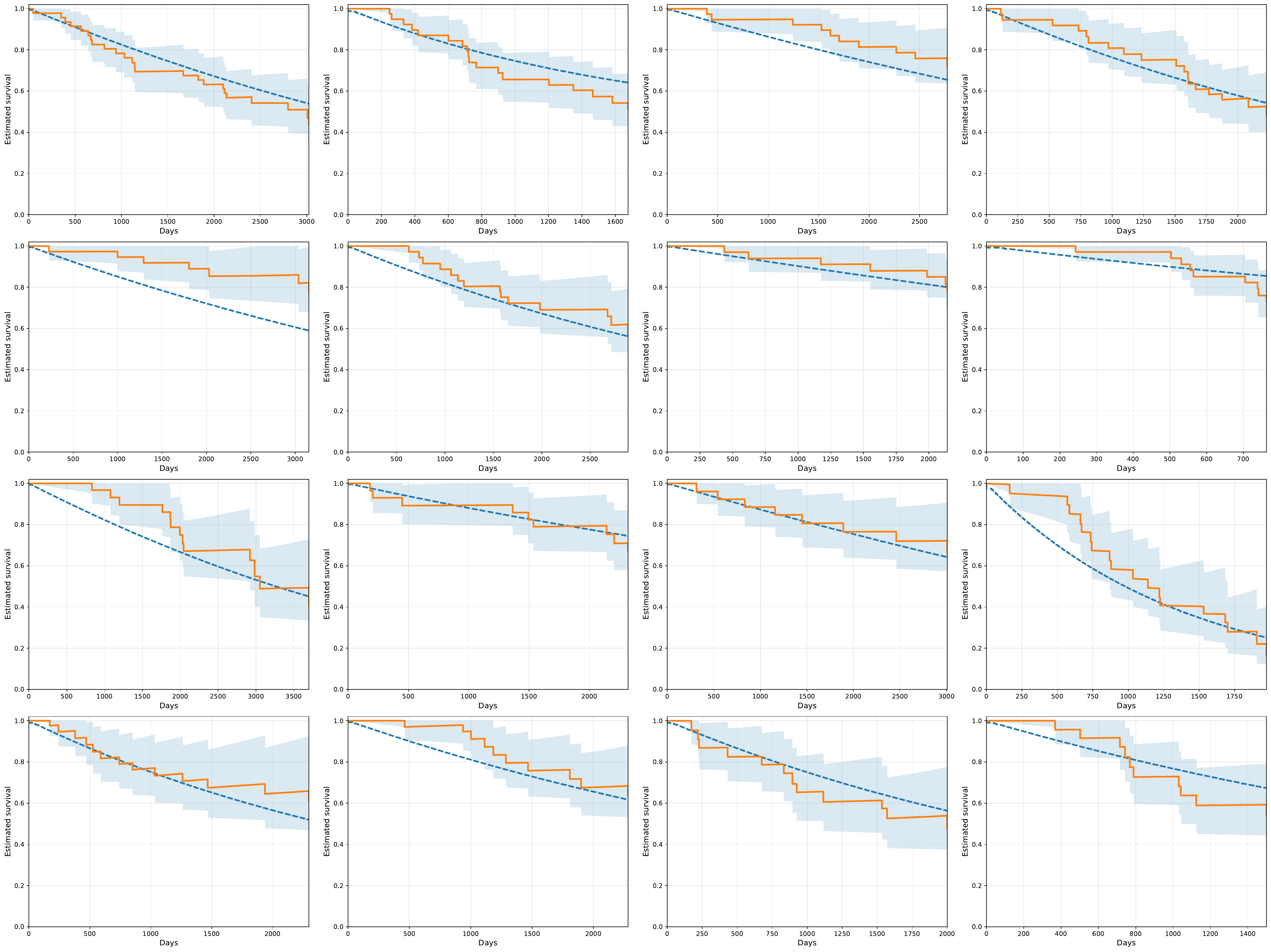}
\caption{Estimated survival curves for 16 Rotterdam patients,
conditional on their grouped nuisance covariates and primary covariates.
Dashed blue curves show the plug-in FLEXI--Haz estimates; orange step
curves show the cross-fitted one-step estimates. Shading represents
pointwise confidence intervals.}
\label{fig:survival-curvesRotData}
\end{figure}

\section{Discussion}
FLEXI--Haz combines an interpretable low-dimensional regression component
with flexible time-varying nuisance effects. We established minimax-optimal
nuisance estimation, semiparametrically efficient $\sqrt n$-inference for
$\btheta_o$, and valid pointwise survival inference via a cross-fitted
one-step estimator. Interpretation of $\btheta_o$ relies on correct specification
of the partially linear structure, particularly a time-invariant conditional
effect of $\bZ$; developing assumption-lean targets under misspecification
is an interesting direction for future work.
A limitation of the survival-inference theory is the finite-support
assumption on $\bX$, which yields an explicit Riesz representer and tractable
one-step correction. An extension to general covariate distributions is an
important direction for future work.

{ \bf Generative AI Use.} The authors used ChatGPT to improve clarity and identify potential mathematical inconsistencies. The authors take full responsibility for the manuscript.

\bibliography{bib}       

\clearpage
\setcounter{page}{1}

\setlength{\abovedisplayskip}{5pt}
\setlength{\belowdisplayskip}{5pt}
\setlength{\abovedisplayshortskip}{3pt}
\setlength{\belowdisplayshortskip}{3pt}

\begin{center}
{\Large  Supplementary File of \\ 
    FLEXI--Haz: Semiparametric Deep Survival Inference Beyond Proportional Hazards\\

\vspace{0.5cm}


}
\end{center}

\begin{spacing}{0.90}

This supplementary file provides additional technical details. Appendix A provides additional details of the simulation study.
Appendix B contains the lemmas, theorems and proofs for the results of Section 4, and Appendix C contains the lemmas, theorems and proofs for the results of Section 5. The
numbering of theorems and lemmas that are restated from the main text is kept
the same as in the paper. Since several  lemmas appear only in the
supplementary material, the numbering in this file is not always consecutive.
Appendix D gives high-level results for the one-step survival estimator in the
general setting of $\bX_0$.

\begin{appendix}
\section{Additional details of the simulation study}
\label{app:simulation-details}

This section specifies the data-generating mechanisms and computational
settings for the experiments in Section~6.1. An architecture $(w_1,\ldots,w_L)$
lists the widths of its $L$ hidden layers. Unless stated otherwise,
networks use ReLU activations, with no dropout or weight decay.

\subsection{Setting 1: a scalar primary covariate}

\subsubsection{Data generation}
The covariates, coefficient $\theta_o=\log(1.5)$, and dependence
construction are defined in Section~6.1. Writing
$\widetilde X_j=X_j-1/2$, the two nuisance functions are
\begin{align*}
 a(\bX)&=6c_a\bigl(\widetilde X_1^2\widetilde X_2
             -\widetilde X_2^2\widetilde X_3
             +\widetilde X_3^2\widetilde X_4
             -\widetilde X_4^2\widetilde X_5\bigr),\\
 b(\bX)&=5.5c_b\bigl(\widetilde X_1^2\widetilde X_3
               +\widetilde X_2^2\widetilde X_4
               +\widetilde X_3^2\widetilde X_5
               +\widetilde X_1^2\widetilde X_5\bigr).
\end{align*}
For $\gamma=0$, $c_b=1.542799$ and $c_a$ is immaterial.
For $\gamma=0.5$, $(c_a,c_b)=(1.574509,1.090924)$.
These constants were calibrated using $2\times10^6$ Monte Carlo draws
and then fixed across replications. The two functions use different
collections of interaction pairs; $X_6$ is included among the nuisance
covariates but does not enter the event hazard. The primary-covariate
construction satisfies
\[
 \frac{\operatorname{Var}\{E(Z\mid \bX)\}}{\operatorname{Var}(Z)}=\rho.
\]
Thus $\rho$ is an explained-variance fraction, rather than the correlation
itself. The values $\gamma\in\{0,0.5\}$ and
$\rho\in\{0,0.5,0.75\}$ give six configurations at each of
$n=4000$ and $n=8000$.

For the hazard in Section~6.1, the cumulative hazard is
\[
 H_o(t\mid \bX,Z)=0.05\exp\{b(\bX)+\theta_o Z-\gamma a(\bX)\}
                  [\exp\{t e^{\gamma a(\bX)}\}-1].
\]
Consequently, for $E\sim\operatorname{Exp}(1)$ independent of the
covariates, exact event times are obtained as
\[
 U=e^{-\gamma a(\bX)}
       \log\left[1+20E\exp\{\gamma a(\bX)-b(\bX)-\theta_o Z\}\right].
\]
Independent censoring is calibrated separately for each $(\gamma,\rho)$
to produce approximately $30\%$ censoring; the calibrated mechanism is
then fixed across replications. We observe
$T=\min(U,C)$ and $\Delta=\mathbf{1}(U\le C)$, and use
$\tau=\max_i T_i$ for rescaling time in the network.

\subsubsection{Architecture selection}
The main network takes $(t/\tau,X)$ as input. Its architecture was
selected by a subject-level $67\%/33\%$ training/validation split over
35 equal-width candidates:
\[
 L\in\{2,3,4,5,6\},\qquad
 w\in\{25,50,100,150,200,300,500\}.
\]
Each candidate was trained with the three-phase learning-rate schedule
specified below and scored by its best validation negative
log-likelihood. The learning-rate schedule was fixed, rather than
tuned. The selected architectures were held fixed across the simulation
replications and are listed in Table~\ref{tab:setting1-architectures}.

\begin{table}[htbp]
\centering
\caption{Selected main-network architectures for Setting 1.
Each entry $w\times L$ denotes $L$ hidden layers with $w$ units per layer.
LR denotes the selected learning rate for Zhong et al.}
\label{tab:setting1-architectures}
\begin{tabular}{ccrrrrrr}
\toprule
& & \multicolumn{3}{c}{$n=4000$} & \multicolumn{3}{c}{$n=8000$}\\
\cmidrule(lr){3-5}\cmidrule(lr){6-8}
$\gamma$ & $\rho$
& FLEXI--Haz & Zhong et al. & LR
& FLEXI--Haz & Zhong et al. & LR\\
\midrule
0   & 0    & $100\times2$ & $25\times3$  & $0.05$
             & $150\times2$ & $50\times2$  & $0.10$\\
0   & 0.5  & $150\times2$ & $500\times5$ & $0.01$
             & $150\times2$ & $300\times2$ & $0.005$\\
0   & 0.75 & $150\times2$ & $500\times6$ & $0.005$
             & $150\times2$ & $25\times2$  & $0.10$\\
0.5 & 0    & $300\times2$ & $150\times4$ & $0.005$
             & $300\times3$ & $100\times2$ & $0.01$\\
0.5 & 0.5  & $100\times2$ & $500\times2$ & $0.0001$
             & $150\times2$ & $300\times5$ & $0.001$\\
0.5 & 0.75 & $150\times2$ & $100\times3$ & $0.05$
             & $150\times2$ & $300\times5$ & $0.001$\\
\bottomrule
\end{tabular}
\end{table}

\subsubsection{Optimization and initialization}
Within each replication, subjects were divided into $67\%$ for training
and $33\%$ for validation, keeping all expanded observations from the
same subject together. Covariates were standardized using the training
subjects and time was rescaled by $\tau$. The linear coefficient was
initialized from a Cox model containing both $\bX$ and $\bZ$, fitted to the
training subjects. Coefficients and standard errors are reported on the
original covariate scale.

The main network was optimized by Adam over three warm-started phases
with learning rates $10^{-3}$, $10^{-4}$, and $10^{-5}$. Each phase used
at most 500 epochs and patience 30. The best validation loss was carried
across phases, with the best weights restored. For $n=4000$, the linear
coefficient had a separate Adam parameter group with learning rate
20 times the current network rate; for $n=8000$, the two rates were equal.
At each phase boundary, the nuisance network was held fixed and the
coefficient was numerically re-optimized on the expanded training data
using L-BFGS-B with an analytic gradient. The next Adam phase started
from this refined coefficient.

If $N_{\mathrm{exp,tr}}$ denotes the number of expanded training rows,
the main-network batch size was
$\lceil N_{\mathrm{exp,tr}}/213\rceil$. This keeps the number of batches
per epoch approximately fixed as sample size changes. The projection
batch size was $\lceil N_{\mathrm{ev,in}}/8\rceil$, where
$N_{\mathrm{ev,in}}$ is the number of event subjects in its inner
training sample. Table~\ref{tab:setting1-training} summarizes the
remaining constants.

\begin{table}[htbp]
\centering
\caption{Fixed training settings for FLEXI--Haz in Setting 1.}
\label{tab:setting1-training}
\begin{tabular}{lll}
\toprule
Component & Parameter & Value\\
\midrule
Main network & Optimizer & Adam\\
& Learning rates & $10^{-3}\to10^{-4}\to10^{-5}$\\
& Epoch limit / patience & 500 per phase / 30\\
& Batch size & $\lceil N_{\mathrm{exp,tr}}/213\rceil$\\
& Initializations per replication & 5\\
Projection network & Folds & 5, by subject\\
& Optimizer / loss & Adam / squared error\\
& Learning rate & $10^{-3}$\\
& Inner training/validation split & $80\%/20\%$ of event subjects\\
& Epoch limit / patience & 300 / 30\\
& Batch size & $\lceil N_{\mathrm{ev,in}}/8\rceil$\\
Both components & Dropout / weight decay & 0 / 0\\
\bottomrule
\end{tabular}
\end{table}

\subsubsection{Projection, restart selection, and standard errors}
For the event-weighted projection, define
$m(t,\bx)=E(Z\mid T=t,\bX=\bx,\Delta=1)$.
Training subjects were assigned to five folds. For each fold, a network
was fitted by squared-error loss using event subjects from the other
four folds, with an inner $80\%/20\%$ split for early stopping.
The resulting prediction $\widehat m^{(-k(i))}(t,\bX_i)$ was used for
subject $i$ in fold $k(i)$. Residuals on the expanded rows were
\[
 r_{ij}=Z_i-\widehat m^{(-k(i))}(t_j,\bX_i).
\]
The projection fits were computed once per dataset and reused across
the five main-network initializations. For initialization $k$, with
fit $(\widehat\theta_k,\widehat g_k)$, the selection score was
\[
 S_k=\sum_{i,j}r_{ij}
       \left[\ell_{ij}
          e^{\widehat\theta_kZ_i+\widehat g_k(t_j,\bX_i)}-\Delta_{ij}\right],
\]
where $\ell_{ij}$ is the observed time at risk in interval $j$ and
$\Delta_{ij}$ is its event indicator. The fit with the smallest
$|S_k|$ was retained. Thus the reported coefficient is the selected
fitted coefficient.

With $n_{\mathrm{tr}}$ training subjects, let
$R_i=\Delta_i\{Z_i-\widehat m^{(-k(i))}(T_i,\bX_i)\}$. The information
and variance estimates were
\[
 \widehat I=\frac{1}{n_{\mathrm{tr}}}\sum_i R_i^2 \, ,
 \qquad
 \widehat{\operatorname{Var}}(\widehat\theta)
       =\frac{1}{n_{\mathrm{tr}}\widehat I} \, .
\]

\subsection{Setting 2: two primary covariates}

\subsubsection{Data generation and censoring}
The five components of $(\bX^\top,\bZ^\top)^\top$ were independent
$U[-1,1]$ variables, with $\bX\in\mathbb{R}^3$, $\bZ\in\mathbb{R}^2$,
and $\btheta_o=(1,-1)^\top$. Define
\[
 f(\bX)=0.2(X_1+X_2)+0.5X_1X_2+X_3^2,
 \qquad q(\bX)=0.1+f(\bX)^2>0.
\]
The event hazard and cumulative hazard were
\[
 h_o(t\mid \bX,\bZ)=0.1 \exp\{q(\bX)t+\btheta_o^\top \bZ\} \, ,\qquad
 H_o(t\mid \bX,\bZ)=0.1 \exp\{\btheta_o^\top \bZ\}
                         \frac{e^{q(\bX)t}-1}{q(\bX)} \, .
\]
For $E\sim\operatorname{Exp}(1)$, event times were generated exactly as
\[
 U=\frac{1}{q(\bX)}
       \log\{1+10E q(\bX)e^{-\btheta_o^\top \bZ}\}.
\]
Independently, $C$ was exponential with mean 30. The observed data were
\[
 T=\min(U,C,10),\qquad
 \Delta=\mathbf{1}\{U\le C,\ U\le10\}.
\]
Across the 200 saved datasets at each sample size, mean censoring was
$20.86\%$, $20.88\%$, and $20.83\%$ for
$n=2000$, 4000, and 8000, respectively.

\subsubsection{FLEXI--Haz architecture selection}
For each sample size, main-network architectures were selected
using independent tuning datasets and five common subject-level folds.
All 27 architectures in Table~\ref{tab:setting2-grid} were evaluated,
giving 135 fits at each sample size. Within each outer training fold,
an additional $67\%/33\%$ subject split was used for early stopping;
the outer held-out fold was used to evaluate the fitted likelihood.
The architecture minimizing mean outer-fold negative log-likelihood
was selected, with initial learning rate fixed at $10^{-3}$.

\begin{table}[htbp]
\centering\small
\caption{The 27 candidate hidden-layer configurations used in Setting 2.}
\label{tab:setting2-grid}
\begin{tabular}{cp{0.78\linewidth}}
\toprule
Depth & Width sequences\\
\midrule
1 & $(16)$, $(32)$, $(64)$, $(128)$, $(256)$\\
2 & $(16,16)$, $(16,32)$, $(32,32)$, $(32,64)$,
    $(64,64)$, $(64,128)$, $(128,128)$\\
3 & $(16,16,16)$, $(16,16,32)$, $(32,32,32)$,
    $(32,32,64)$, $(64,64,64)$, $(64,64,128)$, $(128,128,128)$\\
4 & $(16,16,16,16)$, $(16,16,32,32)$, $(32,32,32,32)$,
    $(32,32,64,64)$, $(64,64,64,64)$\\
5 & $(16,16,16,16,16)$, $(16,16,32,32,32)$,
    $(32,32,32,32,32)$\\
\bottomrule
\end{tabular}
\end{table}

\begin{table}[htbp]
\centering
\caption{Architectures and learning rates used in Setting 2.
For FLEXI--Haz, the projection networks use the same hidden widths as
the main network; their learning rate is fixed at $10^{-3}$.
FLEXI--Haz uses the displayed three-phase schedule, whereas the
proportional comparator uses a constant selected learning rate.}
\label{tab:setting2-selected}
\begin{tabular}{rlll}
\toprule
$n$ & Method & Hidden widths & Learning rate(s)\\
\midrule
2000 & FLEXI--Haz & $(16,16,16)$ & $10^{-3}\to10^{-4}\to10^{-5}$\\
4000 & FLEXI--Haz & $(32,32,32)$ & $10^{-3}\to10^{-4}\to10^{-5}$\\
8000 & FLEXI--Haz & $(128,128,128)$ & $10^{-3}\to10^{-4}\to10^{-5}$\\
\midrule
2000 & Zhong et al. & $(32,32)$ & $0.03222255749$\\
4000 & Zhong et al. & $(16,16,16,16)$ & $0.02204518664$\\
8000 & Zhong et al. & $(64,64)$ & $0.04726183068$\\
\bottomrule
\end{tabular}
\end{table}

\subsubsection{FLEXI--Haz fitting and inference}
The main network used a $67\%/33\%$ subject-level training/validation
split, a batch size of 100,000 expanded rows, and three warm-started
Adam phases with learning rates $10^{-3}$, $10^{-4}$, and $10^{-5}$.
Each phase allowed at most 2000 epochs with patience 30.
The integration grid combined $n$ equally spaced points on $[0,10]$
with all observed event and censoring times. Time was rescaled by 10.
Covariate scaling and Cox initialization were computed from all subjects
supplied to a fit, before the internal
training/validation split. 
At each phase boundary, L-BFGS-B numerically refined the coefficient
with the nuisance network held fixed, using all expanded observations
supplied to the fit. 

The projection step used only the internal training subjects, giving
$n_{\mathrm{tr}}=1340$, 2680, and 5360 at the three sample sizes.
For each coordinate of $\bZ$, a separate network regressed that coordinate
on $(T/10,\bX)$ among observed events. Five folds were used, with Adam
learning rate $10^{-3}$, batch size 128, at most 2000 epochs, and
patience 30. The hidden widths were inherited from the main-network
architecture. Each held-out projection fold was used both to monitor
early stopping and to obtain predictions.
Writing ${\bf R}_i=\Delta_i\{\bZ_i-\widehat m^{(-k(i))}(T_i,\bX_i)\}$ on the
standardized scale, the implemented covariance was
\[
 \widehat V_{\mathrm{std}}
 =\frac{1}{n_{\mathrm{tr}}}
   \left\{\frac{1}{n_{\mathrm{tr}}}\sum_i {\bf R}_i {\bf R}_i^\top
                                      +10^{-8}I_2\right\}^{-1},
\]
where the Moore--Penrose inverse is used if needed. If $D$ contains the
standard deviations used to scale $\bZ$, the reported covariance is
$D^{-1}\widehat V_{\mathrm{std}}D^{-1}$.

\subsection{\cite{zhong2022deep} fitting: network specification and fitting}
\label{app:zhong-networks}
In both settings, the comparator based on \cite{zhong2022deep} uses a feedforward network taking $\bX$ as input, with ReLU
activations, no dropout, and no weight decay.  The architecture-selection and optimization
protocols differed between the two settings as follows.

{\bf Setting 1.}
Architectures were selected separately for each $(n,\gamma,\rho)$ from
the 35 equal-width candidates, using the
$67\%/33\%$ subject-level training/validation split and validation Cox
negative partial log-likelihood. The selected architectures and learning rates are
in Table~\ref{tab:setting1-architectures} and were fixed across
replications. The comparator used the same covariate standardization and Cox initialization, and an Adam schedule with at most 500 epochs per phase, patience 30, and restoration of the best validation weights. At each phase boundary, the
coefficient was refined with the network held fixed by maximizing the
Cox partial likelihood.
Five initializations were fitted per
replication, and the fit with the smallest validation negative partial
log-likelihood was retained. Thus restart selection used validation
loss, whereas FLEXI--Haz used the efficient-score criterion. 

{\bf Setting 2.}
Hyperparameters were tuned
separately at each sample size on an independent dataset using 100
Optuna tree-structured Parzen estimator trials over the 27 architectures
in Table~\ref{tab:setting2-grid} and learning rates sampled
logarithmically from $[0.001,0.05]$. Each trial used five-fold CV, with
folds reshuffled between trials, and minimized mean validation Cox
negative partial log-likelihood. The validation fold also determined
early stopping, and all five fits were required to succeed. Selected
architectures and learning rates are given in
Table~\ref{tab:setting2-selected}. Each replication used one
initialization, raw covariates, and Cox initialization on $(\bX,\bZ)$ using
the training subjects. Optimization used full-batch Adam with nominal
batch size 100,000, a constant selected learning rate, at most 2000
epochs, patience 30, and restoration of the best validation weights.
The L1 and Cox penalties were zero.

\subsection{Rotterdam data: network fitting and aggregation}
\label{app:rotterdam-fitting}

For each method and each specification of $\bZ$, the analysis used five
subject-level folds, with hyperparameters selected separately for each
fold. Network fitting followed the corresponding Setting~1, including covariate standardization, Cox
initialization, early stopping, and coefficient refinement with the nuisance network held fixed. FLEXI--Haz used the three-phase learning-rate
schedule $10^{-3}\to10^{-4}\to10^{-5}$; the selected Zhong et al. learning
rates are reported in Table~\ref{tab:rotterdam-architectures}.
Five initializations were fitted for each fold-specific configuration.
FLEXI--Haz retained the fit with the smallest efficient-score norm,
whereas Zhong et al.\ retained the fit with the smallest validation Cox
negative partial log-likelihood. Coefficient refinement used the
respective full and partial likelihoods.
The reported point estimate was the average of the five fold-specific coefficient estimates.
All coefficients are expressed on the original covariate scale.
Table~\ref{tab:rotterdam-architectures} gives the selected architectures
for both analyses and both methods.

\begin{table}[htbp]
\centering\small
\caption{Selected main-network architectures for the Rotterdam analyses.
An architecture lists the widths of successive hidden layers.
LR is the reported Zhong main-network learning rate. Fold labels 1--5
correspond to labels 0--4 in the supplied configuration file.}
\label{tab:rotterdam-architectures}
\begin{tabular}{lcllr}
\toprule
Analysis & Fold & FLEXI--Haz & Zhong et al. & LR\\
\midrule
Overall effects & 1 & $(200,200)$ & $(25,25,25)$ & $0.01$\\
& 2 & $(150,150)$ & $(25,25,25)$ & $0.01$\\
& 3 & $(100,100)$ & $(100,100)$ & $0.0005$\\
& 4 & $(200,200)$ & $(25,25,25)$ & $0.0005$\\
& 5 & $(100,100)$ & $(25,25)$ & $0.01$\\
\midrule
ER interactions & 1 & $(200,200)$ & $(25,25,25)$ & $0.01$\\
& 2 & $(150,150)$ & $(25,25,25)$ & $0.01$\\
& 3 & $(150,150)$ & $(25,25,25)$ & $0.005$\\
& 4 & $(25,25,25)$ & $(25,25,25)$ & $0.0005$\\
& 5 & $(200,200)$ & $(25,25)$ & $0.01$\\
\bottomrule
\end{tabular}
\par\smallskip
\end{table}

\subsubsection{Covariate discretization for survival inference}
\label{app:rotterdam-discretization}

For the survival curves and pointwise confidence intervals in Figure~1 of the main text,
we applied the finite-support inference procedure by grouping four
continuous or count-valued nuisance covariates: age, number of positive
lymph nodes, estrogen receptor (ER) level, and progesterone receptor
(PGR) level. Table~\ref{tab:rotterdam-discretization} gives the grouping
rules. Menopausal status, tumor size, and tumor grade retained their
original categories. The resulting nuisance vector was
\[
\bX=(\mathrm{age},\mathrm{meno},\mathrm{size},\mathrm{grade},
   \mathrm{nodes},\mathrm{PGR},\mathrm{ER})^\top.
\]

\begin{table}[htbp]
\centering
\caption{Covariate grouping for the Rotterdam survival-inference
analysis in Figure~1.}
\label{tab:rotterdam-discretization}
\begin{tabular}{lllll}
\toprule
Variable & Group 1 & Group 2 & Group 3 & Group 4\\
\midrule
Age (years)          & $\leq45$ & $46$--$54$ & $55$--$65$ & $>65$\\
Positive lymph nodes & $0$     & $1$--$3$   & $\geq4$    & ---\\
ER                   & $\leq3$ & $4$--$9$   & $10$--$49$ & $\geq50$\\
PGR                  & $\leq9$ & $10$--$99$ & $\geq100$  & ---\\
\bottomrule
\end{tabular}
\end{table}

After grouping, the 2982 patients represented 590 distinct observed
joint covariate profiles. Thus, the finite-support formulation
accommodated hundreds of different combinations of the seven nuisance
covariates. The estimated curves and pointwise intervals concern
survival conditional on the grouped vector $\bX$ and the primary
covariates $\bZ$, under the model and censoring assumptions for these
analysis covariates.  The precision of inference also depends on the
numbers of subjects and events available within each profile.

Figure~1 of the main text displays
the 16 most frequent profiles, containing 505 patients (16.9\% of the
cohort) and 215 observed events. Their sizes ranged from 23 to 46 patients
(median 31.5), with 4--26 events per profile (median 12), compared with a
median of 2.5 patients and one event across all profiles.
Table~\ref{tab:rotterdam-displayed-profiles} gives the individual
profile counts.

\begin{table}[htbp]
\centering\small
\setlength{\tabcolsep}{3.5pt}
\caption{Covariate values and counts for the 16 most frequent Rotterdam
adjustment profiles.}
\label{tab:rotterdam-displayed-profiles}
\begin{tabular}{rcccccccrrr}
\toprule
Rank & Age & Meno & Size & Grade & Nodes & PGR & ER & $n$ & Events & Censored\\
\midrule
1 & $>65$ & Post & $\leq20$ & 3 & 0 & $\geq100$ & $\geq50$ & 46 & 26 & 20\\
2 & $\leq45$ & Pre & 20--50 & 3 & 0 & $\leq9$ & $\leq3$ & 38 & 20 & 18\\
3 & 55--65 & Post & $\leq20$ & 3 & 0 & $\geq100$ & $\geq50$ & 37 & 12 & 25\\
4 & 55--65 & Post & $\leq20$ & 3 & 0 & 10--99 & $\geq50$ & 36 & 9 & 27\\
5 & $>65$ & Post & 20--50 & 3 & 0 & $\geq100$ & $\geq50$ & 36 & 20 & 16\\
6 & $\leq45$ & Pre & $\leq20$ & 3 & 0 & $\geq100$ & $\geq50$ & 35 & 16 & 19\\
7 & 55--65 & Post & $\leq20$ & 2 & 0 & $\geq100$ & $\geq50$ & 34 & 9 & 25\\
8 & $\leq45$ & Pre & $\leq20$ & 3 & 0 & $\leq9$ & $\leq3$ & 33 & 11 & 22\\
9 & 46--54 & Pre & $\leq20$ & 2 & 0 & $\geq100$ & $\geq50$ & 30 & 5 & 25\\
10 & $>65$ & Post & $\leq20$ & 3 & 0 & 10--99 & $\geq50$ & 28 & 15 & 13\\
11 & 46--54 & Pre & $\leq20$ & 3 & 0 & $\geq100$ & $\geq50$ & 27 & 11 & 16\\
12 & $\leq45$ & Pre & $\leq20$ & 2 & 0 & $\geq100$ & $\geq50$ & 26 & 4 & 22\\
13 & $>65$ & Post & $\leq20$ & 2 & 0 & $\geq100$ & $\geq50$ & 26 & 10 & 16\\
14 & 55--65 & Post & 20--50 & 3 & 1--3 & $\geq100$ & $\geq50$ & 25 & 15 & 10\\
15 & $>65$ & Post & 20--50 & 3 & $\geq4$ & $\geq100$ & $\geq50$ & 25 & 20 & 5\\
16 & 55--65 & Post & $\leq20$ & 3 & 1--3 & $\geq100$ & $\geq50$ & 23 & 12 & 11\\
\midrule
\multicolumn{8}{l}{Total} & 505 & 215 & 290\\
\bottomrule
\end{tabular}
\par\smallskip
\end{table}

\newpage
\section{Proofs for Section 4}\label{appA}

We begin by establishing several basic properties of the model. A key step is to understand how the counting–process framework allows us to translate expressions involving the at-risk processes into simpler quantities expressed directly in terms of the event indicator $\Delta$. The next lemma formalizes this relationship, showing that these two representations are equivalent up to a  multiplicative constant. This equivalence will be used repeatedly throughout the analysis.


\begin{restatable}{lemma}{IntVsEvent}\label{lem:int-vs-event}
Under \ref{ass:A1}–\ref{ass:A2}, for any nonnegative, predictable function 
$f(t,\bX,\bZ)$ with $\mathbb{E}\!\left[\int_0^\tau Y(t)\,f(t,\bX,\bZ)\,dt\right]<\infty$,
$$
\mathbb{E}\!\left\{\int_0^\tau Y(t)\,f(t,\bX,\bZ)\,dt\right\}
\;\asymp\;
\mathbb{E}\!\left\{ \Delta\,f(T,\bX,\bZ)\right\}.
$$
\end{restatable}


\begin{proof}[Proof of \Cref{lem:int-vs-event}]
Let $N(t)=\Delta\,\mathbf{1}(T\le t)$ be the counting process associated with the event time, and let
$
M(t)
=
N(t) - \int_0^t Y(s)\, h_o(s\,|\,{\bX},{\bZ})\, ds
$
denote its martingale decomposition under the true hazard (\ref{eq:model1}). Because $f(t,\bX,\bZ)$ is predictable, the stochastic integral $\int_0^\tau f(t,{\bX},{\bZ})\, dM(t)$ has expectation zero. Consequently,
$$
\mathbb{E}\!\left\{\int_0^\tau f(t,\bX,\bZ)\, dN(t)\right\}
=
\mathbb{E}\!\left\{ \int_0^\tau f(t,\bX,\bZ)\, Y(t)\, h_o(t\,|\,\bX,\bZ)\, dt\right\}.
$$
Since
$\int_0^\tau f(t,\bX,\bZ)\, dN(t)=
\Delta\, f(T,\bX,\bZ)$, we obtain the identity
\begin{equation}
\label{eq:lemma1-id}
\mathbb{E}\{\Delta\, f(T,\bX,\bZ)\}
=
\mathbb{E}\!\left\{ \int_0^\tau f(t,\bX,\bZ)\, Y(t)\, h_o(t\,|\,\bX,\bZ)\, dt\right\} \, .
\end{equation}
Under Assumptions A1--A2, the hazard is uniformly bounded and bounded away from zero,
$$
0 < c_h \;\le\; h_o(t\,|\,\bX,\bZ) \;\le\; C_h < \infty.
$$
Applying these bounds to the right-hand side of (\ref{eq:lemma1-id}) yields
$$
c_h\, \mathbb{E}\!\left\{ \int_0^\tau f(t,\bX,\bZ)\, Y(t)\, dt\right\}
\;\le\;
\mathbb{E}\{ \Delta\, f(T,\bX,\bZ)\}
\;\le\;
C_h\, \mathbb{E}\!\left\{ \int_0^\tau f(t,\bX,\bZ)\, Y(t)\, dt\right\} \, .
$$
This implies the desired equivalence:
\[
\mathbb{E}\!\left\{ \int_0^\tau f(t,\bX,\bZ)\, Y(t)\, dt\right\}
\asymp
\mathbb{E}\!\{ \Delta\, f(T,\bX,\bZ)\} \, .
\]
\end{proof}


\begin{restatable}{lemma}{DeltaZeroImpliesUnweightedZero}%
\label{lem:delta-zero-implies-unweighted-zero}
Let $\chi_{\eta_j}(t,\bX,\bZ)$, $j=1,2$, be bounded predictable processes satisfying
$$
\EE\!\left\{\int_0^\tau \chi_{\eta_j}(t,\bX,\bZ)^2\,dt \right\}<\infty \qquad j=1,2 \, .
$$ 
Define
$\phi(t,\bX,\bZ)= \chi_{\eta_1}(t,\bX,\bZ) - \chi_{\eta_2}(t,\bX,\bZ)$. Under Assumptions~\ref{ass:A1}, \ref{ass:A2} and~\ref{ass:A4}, there exist finite constants $0<c_1\le c_2<\infty$ and
$0<c_3\le c_4<\infty$, depending only on
$(\tau,\delta_1,\delta_2,c_h,C_h)$, such that
\begin{align}
c_1\,\EE\!\big\{\Delta\,\phi(T,\bX,\bZ)^2\big\}
&\;\le\;
\EE\!\big\{\phi(T,\bX,\bZ)^2\big\}
\;\le\;
c_2\,\EE\!\big\{\Delta\,\phi(T,\bX,\bZ)^2\big\} \, ,
\label{eq:phiT-Delta-equivalence}
\\[0.5em]
c_3\,\EE\!\big\{ \Delta\,\phi(T,\bX,\bZ)^2\big\}
&\;\le\;
\EE_{\bX,\bZ}\!\Big\{ \int_0^\tau \phi(t,\bX,\bZ)^2\,dt\Big\}
\;\le\;
c_4\,\EE\!\big\{ \Delta\,\phi(T,\bX,\bZ)^2\big\} \, . 
\label{eq:time-integral-Delta-equivalence}
\end{align}
In particular, if $\EE\{\Delta\,\phi(T,\bX,\bZ)^2\}=0$, then
$\phi(T,\bX,\bZ)=0$ $\PP_0$–a.s. and
$\phi(t,\bX,\bZ)=0$ for Lebesgue-a.e. $t\in[0,\tau]$ and
$\PP_{\bX,\bZ}$–a.e. $(\bX,\bZ)$, i.e.
$\chi_{\eta_1}(t,\bX,\bZ)=\chi_{\eta_2}(t,\bX,\bZ)$
for $dt\otimes \PP_{\bX,\bZ}$–almost all $(t,\bX,\bZ)$.
\end{restatable}

Lemma~\ref{lem:delta-zero-implies-unweighted-zero} shows that the distance $d(\cdot,\cdot)$ is a valid  metric on the parameter space: $d(\eta_1,\eta_2)=0$ if and only if the induced linear score coincides almost everywhere. In addition, the lemma shows that the event-weighted norm and the corresponding unweighted norms are equivalent up to multiplicative constants. As a result, any bound or convergence statement expressed in terms of the metric $d$ can be translated into an equivalent statement under the unweighted $L^2(dt\otimes \PP_{\bX,\bZ})$ norm and vice versa. This equivalence is used repeatedly throughout the theoretical development.

\begin{proof}[Proof of \Cref{lem:delta-zero-implies-unweighted-zero}]
By the tower property of conditional expectation,
$$
\EE\!\big\{\Delta\,\phi(T,\bX,\bZ)^2\big\}
=
\EE\Big\{
  \phi(T,\bX,\bZ)^2\,\Pr(\Delta=1\mid T,\bX,\bZ)
\Big\} \, .
$$
Assumption~\ref{ass:A4} states that 
$$
0<\delta_1
\;\le\;
\Pr(\Delta=1\mid T,\bX,\bZ)
\;\le\;
1
\qquad\text{a.s}.
$$
Since $\phi(T,\bX,\bZ)^2\ge 0$, we obtain
\[
\delta_1\,\EE\!\big\{\phi(T,\bX,\bZ)^2\big\}
\;\le\;
\EE\!\big\{\Delta\,\phi(T,\bX,\bZ)^2\big\}
\;\le\;
\,\EE\!\big\{\phi(T,\bX,\bZ)^2\big\} \, .
\]
Rearranging the inequalities gives
$$
\EE\!\big\{\Delta\,\phi(T,\bX,\bZ)^2\big\}
\;\le\;
\EE\!\big\{\phi(T,\bX,\bZ)^2\big\}
\;\le\;
\frac{1}{\delta_1}\,\EE\!\big\{\Delta\,\phi(T,\bX,\bZ)^2\big\} \, .
$$
Thus \eqref{eq:phiT-Delta-equivalence} holds with
$c_1 \leq 1$ and $c_2 = 1/\delta_1$.

Next, we  relate the event–weighted quantity
$\EE\{\Delta\,\phi(T,\bX,\bZ)^2\}$ to the time–integrated $L^2$ norm
of $\phi$. Write,
\begin{equation}\label{eq:condexp}
\EE\!\big\{\Delta\,\phi(T,\bX,\bZ)^2\big\}
=
\EE_{\bX,\bZ}\!\Big[
  \EE\!\big\{\Delta\,\phi(T,\bX,\bZ)^2 \,\big|\, \bX,\bZ\big\}
\Big] \, ,    
\end{equation}
and let $f_{T,\Delta=1\mid\bX,\bZ}(t\mid\bX,\bZ)$ denote the conditional
density of $(T,\Delta)$ at $(t,1)$ given $(\bX,\bZ)$.
Then,
$$
\EE\!\big\{\Delta\,\phi(T,\bX,\bZ)^2 \,\big|\, \bX,\bZ\big\}
=
\int_0^\tau \phi(t,\bX,\bZ)^2 \,
       f_{T,\Delta=1\mid\bX,\bZ}(t\mid\bX,\bZ)\,dt \, .
$$
Under Assumptions~\ref{ass:A1} and~\ref{ass:A2}, the event-time
hazard $h_o(t,\bX,\bZ)$ is bounded away from zero and infinity.
Similarly, Assumption~\ref{ass:A4} implies the censoring mechanism is
well behaved. Therefore,
$$
\Pr(T>t \mid \bX,\bZ) = \exp\left[ -\int_0^t \{h_o(s \mid \bX, \bZ) + \mu(s \mid \bX, \bZ) ds\} \right]
$$
where $\mu(t \mid \bX , \bZ)$ is the censoring hazard bounded away from zero and infinity. Then,
$$f_{T,\Delta=1\mid\bX,\bZ}(t\mid\bX,\bZ) = h_o(t \mid \bX, \bZ) \exp\left[ -\int_0^t \{h_o(s \mid \bX, \bZ) + \mu(s \mid \bX, \bZ) ds\} \right]
$$ is bounded above and below
by positive constants that depend only on
$(\tau,\delta_1,\delta_2,c_h,C_h)$. Specifically,
there exist $0<c_f\le C_f<\infty$ such that, almost surely,
$$
c_f
\;\le\;
f_{T,\Delta=1\mid\bX,\bZ}(t\mid\bX,\bZ)
\;\le\;
C_f,
\qquad t\in[0,\tau] \, .
$$
Using these bounds inside the conditional expectation (\ref{eq:condexp}), we obtain,
conditional on $(\bX,\bZ)$,
$$
c_f \int_0^\tau \phi(t,\bX,\bZ)^2\,dt
\;\le\;
\EE\!\big\{\Delta\,\phi(T,\bX,\bZ)^2 \,\big|\, \bX,\bZ\big\}
\;\le\;
C_f \int_0^\tau \phi(t,\bX,\bZ)^2\,dt \, .
$$
Taking expectation over $(\bX,\bZ)$ yields
\[
c_f\,\EE_{\bX,\bZ}\!\Big\{\int_0^\tau \phi(t,\bX,\bZ)^2\,dt\Big\}
\;\le\;
\EE\!\big\{\Delta\,\phi(T,\bX,\bZ)^2\big\}
\;\le\;
C_f\,\EE_{\bX,\bZ}\!\Big\{\int_0^\tau \phi(t,\bX,\bZ)^2\,dt\Big\}.
\]
Thus \eqref{eq:time-integral-Delta-equivalence} holds with
$c_3 = 1/C_f$ and $c_4 = 1/c_f$ (or equivalently with $c_3,c_4$
chosen directly as above).

Finally, suppose $\EE\{\Delta\,\phi(T,\bX,\bZ)^2\}=0$.  Since
$\Delta\,\phi(T,\bX,\bZ)^2\ge 0$, this implies
$\Delta\,\phi(T,\bX,\bZ)^2=0$ almost surely, hence
$\phi(T,\bX,\bZ)=0$ $\PP_0$–a.s. on $\{\Delta=1\}$.
From \eqref{eq:phiT-Delta-equivalence} we also have
$$
0
\;\le\;
\EE\!\big\{\phi(T,\bX,\bZ)^2\big\}
\;\le\;
c_2\,\EE\!\big\{\Delta\,\phi(T,\bX,\bZ)^2\big\}
= 0 \, ,
$$
so $\EE\{\phi(T,\bX,\bZ)^2\}=0$, which forces
$\phi(T,\bX,\bZ)=0$ $\PP_0$–almost surely. Likewise, from \eqref{eq:time-integral-Delta-equivalence},
$$
0
\;\le\;
\EE_{\bX,\bZ}\!\Big\{\int_0^\tau \phi(t,\bX,\bZ)^2\,dt\Big\}
\;\le\;
c_4\,\EE\!\big\{\Delta\,\phi(T,\bX,\bZ)^2\big\}
=0 \, ,
$$
so
$$
\EE_{\bX,\bZ}\!\Big\{\int_0^\tau \phi(t,\bX,\bZ)^2\,dt\Big\}=0 \, .
$$
Since the integrand $\phi(t,\bX,\bZ)^2$ is nonnegative, this implies
\[
\int_0^\tau \phi(t,\bX,\bZ)^2\,dt = 0
\quad\text{for $\PP_{\bX,\bZ}$–a.e. }(\bX,\bZ),
\]
and hence $\phi(t,\bX,\bZ)=0$ for Lebesgue–a.e. $t\in[0,\tau]$ and
$\PP_{\bX,\bZ}$–a.e. $(\bX,\bZ)$.  Equivalently,
$$
\chi_{\eta_1}(t,\bX,\bZ)
=
\chi_{\eta_2}(t,\bX,\bZ)
\quad\text{for }dt\otimes \PP_{\bX,\bZ}\text{–almost all }(t,\bX,\bZ) \, .
$$
\end{proof}


Next we study the shape of the population loss around $\eta_o=(\btheta_o,g_o)$.

\begin{restatable}{lemma}{ConcutivityOfPopulationLoss}\label{lmm:concavity}
Assume \ref{ass:A1}--\ref{ass:A2}.  For any $\eta\in \Theta \times \mathcal{H}(q,\boldsymbol{\alpha},{\bf d}^*,\tilde{{\bf d}},M)$,
\[
\ell_0(\eta)-\ell_0(\eta_o)\;\asymp\;-\,d(\eta,\eta_o)^2.
\]
\end{restatable}

Lemma~\ref{lmm:concavity} shows that the population loss is well behaved in a neighbourhood of the truth: any deviation from $\eta_o$ incurs a loss of order $d(\eta,\eta_o)^2$. Combined with Lemma~\ref{lem:delta-zero-implies-unweighted-zero}, this implies that $\eta_o$ is the unique maximizer of the population loss over the square–integrable class under consideration. This population curvature is a key ingredient in the analysis that will later be paired with empirical process bounds to obtain convergence rates for $\widehat{g}$ and to derive the asymptotic distribution of $\widehat{\btheta}$.


\begin{proof}[Proof of \Cref{lmm:concavity}]
Recall that $\ell_0(\eta)$ denotes the population (expected) log-likelihood,
\[
  \ell_0(\eta)
  = \EE\Big\{ \,\Delta\,\chi_\eta(T,\bX,\bZ)
   - \int_0^\tau Y(t)\,h_\eta(t,\bX,\bZ)\,dt\Big\},
\]
where $\chi_\eta$ is the linear predictor and
$h_\eta(t,\bX,\bZ)$ is the corresponding instantaneous hazard process.  
Let $\eta_o=(\theta_o,g_o)$ denote the true parameter, with true hazard
$h_o(t,\bX,\bZ) = h_{\eta_o}(t,\bX,\bZ)$.

We study the behavior of $\ell_0$ along a one–dimensional path through
$\eta_o$.  Fix an arbitrary direction $w = (w_\theta,w_g)$ and define
$$
  \eta_\varepsilon = \eta_o + \varepsilon w \, ,
  \qquad \varepsilon \in \mathbb{R} \, .
$$
Assume that $\varepsilon$ is restricted to a small neighbourhood of $0$
so that $\eta_\varepsilon \in  \Theta \times \mathcal{H}(q,\boldsymbol{\alpha},{\bf d}^*,\tilde{{\bf d}},M)$. Because $\chi_\eta$ is linear in $(\btheta,g)$, the Gateaux derivative of
$\chi_\eta$ in direction $w$ at $\eta_o$ is simply
$$
  \dot\chi_w(t,\bX,\bZ)
  = \frac{d}{d\varepsilon}\,\chi_{\eta_\varepsilon}(t,\bX,\bZ)\Big|_{\varepsilon=0} \, ,
$$
and we can write
\[
  \chi_{\eta_\varepsilon}(t,\bX,\bZ)
  = \chi_{\eta_o}(t,\bX,\bZ) + \varepsilon\,\dot\chi_w(t,\bX,\bZ) \, .
\]
For notational simplicity, when we fix a particular $\eta$, we also write
$$
  q(t,\bX,\bZ) = \chi_\eta(t,\bX,\bZ) - \chi_{\eta_o}(t,\bX,\bZ) \, ,
$$
so $q$ is the difference in linear predictors between $\eta$ and $\eta_o$.
Along the path $\eta_\varepsilon$, $q$ is proportional to the direction
$\dot\chi_w$.

Next, note that 
$$
  h_{\eta_\varepsilon}(t \mid \bX,\bZ)
  = h_o(t \mid \bX,\bZ)\,\exp\{\chi_{\eta_\varepsilon}(t,\bX,\bZ)
                            - \chi_{\eta_o}(t,\bX,\bZ)\}
  = h_o(t \mid \bX,\bZ)\,\exp\{\varepsilon\,\dot\chi_w(t,\bX,\bZ)\} \, ,
$$
so at $\varepsilon=0$ we have $h_{\eta_o}=h_o$ and
$$
  \frac{d}{d\varepsilon}h_{\eta_\varepsilon}(t,\bX,\bZ)\Big|_{\varepsilon=0}
  = h_o(t \mid \bX,\bZ)\,\dot\chi_w(t,\bX,\bZ) \, ,
$$
$$
  \frac{d^2}{d\varepsilon^2}h_{\eta_\varepsilon}(t,\bX,\bZ)\Big|_{\varepsilon=0}
  = h_o(t \mid \bX,\bZ)\,\dot\chi_w(t,\bX,\bZ)^2 \, .
$$
By the definition of the Gateaux derivative,
$$
  \ell_0'(\eta_o)[w]
  = \frac{d}{d\varepsilon}\,\ell_0(\eta_\varepsilon)\Big|_{\varepsilon=0} \, .
$$
Using the explicit form of $\ell_0$ and differentiating under the expectation,
\begin{align*}
  \ell_0'(\eta_o)[w]
  &= \frac{d}{d\varepsilon}\EE\Big \{
        \Delta\,\chi_{\eta_\varepsilon}(T,\bX,\bZ)
        - \int_0^\tau Y(t)\,h_{\eta_\varepsilon}(t \mid \bX,\bZ)\,dt
     \Big\} \Big|_{\varepsilon=0} \\
  &= \EE\Big\{
        \Delta\,\dot\chi_w(T,\bX,\bZ)
        - \int_0^\tau Y(t)\,
           \frac{d}{d\varepsilon}h_{\eta_\varepsilon}(t \mid  \bX,\bZ)\Big|_{\varepsilon=0}\,dt
     \Big\} \\
  &= \EE\Big\{
        \Delta\,\dot\chi_w(T,\bX,\bZ)
        - \int_0^\tau Y(t)\,h_o(t \mid  \bX,\bZ)\,\dot\chi_w(t,\bX,\bZ)\,dt
     \Big\} \, .
\end{align*}
Since
$$
  \Delta\,\dot\chi_w(T,\bX,\bZ)
  = \int_0^\tau \dot\chi_w(t,\bX,\bZ)\,dN(t) \, ,
$$
the Doob–Meyer decomposition for the true model (\ref{eq:model1}) gives
$$
  dN(t) = Y(t)h_o(t \mid  \bX,\bZ)\,dt + dM(t) \, ,
$$
where $M(t)$ is a martingale. Therefore,
\begin{align*}
  \Delta\,\dot\chi_w(T,\bX,\bZ)
  &= \int_0^\tau \dot\chi_w(t,\bX,\bZ)\,Y(t)h_o(t \mid \bX,\bZ)\,dt
     + \int_0^\tau \dot\chi_w(t,\bX,\bZ)\,dM(t) \, .
\end{align*}
Substituting this into the expression for $\ell_0'(\eta_o)[w]$ and cancelling
the compensator term,
\begin{align*}
  \ell_0'(\eta_o)[w]
  &= \EE\Big\{
      \int_0^\tau \dot\chi_w(t,\bX,\bZ)\,dM(t)
     \Big\} \, .
\end{align*}
The stochastic integral $\int_0^\tau \dot\chi_w(t,\bX,\bZ)\,dM(t)$ is a martingale
with mean zero. Hence,
\[
  \ell_0'(\eta_o)[w] = 0
  \qquad\text{for every direction } w \, .
\]
In particular, when we fix $\eta$ and set $q(t,\bX,\bZ)
= \chi_\eta(t,\bX,\bZ)-\chi_{\eta_o}(t,\bX,\bZ)$, we can summarize this step as
\[
  \ell_0'(\eta_o)[w]
  = \EE\Big\{
       \Delta\,q(T,\bX,\bZ)
       - \int_0^\tau Y(t)h_o(t \mid \bX,\bZ)\,q(t,\bX,\bZ)\,dt
    \Big \}
  = 0.
\]

To control $\ell_0(\eta)-\ell_0(\eta_o)$, we apply a one–dimensional Taylor
expansion of $\ell_0(\eta_\varepsilon)$ around
$\varepsilon=0$.  Write $\eta=\eta_1$ and define the line segment
$\eta_v = \eta_o + v(\eta-\eta_o)$ for $v\in[0,1]$.  The mean value theorem
yields
$$
  \ell_0(\eta)
  = \ell_0(\eta_o) + \ell_0'(\eta_o)[w]
    + \frac{1}{2}\,\ell_0''(\eta_{\bar v})[w,w]
$$
for some $\bar v\in(0,1)$, where $w = \eta-\eta_o$.  Since we have already
shown that $\ell_0'(\eta_o)[w]=0$, it remains to analyse the second derivative
term.

For a fixed $v\in[0,1]$, consider the path
$\varepsilon\mapsto\eta_v + \varepsilon w$ and write
\[
  q_v(t,\bX,\bZ)
  = \frac{d}{d\varepsilon}\,\chi_{\eta_v+\varepsilon w}(t,\bX,\bZ)\Big|_{\varepsilon=0} \, ,
\]
which is again linear in $w$.  Repeating the differentiation steps above but
now starting from $\eta_v$ instead of $\eta_o$ gives
\begin{align*}
  \ell_0''(\eta_v)[w,w]
  &= \frac{d^2}{d\varepsilon^2}\,\ell_0(\eta_v+\varepsilon w)\Big|_{\varepsilon=0} \\
  &= \EE\Big\{
        \Delta\,\frac{d^2}{d\varepsilon^2}
        \chi_{\eta_v+\varepsilon w}(T,\bX,\bZ)\Big|_{\varepsilon=0}
        - \int_0^\tau Y(t)\,
          \frac{d^2}{d\varepsilon^2}
           h_{\eta_v+\varepsilon w}(t,\bX,\bZ)\Big|_{\varepsilon=0}\,dt
     \Big\} \, .
\end{align*}
Since $\chi_\eta$ is linear in $\eta$, its second derivative in $\varepsilon$
vanishes, while the second derivative of the hazard is
\[
  \frac{d^2}{d\varepsilon^2}
  h_{\eta_v+\varepsilon w}(t \mid \bX,\bZ)\Big|_{\varepsilon=0}
  = h_{\eta_v}(t \mid \bX,\bZ)\,q_v(t,\bX,\bZ)^2 \, .
\]
Hence
\[
  \ell_0''(\eta_v)[w,w]
  = - \EE\Big\{
        \int_0^\tau Y(t)\,q_v(t,\bX,\bZ)^2
        \,h_{\eta_v}(t \mid \bX,\bZ)\,dt
     \Big\} \, .
\]
For notational simplicity (since we will ultimately evaluate this at
$\varepsilon=1$), we drop the subscript $v$ and simply write
$q(t,\bX,\bZ)=\chi_\eta(t,\bX,\bZ)-\chi_{\eta_o}(t,\bX,\bZ)$ along the line,
so that
\[
  \ell_0''(\eta_v)[w,w]
  = - \EE\Big\{
        \int_0^\tau Y(t)\,q(t,\bX,\bZ)^2
        \,h_{\eta_v}(t \mid \bX,\bZ)\,dt
     \Big\} \, .
\]

Next, for relating the second derivative to $d(\eta,\eta_o)^2$, we
apply Lemma~4 for 
$$
  f_v(t,\bX,\bZ)
  = q(t,\bX,\bZ)^2\,
     \frac{h_{\eta_v}(t \mid \bX,\bZ)}{h_o(t \mid  \bX,\bZ)} \, .
$$
Indeed, Lemma~4 gives
$$
  \EE\Big\{
       \int_0^\tau Y(t)\,q(t,\bX,\bZ)^2\,h_{\eta_v}(t \mid \bX,\bZ)\,dt
    \Big\}
  \asymp
  \EE\Big\{
       \Delta\,q(T,\bX,\bZ)^2\,
       \frac{h_{\eta_v}(T \mid \bX,\bZ)}{h_o(T \mid \bX,\bZ)}
    \Big\} \, .
$$
Under Assumptions~\ref{ass:A1}--\ref{ass:A2}, the processes $\chi_\eta$ are uniformly bounded, and hence
$$
  0 < c_\star
  \;\le\;
  \frac{h_{\eta_v}(t \mid \bX,\bZ)}{h_o(t \mid \bX,\bZ)}
  \;\le\;
  C_\star < \infty,
  \qquad\text{for all }(t,\bX,\bZ),\; v\in[0,1] \, .
$$
Therefore,
$$
  c_\star \,\EE\big[\Delta\,q(T,\bX,\bZ)^2\big]
  \;\le\;
  \EE\Big\{
       \Delta\,q(T,\bX,\bZ)^2\,
       \frac{h_{\eta_v}(T \mid \bX,\bZ)}{h_o(T \mid \bX,\bZ)}
    \Big]
  \;\le\;
  C_\star \,\EE\big[\Delta\,q(T,\bX,\bZ)^2\big \} \, .
$$
Combining this with the expression for $\ell_0''(\eta_v)[w,w]$ yields the two-sided bound
$$
  - C_\star\,\EE\big\{ \Delta\,q(T,\bX,\bZ)^2\big\}
  \;\le\;
  \ell_0''(\eta_v)[w,w]
  \;\le\;
  - c_\star\,\EE\big \{ \Delta\,q(T,\bX,\bZ)^2\big \}.
$$
Recall that the distance $d(\eta,\eta_o)$ is defined by
\[
  d(\eta,\eta_o)^2
  = \EE\Big\{ \Delta\big\{\chi_\eta(T,\bX,\bZ)-\chi_{\eta_o}(T,\bX,\bZ)\big\}^2\Big\} 
  = \EE\big\{ \Delta\,q(T,\bX,\bZ)^2\big \} \, .
\]
Hence
\[
  - C_\star\,d(\eta,\eta_o)^2
  \;\le\;
  \ell_0''(\eta_v)[w,w]
  \;\le\;
  - c_\star\,d(\eta,\eta_o)^2 \, .
\]
Therefore, by the Taylor expansion with the mean value theorem,
\[
  \ell_0(\eta)-\ell_0(\eta_o)
  = \frac{1}{2}\,\ell_0''(\eta_{\bar v})[w,w]
\]
for some $\bar v\in(0,1)$, and we have just shown that
$\ell_0''(\eta_{\bar v})[w,w]$ is comparable to $-d(\eta,\eta_o)^2$.  Thus,
\[
  \ell_0(\eta)-\ell_0(\eta_o)
  \;\asymp\; -\,d(\eta,\eta_o)^2,
\]
which completes the proof.
\end{proof}

In the following, $\PP_n$ and $\PP$ denote the empirical and probability measures of $(T,\Delta,\bX,\bZ)$, and
$\GG_n = \sqrt{n}(\PP_n - \PP)$ is the empirical process.
The symbol $\mathbb{E}^{\ast}$ stands for outer expectation, used to avoid
measurability issues in the supremum over an uncountable class of functions.

\begin{lemma}
\label{lem:unified_maximal}
Consider the following classes of functions indexed by the parameter $\eta$:
\begin{itemize}
\item The log-likelihood increment
\[
f_\eta(T,\mathbf{X},\mathbf{Z},\Delta)
=
\Delta\,\chi_\eta(T,\mathbf{X},\mathbf{Z})
-
\int_0^{T}
Y(t)\,\exp\{\chi_\eta(t,\mathbf{X},\mathbf{Z})\}\,dt ,
\]
\item The efficient score-type term
\[
\psi_\eta(T,\mathbf{X},\mathbf{Z},\Delta)
=
\int_0^{T}
\bigl\{ \mathbf{Z}-\bg^{\ast}(t,\mathbf{X}) \bigr\}
\, dM_\eta(t),
\]
where $\bg^*$ denotes the efficient projection defined in
Theorem~\ref{thm:eff-score} and
\[
M_\eta(u)
=
N(u)-\int_0^u
Y(v)\exp\{\chi_\eta(v,\bX,\bZ)\}\,dv,
\qquad 0\le u\le\tau.
\]
We use the equivalent notation
$r_\eta=\psi_\eta=r_{\theta,g}$, where $r_{\theta,g}$
is the projected-score function defined in Assumption~\ref{ass:Anew}.
The process $M_\eta$ is a residual process; at the truth,
$M_{\eta_o}=M$ is a martingale under $\PP_0$, but this
need not hold for general $\eta$.
\end{itemize}

For $0<\delta\le 1$, define the localized classes
\[
\mathcal{F}_\delta
=
\{\, f_\eta - f_{\eta_o} : d(\eta,\eta_o)\le \delta \,\} \, ,
\qquad
\mathcal{G}_\delta
=
\{\, \psi_\eta - \psi_{\eta_o} : d(\eta,\eta_o)\le \delta \,\} \, .
\]

Under Assumptions~\ref{ass:A1}--\ref{ass:A3} and \ref{ass:A5}--\ref{ass:Anew}, there exists a constant $C>0$ such that for every $\delta\in(0,1]$,
\[
\mathbb{E}^{\ast}
\Bigl\{
\sup_{\mu\in \mathcal{F}_\delta \cup \mathcal{G}_\delta}
\bigl|
\sqrt{n}\,(\PP_n - \PP)\,\mu
\bigr|
\Bigr\}
\;\le\;
C\Bigl\{
\delta\,\sqrt{s\log(\mathcal{U}/\delta)}
+
n^{-1/2}\, s\log(\mathcal{U}/\delta)
\Bigr\},
\]
where $s$ is the sparsity of the neural network class, and
$$
\mathcal{U} = K \prod_{k=0}^{K} (p_k+1) \sum_{k=0}^{K} p_k p_{k+1}.
$$
\end{lemma}

\begin{proof}[Proof of \Cref{lem:unified_maximal}]

Fix $\eta_1,\eta_2$ and denote
$q(t,\bX,\bZ)=\chi_{\eta_1}(t,\bX,\bZ)-\chi_{\eta_2}(t,\bX,\bZ)$. For notational convenience, we write
$\chi_\eta(t) = \chi_\eta(t,\bX,\bZ)$
and
$q(t)=\chi_{\eta_1}(t)-\chi_{\eta_2}(t)$.  
Because $\chi_\eta$ is uniformly bounded under \ref{ass:A1}--\ref{ass:A2} and depends on the network parameters in a Lipschitz way, we have
$$
\bigl|\exp\{\chi_{\eta_1}(t,\bX,\bZ)\}
      -\exp\{\chi_{\eta_2}(t,\bX,\bZ)\}\bigr|
  \;\lesssim\; |q(t,\bX,\bZ)| \, .
$$
Using the definition of $f_\eta$ and the triangle and Cauchy--Schwarz inequalities,
$$
|f_{\eta_1} - f_{\eta_2}|
\le
\Delta\,|q(T)|
+
\int_0^\tau
Y(t)\,
|\exp\{\chi_{\eta_1}(t)\}-\exp\{\chi_{\eta_2}(t)\}|
\, dt \, .
$$
Plugging the bound, squaring and applying the inequality $(a+b)^2\le2a^2+2b^2$ gives
$$
(f_{\eta_1}-f_{\eta_2})^2
  \;\lesssim\;
    2 \Delta\,q(T)^2
    + 2 \tau \int_0^\tau Y(t)\,q(t)^2\,dt,
$$
where all quantities are uniform in $(\bX,\bZ)$ by \ref{ass:A1}--\ref{ass:A2}.  
By Lemma~4 we get
$$
E\!\left[\int_0^\tau Y(t)\, q(t)^2\,dt\right]
 \asymp
E\bigl[\Delta\, q(T)^2\bigr] \, .
$$
The right-hand side is exactly the squared metric $d^2(\eta_1,\eta_2)$.  
Thus,
$$
\|f_{\eta_1}-f_{\eta_2}\|_{L^2(\PP)}
   \;\lesssim\; d(\eta_1,\eta_2) \, .
$$
Repeating the argument for $\psi_\eta$ requires only replacing the exponential term with the stochastic integral.  
By Assumption~\ref{ass:A2}, $\bZ$ is uniformly bounded, and
by Assumption~\ref{ass:Anew}, $\bg^*$ is uniformly bounded.
Consequently,
\[
\sup_{t,\bX,\bZ}
\|\bZ-\bg^*(t,\bX)\|<\infty.
\]
Together with the uniform boundedness of the
linear predictors under \ref{ass:A2}--\ref{ass:A3}, this provides a
bounded envelope for the projected-score class.
Moreover, for
$\chi_\eta(t,\bX,\bZ)=\theta^\top\bZ+g(t,\bX)$,
the mean-value theorem gives
\[
|e^{\chi_{\eta_1}}-e^{\chi_{\eta_2}}|
\le C|\chi_{\eta_1}-\chi_{\eta_2}|.
\]
Hence, by Cauchy--Schwarz and the compensator
identity,
\[
\begin{split}
\PP_0\|r_{\eta_1}-r_{\eta_2}\|^2
&\le
C E_0\int_0^\tau
Y(t)\{\chi_{\eta_1}(t,\bX,\bZ)
-\chi_{\eta_2}(t,\bX,\bZ)\}^2dt\\
&\le C'd^2(\eta_1,\eta_2).
\end{split}
\]

The localized classes have uniformly bounded envelopes,
although their envelopes need not be of order $\delta$.
Indeed, Assumptions~\ref{ass:A2}--\ref{ass:A3} and \ref{ass:Anew} imply uniform boundedness
of the linear predictors and of the least-favorable
projection $\bg^*$. Consequently, the likelihood and
projected-score increments admit envelopes bounded by a
constant independent of $n$ and $\delta$.
Furthermore, the mean-value theorem, Cauchy--Schwarz,
and the compensator identity imply
\[
\sup_{d(\eta,\eta_o)\le\delta}
\PP_0\|r_\eta-r_{\eta_o}\|^2
\lesssim\delta^2.
\]
The analogous localized second-moment bound holds for
the likelihood-increment class.
Since $\bg^*$ is fixed and bounded, the map from the
linear predictor $\chi_\eta$ to the projected-score
increment is Lipschitz with respect to the supremum
norm. Therefore, the sparse-ReLU covering-number
bound implies
\[
\sup_Q\log N\left(
\varepsilon\|F\|_{Q,2},
\mathcal G_\delta,L_2(Q)
\right)
\lesssim
(s+1)\log(U/\varepsilon),
\qquad 0<\varepsilon\le1,
\]
where $F$ is a uniformly bounded envelope of
$\mathcal G_\delta$. Let $\mathbb G_n=\sqrt n(\PP_n-\PP_0)$ denote
the empirical process.
Applying the bounded-envelope VC-type maximal
inequality yields
\[
\begin{split}
&E^*\sup_{d(\eta,\eta_o)\le\delta}
\left\|
\mathbb G_n(r_\eta-r_{\eta_o})
\right\|\\
&\qquad\lesssim
\delta\sqrt{(s+1)\log(U/\delta)}
+
\frac{(s+1)\log(U/\delta)}{\sqrt n}.
\end{split}
\]
The same argument applies to the localized
likelihood-increment class, completing the proof.
\end{proof}

The following lemma establishes equivalence of expectations under two models whose hazard functions differ only through the nuisance function $g$. Because the hazards are uniformly bounded above and below, the distributions under $g_0$ and $g_1$ are mutually absolutely continuous with Radon–Nikodym derivatives uniformly bounded.  This equivalence  allows us to control expectations under one model using expectations under another.

\begin{restatable}{lemma}{EvalAtTEquivalence}\label{lem:evalT-equivalence}
Suppose Assumptions \ref{ass:A1}--\ref{ass:A2} and \ref{ass:A4} hold. Let $g_0,g_1 \in \mathcal{H}(q,\boldsymbol{\alpha},{\bf d}^*,\tilde{{\bf d}},M)$ and fix $\btheta \in \Theta$. For $j=0,1$, let $\PP_{g_j}$ denote the distribution of $(T,\bX,\bZ,\Delta)$
under event hazard $h_{g_j}(t\mid \bX,\bZ)=\exp\{g_j(t,\bX)+\btheta^{\top}\bZ\}$ and 
a common censoring hazard $\mu(t\mid \bX,\bZ)$.  As a consequence of Assumptions~\ref{ass:A1}--\ref{ass:A2}, there exists a constant
$0< D<\infty$ such that, almost surely,
$$
e^{-D} \;\le\; h_{g_j}(t\mid\boldsymbol{X},\boldsymbol{Z}) 
\;\le\; e^D \, ,
\qquad \mbox{for all} \,\,  t\in[0,\tau] \, ,\ j=0,1 \, .
$$
Define $c=e^{-D}$ and $C=e^{D}$.
Then, for any nonnegative measurable function $q:[0,\tau]\times\mathcal X\to[0,\infty)$ with 
$\mathbb{E}_{\PP_{g_0}}[q(T,\bX)] < \infty$,
$$
K_-\,\EE_{\PP_{g_0}}\!\big\{q(T,\bX)\big\}
\;\le\;
\EE_{\PP_{g_1}}\!\big\{ q(T,\bX)\big \}
\;\le\;
K_+\,\EE_{\PP_{g_0}}\!\big \{ q(T,\bX)\big\} \, ,
$$
where
$$
K_- = \frac{c}{C}\,e^{-\tau(C-c)} \, ,
\qquad
K_+= \frac{C}{c}\,e^{\,\tau(C-c)} \, .
$$
In particular,
\(
\EE_{\PP_{g_1}}[q(T,\bX)]\asymp \EE_{\PP_{g_0}}[q(T,\bX)]
\)
with constants depending only on $(c,C,\tau)$.
\end{restatable}

\begin{proof}[Proof of \Cref{lem:evalT-equivalence}]
We work conditional on the covariates $(\bx,\bz)$ and then integrate over their distribution.
Under model $g_j$, the event time $U_j$ has density and survival
\[
f^U_j(t\mid \bx,\bz)=h_j(t\mid \bx,\bz)\,S^U_j(t\mid \bx,\bz),
\qquad
S^U_j(t\mid \bx,\bz)=\exp\!\Big\{-\!\int_0^t h_j(s\mid \bx,\bz)\,ds\Big\}.
\]
The censoring time $C$ has density $f^C(t\mid \bx,\bz)=\mu(t\mid \bx,\bz)S^C(t\mid \bx,\bz)$ 
and survival $S^C(t\mid \bx,\bz)$, common to both models. The observed time and event indicator are
$T_j=\min(U_j,C)$ and $\Delta_j=\mathbf{1}\{U_j\le C\}$. Given $(\bx,\bz)$, their joint density under $g_j$ is
$$
f_j(t,\delta\mid \bx,\bz)=
\begin{cases}
h_j(t\mid \bx,\bz)\,S^U_j(t\mid \bx,\bz)\,S^C(t\mid \bx,\bz)\, , & \delta=1\\[0.4em]
\mu(t\mid \bx,\bz)\,S^C(t\mid \bx,\bz)\,S^U_j(t\mid \bx,\bz)\, , & \delta=0 \, .
\end{cases}
$$
By Assumptions~\ref{ass:A1}--\ref{ass:A2},  
$$
ct \le \int_0^t h_j(s\mid \bx,\bz)\,ds \le Ct,
\qquad
e^{-Ct} \le S^U_j(t\mid \bx,\bz) \le e^{-ct}.
$$
The likelihood ratio of $(T_j,\Delta_j)$ under $g_1$ versus $g_0$, conditional on $(\bx,\bz)$, is
$$
L(t,\delta\mid \bx,\bz)
=
\frac{f_1(t,\delta\mid \bx,\bz)}{f_0(t,\delta\mid \bx,\bz)}
=
\left\{\frac{h_1(t\mid \bx,\bz)}{h_0(t\mid \bx,\bz)}\right\}^\delta
\frac{S^U_1(t\mid \bx,\bz)}{S^U_0(t\mid \bx,\bz)}.
$$
Since $c\le h_j\le C$,
$$
\frac{c}{C} \le \frac{h_1(t\mid \bx,\bz)}{h_0(t\mid \bx,\bz)} \le \frac{C}{c}.
$$
Moreover,
$$
\frac{S^U_1(t\mid \bx,\bz)}{S^U_0(t\mid \bx,\bz)}
=\exp\!\Big[-\!\int_0^t \{h_1(s\mid \bx,\bz)-h_0(s\mid \bx,\bz)\}\,ds\Big],
$$
and because $|h_1-h_0|\le C-c$,
$$
e^{-(C-c)t} \le \frac{S^U_1(t\mid \bx,\bz)}{S^U_0(t\mid \bx,\bz)}
\le e^{(C-c)t}.
$$
For $t\le\tau$ this yields
$$
\frac{c}{C}\,e^{-(C-c)\tau}
\;\le\;
L(t,\delta\mid \bx,\bz)
\;\le\;
\frac{C}{c}\,e^{(C-c)\tau}.
$$
Thus the Radon--Nikodým derivative satisfies
\[
K_- \;\le\; 
\frac{d\PP_{g_1}}{d\PP_{g_0}}(T,\Delta,\bX,\bZ)
\;\le\;
K_+,
\qquad \PP_{g_0}\text{--a.s.},
\]
where 
$K_-=\frac{c}{C}e^{-\tau(C-c)}$ and $K_+=\frac{C}{c}e^{\tau(C-c)}$.
For any nonnegative measurable function $q$ with finite expectation under $\PP_{g_0}$,
\[
\EE_{\PP_{g_1}}\{q(T,\bX)\}
=
\EE_{\PP_{g_0}}\!\left\{q(T,\bX)\,\frac{d \PP_{g_1}}{d \PP_{g_0}}(T,\Delta,\bX,\bZ)\right\}
\]
and the bounds on the likelihood ratio give
$$
K_-\,\EE_{\PP_{g_0}}\{q(T,\bX)\}
\;\le\;
\EE_{\PP_{g_1}}\{q(T,\bX)\}
\;\le\;
K_+\,\EE_{\PP_{g_0}}\{q(T,\bX)\} \, .
$$
\end{proof}


\begin{lemma}\label{lem:g-vs-chi}
Assume \ref{ass:A2}. 
Let $\eta_j=(\btheta_j,g_j)$, $j=1,2$, be two parameter values in 
$(\Theta \times \mathcal{H}(q,\boldsymbol{\alpha},{\bf d}^*,\tilde{{\bf d}},M))
\cup
(\Theta \times \mathcal{G}(K,s,\boldsymbol{p},D))$.
Define
\[
\chi_{\eta_j}(t,\bX,\bZ)
=
g_j(t,\bX) + \btheta_j^\top\bZ ,
\qquad j=1,2.
\]
Assume that $\|\bZ\|\le B_\bZ$ a.s. for some finite constant $B_\bZ$.
Then
\[
\EE\Bigl[
  \Delta\bigl\{g_1(T,\bX)-g_2(T,\bX)\bigr\}^2
\Bigr]
\;\le\;
2\,d(\eta_1,\eta_2)^2
\;+\;
2 B_\bZ^2 \|\btheta_1-\btheta_2\|^2 \, .
\]
\end{lemma}

\begin{proof}[Proof of \Cref{lem:g-vs-chi}]
Fix $(T,\Delta,\bX,\bZ)$ and set
$a= g_1(T,\bX) - g_2(T,\bX)$, 
$b =(\btheta_1-\btheta_2)^\top\bZ$.
Then
$\chi_{\eta_1}(T,\bX,\bZ)-\chi_{\eta_2}(T,\bX,\bZ)= a + b$.
Use the elementary inequality
$$
(u-v)^2 \;\le\; 2u^2 + 2v^2
\qquad\text{for all }u,v\in\mathbb{R},
$$
with $u = a+b$ and $v = b$. This gives
$$
a^2
=
(a+b-b)^2
\;\le\;
2(a+b)^2 + 2b^2.
$$
Multiplying by $\Delta\in\{0,1\}$ preserves the inequality
$$
\Delta\,a^2
\;\le\;
2\Delta\,(a+b)^2 + 2\Delta\,b^2 \, .
$$
Taking expectations and using the definition of $d(\eta_1,\eta_2)$,
\begin{align*}
\EE\Bigl[\Delta\{g_1(T,\bX)-g_2(T,\bX)\}^2\Bigr]
&= \EE(\Delta a^2) \\[0.3em]
&\le 2\,\EE\{\Delta(a+b)^2\} + 2\,\EE(\Delta b^2) \\[0.3em]
&= 2\,d(\eta_1,\eta_2)^2 + 2\,\EE(\Delta b^2).
\end{align*}
For the second term, by boundedness of $\bZ$,
$$
|b|
= |(\btheta_1-\btheta_2)^\top\bZ|
\le \|\btheta_1-\btheta_2\|\,\|\bZ\|
\le B_\bZ\,\|\btheta_1-\btheta_2\|,
$$
and hence
\[
b^2 \;\le\; B_\bZ^2\,\|\btheta_1-\btheta_2\|^2.
\]
Since $\Delta\le 1$,
$$
\EE(\Delta b^2)
\le \EE(b^2)
\le B_\bZ^2\,\|\btheta_1-\btheta_2\|^2 \, .
$$
Putting everything together,
\[
\EE\Bigl[\Delta\{g_1(T,\bX)-g_2(T,\bX)\}^2\Bigr]
\le
2\,d(\eta_1,\eta_2)^2
+ 2B_\bZ^2\,\|\btheta_1-\btheta_2\|^2 \, .
\]
\end{proof}


\begin{lemma}\label{lem:thetaOp}
If Assumptions \ref{ass:A1}--\ref{ass:A5} hold, then
there exists a constant $c_\theta>0$ such that, for all parameters
$\eta_1=(\btheta_1,g_1)$ and $\eta_2=(\btheta_2,g_2)$,
$$
d(\eta_1,\eta_2)^2 \;\ge\; c_\theta\,\|\btheta_1-\btheta_2\|^2 \, .
$$
In particular, one may take
$$
c_\theta={\Pr}_0(\Delta=1)\lambda_{\min}(\Sigma_{\bZ})>0.
$$
\end{lemma}
\begin{proof}[Proof of \Cref{lem:thetaOp}]
Let $\bbeta=\btheta_1-\btheta_2$, $a(T,\bX)=g_1(T,\bX)-g_2(T,\bX)$,
$b=\bbeta^\top \bZ$. Then
\[
\chi_{\eta_1}(T,\bX,\bZ)-\chi_{\eta_2}(T,\bX,\bZ)=a(T,\bX)+b,
\]
and hence
\[
d(\eta_1,\eta_2)^2
=
\EE_0\!\left[\Delta\{a(T,\bX)+b\}^2\right].
\]
Because $a(T,\bX)$ is measurable with respect to $(T,\bX)$, we have
\[
\EE_0\!\left[\Delta\{a+b\}^2 \mid T,\bX\right]
=
\Pr(\Delta=1\mid T,\bX)\,
\EE_0\!\left[\{a+b\}^2 \mid T,\bX,\Delta=1\right].
\]
By the conditional variance decomposition,
\[
\EE_0\!\left[\{a+b\}^2 \mid T,\bX,\Delta=1\right]
=
\Big(a+\EE_0[b\mid T,\bX,\Delta=1]\Big)^2
+
Var_0(b\mid T,\bX,\Delta=1).
\]
Therefore
\[
\EE_0\!\left[\Delta\{a+b\}^2\right]
=
\EE_0\!\left[
\Pr(\Delta=1\mid T,\bX)
\Big(a+\EE_0[b\mid T,\bX,\Delta=1]\Big)^2
\right]
+
\EE_0\!\left[
\Delta\,Var_0(b\mid T,\bX,\Delta=1)
\right].
\]
Since the first term is nonnegative,
\[
d(\eta_1,\eta_2)^2
\ge
\EE_0\!\left[
\Delta\,Var_0(b\mid T,\bX,\Delta=1)
\right].
\]
Now, because $b=\bbeta^\top \bZ$,
\[
Var_0(b\mid T,\bX,\Delta=1)
=
\bbeta^\top
Var_0(\bZ\mid T,\bX,\Delta=1)
\bbeta,
\]
and thus
\[
\EE_0\!\left[
\Delta\,Var_0(b\mid T,\bX,\Delta=1)
\right]
=
\bbeta^\top
\EE_0\!\left[
\Delta\,
\Big\{\bZ-\EE_0(\bZ\mid T,\bX,\Delta=1)\Big\}^{\otimes 2}
\right]
\bbeta
=
{\Pr}_{0}(\Delta=1)\,
\bbeta^\top \Sigma_{\bZ}\bbeta.
\]
By Assumption \ref{ass:A5},
\[
\bbeta^\top \Sigma_{\bZ}\bbeta
\ge
\lambda_{\min}(\Sigma_{\bf Z})\,\|\bbeta\|^2.
\]
Hence
\[
d(\eta_1,\eta_2)^2
\ge
{\Pr}_0(\Delta=1) \lambda_{\min}(\Sigma_{\bf Z})\,\|\btheta_1-\btheta_2\|^2.
\]
This proves the lemma with $c_\theta= \Pr_0(\Delta=1) \lambda_{\min}(\Sigma_{\bf Z})$.
\end{proof}
\RateTheorem*

\begin{proof}[Proof of \Cref{thm:rate}]
Let $\ell(\eta;\bW)$, $\bW = (T,\Delta,\bX^\top,\bZ^\top)^\top$,
 denote the per–observation log–likelihood contribution,
so that $\ell_n(\eta)=\PP_n\ell(\eta;\cdot)$  is the empirical loss and
$\ell_0(\eta)=\PP_0\ell(\eta;\cdot)$ is the population loss.
Let $\GG_n=\sqrt{n}(\PP_n-\PP_0)$ be the empirical process.
Recall the distance
$$
d(\eta_1,\eta_2)
\;=\;
\Big(
  \EE_0\big[
    \Delta\{\chi_{\eta_1}(T,\bX,\bZ)-\chi_{\eta_2}(T,\bX,\bZ)\}^2
  \big]
\Big)^{1/2},
\qquad
\eta=(\btheta,g),
$$
and let $\eta_o=(\btheta_o,g_o)$ denote the true parameter. By Lemma~\ref{lem:unified_maximal}, there exists a constant $C>0$ such that
for every $0<\delta\le 1$,
\begin{equation}\label{eq:Psi-proof}
\EE^{*}\!\sup_{d(\eta,\eta_o)\le\delta}
      \bigl| \GG_n\big\{\ell(\eta;\cdot)-\ell(\eta_o;\cdot)\big\} \bigr|
   \;\le\;
   \Psi_n(\delta)
   \equiv
   \delta\,\sqrt{s\log \frac{\mathcal{U}}{\delta}}
   + \frac{s}{\sqrt n}\log \frac{\mathcal{U}}{\delta} \, ,
\end{equation}
and by Markov's inequality, this implies
$$
\sup_{d(\eta,\eta_o)\le\delta}
\bigl|
 \GG_n\big\{\ell(\eta;\cdot)-\ell(\eta_o;\cdot)\big\}
\bigr|
=
O_p\bigl(\Psi_n(\delta)\bigr) \, .
$$
Set $\tau_n = \gamma_n \log^2 n$ as in the theorem.
By Assumption~\ref{ass:A3} and the specific choice of $(K,s,\boldsymbol{p})$
(cf.\ the construction in \citet{schmidthieber2020nonparametric}
and \citet{zhong2022deep}),
we have
$$
s\log\!\Bigl(\frac{\mathcal{U}}{\tau_n}\Bigr)
= O\!\bigl(n\,\gamma_n^2\log n\bigr) \, ,
$$
so that
$$
\frac{1}{\sqrt{n}}\,\tau_n^{-2}\,\Psi_n(\tau_n)
=
\tau_n^{-1}\sqrt{\frac{s}{n}\log \frac{\mathcal{U}}{\tau_n}}
+
\frac{s}{n\,\tau_n^{2}}\log \frac{\mathcal{U}}{\tau_n}
\;\lesssim\;
1
$$
for $n$ large enough. Hence, there exists a constant $c>0$ such that
\begin{equation}\label{eq:balance1}
\Psi_n(\tau_n)
\;\le\;
\frac{c}{2}\,\sqrt{n}\,\tau_n^2
\quad\mbox{and}\quad
\frac{1}{\sqrt{n}}\Psi_n(\tau_n)
\;\le\;
\frac{c}{2}\,\tau_n^2 \, .
\end{equation}
Consequently,
\begin{equation}\label{eq:emp-fluct}
\sup_{d(\eta,\eta_o)\le\tau_n}
\bigl|
  (\ell_n-\ell_0)(\eta) - (\ell_n-\ell_0)(\eta_o)
\bigr|
=
\sup_{d(\eta,\eta_o)\le\tau_n}
\frac{1}{\sqrt{n}}
\bigl|
  \GG_n\big\{\ell(\eta;\cdot)-\ell(\eta_o;\cdot)\big\}
\bigr|
= O_p\bigl(\tfrac{1}{\sqrt{n}}\Psi_n(\tau_n)\bigr)
= O_p(\tau_n^2),
\end{equation}
where we used \eqref{eq:Psi-proof} and \eqref{eq:balance1}.

By Theorem~1 of \citet{schmidthieber2020nonparametric}, applied to the composite
Hölder class $\mathcal{H}(q,\boldsymbol{\alpha},{\bf d}^*,\tilde{\bf d},M)$
with the choice of $(K,s,\boldsymbol{p},D)$ given in Assumption~\ref{ass:A3},
there exists $\tilde g_1 \in \mathcal{G}(K,s,\boldsymbol{p},D)$ such that
$$
\EE\bigl\{ \tilde g_1(T,\bX) - g_o(T,\bX)\bigr\}^2
\le C_1\,\tau_n^2 \, ,
\qquad
\EE\bigl\{ \Delta\big(\tilde g_1(T,\bX) - g_o(T,\bX)\big)^2\bigr\}
\le C_1\,\tau_n^2
$$
for some constant $C_1>0$ independent of $n$.
Define $\tilde\eta_1=(\btheta_o,\tilde g_1)$.
Then, 
$$
d^2(\tilde\eta_1,\eta_o)
=
\EE\bigl[
  \Delta\big\{
    \chi_{\tilde\eta_1}(T,\bX,\bZ)
    -\chi_{\eta_o}(T,\bX,\bZ)
  \big\}^2
\bigr]
=
\EE\bigl[
  \Delta\big\{ \tilde g_1(T,\bX)-g_o(T,\bX)\big\}^2
\bigr]
\le C_1\,\tau_n^2 \, ,
$$
since the $\btheta$–component is the same in both $\tilde\eta_1$ and $\eta_o$.
Therefore,
\begin{equation}\label{eq:tildeeta-radius}
d(\tilde\eta_1,\eta_o)
\;\le\;
\sqrt{C_1}\,\tau_n \, ,
\end{equation}
and in particular $\tilde\eta_1$ lies inside a ball of radius proportional to $\tau_n$ around $\eta_o$.

Let $\mathcal F_D=\Theta\times\mathcal G(K,s,\mathbf p,D)$
denote the sieve parameter space in (3.3), with $D=D_n$ as
in Assumption~\ref{ass:A3}. The curvature argument in the proof of
Lemma~\ref{lmm:concavity} also applies to this sieve under the uniform bounds
in \ref{ass:A2}--\ref{ass:A3}, and yields, for every $\eta\in\mathcal F_D$,
$$
\ell_0(\eta)-\ell_0(\eta_o)
\;\asymp\;
-\,d(\eta,\eta_o)^2 \, ,
$$
that is, there exist constants $0<c_-,c_+<\infty$ such that
\[
-\,c_+\,d(\eta,\eta_o)^2
\;\le\;
\ell_0(\eta)-\ell_0(\eta_o)
\;\le\;
-\,c_-\,d(\eta,\eta_o)^2 \, .
\]
Applying this to $\eta=\tilde\eta_1$ and using \eqref{eq:tildeeta-radius}, we obtain
\begin{equation}\label{eq:pop-gap}
\ell_0(\tilde\eta_1)-\ell_0(\eta_o)
\;\ge\;
-\,c_+\,d(\tilde\eta_1,\eta_o)^2
\;\ge\;
-\,c_+ C_1\,\tau_n^2
\;=\;
-\,C_2\,\tau_n^2 \, ,
\end{equation}
for some constant $C_2>0$.

By combining \eqref{eq:emp-fluct} and \eqref{eq:pop-gap} we get
\begin{align}
\ell_n(\tilde\eta_1) - \ell_n(\eta_o)
&=
\bigl\{\ell_n(\tilde\eta_1)-\ell_0(\tilde\eta_1)\bigr\}
-
\bigl\{\ell_n(\eta_o)-\ell_0(\eta_o)\bigr\}
+
\bigl\{\ell_0(\tilde\eta_1)-\ell_0(\eta_o)\bigr\}
\notag\\[0.3em]
&\ge
-\,\sup_{d(\eta,\eta_o)\le\tau_n}
   \bigl| (\ell_n-\ell_0)(\eta)-(\ell_n-\ell_0)(\eta_o)\bigr|
\;-\; C_2\,\tau_n^2
\notag\\[0.3em]
&= -\,O_p(\tau_n^2).\label{eq:approx-gap}
\end{align}
Since $\widehat\eta=(\widehat\btheta,\widehat g)$ is the maximizer of the empirical likelihood over the sieve, then
\begin{equation}\label{eq:likelihood-chain}
\ell_n(\widehat\eta)
=
\max_{\eta}\ell_n(\eta)
\;\ge\;
\ell_n(\tilde\eta_1)
\;\ge\;
\ell_n(\eta_o) - O_p(\tau_n^2) \, .
\end{equation}
On the other hand, the same population curvature bound gives
$$
\ell_0(\eta)-\ell_0(\eta_o)
\;\le\;
-\,c_-\,d(\eta,\eta_o)^2
\quad\text{for all }\eta\in\mathcal{F}_D \, .
$$
The empirical-process bound \eqref{eq:emp-fluct} and the likelihood chain \eqref{eq:likelihood-chain}
verify the conditions (i)–(ii) of Theorem~3.4.1 in \citet{vdvwellner1996} with radius of order $\tau_n$.
Therefore,
$$
d(\widehat\eta,\eta_o)
=
O_p(\tau_n) \, .
$$
Finally, by writing 
$$
\chi_{\widehat\eta}(t,\bX,\bZ)-\chi_{\eta_o}(t,\bX,\bZ)
=
\{\widehat g(t,\bX)-g_o(t,\bX)\}
+
(\widehat\btheta-\btheta_o)^\top\bZ \, ,
$$
and applying Lemma~\ref{lem:g-vs-chi} with 
$\eta_1 = \widehat\eta = (\widehat\btheta,\widehat g)$
and $\eta_2 = \eta_o = (\btheta_o,g_o)$, we obtain
\begin{equation}\label{eq:basedonLem6}
\EE\Bigl[\Delta\{\widehat g(T,\bX)-g_o(T,\bX)\}^2\Bigr]
\;\le\;
C\Bigl(
  d(\widehat\eta,\eta_o)^2
  + \|\widehat\btheta-\btheta_o\|^2
\Bigr) \, .   
\end{equation}
Applying Lemma \ref{lem:thetaOp} with $\eta_1=\widehat\eta$ and $\eta_2=\eta_o$ yields
$$
\|\widehat\btheta-\btheta_o\|^2\le c_\theta^{-1} d(\widehat\eta,\eta_o)^2 = O_p(\tau_n^2) \, .
$$ Hence,
$\|\widehat\btheta-\btheta_o\|=O_p(\tau_n)$, and
the right-hand side of (\ref{eq:basedonLem6}) is $O_p(\tau_n^2)$, 
which is exactly the claim of Theorem~\ref{thm:rate}.  The expectation is taken with respect to a new observation $(T,\Delta,\bX,\bZ)\sim \PP_0$, while the $O_p(\cdot)$
is with respect to the randomness in the training sample used to construct $\widehat g$.
\end{proof}

\LowerBoundTheorem*

\begin{proof}[Proof of \Cref{thm:minimax-delta}]
For $g\in\mathcal{H}(q,\boldsymbol{\alpha},{\bf d}^*,\tilde{{\bf d}},M)$ and $\btheta_o\in\Theta$, let $\PP_{g,\theta_o}$ denote the joint distribution of 
$(T,\Delta,\bX,\bZ)$ induced by the hazard
$h_{g,\theta_o}(t\mid\bX,\bZ)=\exp\{g(t,\bX)+\btheta_o^\top\bZ\}$  
together with a censoring hazard that is held fixed across all $g$.
By the standard Kullback-Leibler (KL) distance for  bounded hazards (Assumptions~\ref{ass:A1}--\ref{ass:A2}), 
there exists $C>0$ such that 
for any $g_0,g_1\in\mathcal H(q,\boldsymbol\alpha,
{\bf d}^\ast,\widetilde{\bf d},M)$,
\[
\mathrm{KL}\!\left(
\PP_{g_1,\theta_o}^{\otimes n}
\,\|\, 
\PP_{g_0,\theta_o}^{\otimes n}
\right)
=
n\,\mathrm{KL}\!\left(
\PP_{g_1,\theta_o}
\,\|\, 
\PP_{g_0,\theta_o}
\right)
\le
C n\,
\EE_{\PP_{g_1,\theta_o}}
\left[
\Delta
\{g_1(T,\bX)-g_0(T,\bX)\}^2
\right].
\]
Since $\Delta\le1$,
\[
\mathrm{KL}\!\left(
\PP_{g_1,\theta_o}^{\otimes n}
\,\|\, 
\PP_{g_0,\theta_o}^{\otimes n}
\right)
\le
C n\,
\EE_{\PP_{g_1,\theta_o}}
\left[
\{g_1(T,\bX)-g_0(T,\bX)\}^2
\right].
\]
By Lemma~~\ref{lem:evalT-equivalence},
\[
\mathrm{KL}\!\left(
\PP_{g_1,\theta_o}^{\otimes n}
\,\|\, 
\PP_{g_0,\theta_o}^{\otimes n}
\right)
\le
C' n\,
\EE_{\PP_{g_0,\theta_o}}
\left[
\{g_1(T,\bX)-g_0(T,\bX)\}^2
\right].
\]
By the packing construction for composite Hölder networks
\citep{schmidthieber2020nonparametric}, 
there exist functions $g^{(0)},\dots,g^{(N)}\in \mathcal{H}(q,\boldsymbol{\alpha},{\bf d}^*,\tilde{{\bf d}},M)$  and constants
$c_1,c_2>0$ such that
\begin{equation}\label{eq:packingg}
\|g^{(j)}-g^{(k)}\|_{L^2(\PP_{g^{(0)},\theta_o})}
\;\ge\;
2c_1\,\gamma_n \, ,
\qquad j\neq k \, ,
\end{equation}
and
\begin{equation}\label{eq:packing-average}
\frac{1}{N}\sum_{j=1}^N 
\|g^{(j)}-g^{(0)}\|_{L^2(\PP_{g^{(0)},\theta_o})}^2
\;\le\;
\frac{c_2\,\log N}{n} \, .
\end{equation}

Applying the preceding KL bound with
$g_0=g^{(0)}$ and $g_1=g^{(j)}$, and using \eqref{eq:packing-average},
\[
\frac{1}{N}\sum_{j=1}^N
\mathrm{KL}\!\left(
\PP_{g^{(j)},\theta_o}^{\otimes n}
\,\|\, 
\PP_{g^{(0)},\theta_o}^{\otimes n}
\right)
\le
C'n
\frac{1}{N}\sum_{j=1}^N
\|g^{(j)}-g^{(0)}\|_{L^2(\PP_{g^{(0)},\theta_o})}^2
\le
C'c_2\log N.
\]

Choose $\alpha\in(0,1/8)$ so that
\(C'c_2 \le \alpha\). Then,
$$
\frac{1}{N}\sum_{j=1}^N 
\mbox{KL}(\PP^{\otimes n}_{g^{(j)},\theta_o}\,\|\,\PP^{\otimes n}_{g^{(0)},\theta_o})
\;\le\; \alpha \log N \, .
$$
Applying Tsybakov’s Fano inequality to  
$\{g^{(0)},\dots,g^{(N)}\}$ with separation \eqref{eq:packingg}, we obtain that there exists 
$c>0$ (depending only on $c_1$ and $\alpha$) such that for any estimator $\widehat g$,
$$
\sup_{j=0,\dots,N}\
\EE_{\PP^{\otimes n}_{g^{(j)},\theta_o}}
\!\Big(
  \|\widehat g - g^{(j)}\|_{L^2(\PP_{g^{(0)},\theta_o})}^2
\Big)
\;\ge\;
c\,\gamma_n^2 \, .
$$ 
Since the supremum over $g_o\in\mathcal H$ dominates this finite set,
$$
\inf_{\widehat g}\ 
\sup_{g_o\in\mathcal H}
\EE_{\PP^{\otimes n}_{g_o,\theta_o}}
\!\Bigl(
  \|\widehat g - g_o\|_{L^2(\PP_{g^{(0)},\theta_o})}^2
\Bigr)
\;\ge\;
c\,\gamma_n^2 \, .
$$
By Lemma~\ref{lem:evalT-equivalence}, there exists a constant
$C_{\mathrm{eq}}>0$ (depending only on \ref{ass:A1}--\ref{ass:A2}) such that, uniformly over $g_o\in\mathcal H$ and all sieve estimators $\widehat g$,
$$
\EE_{\PP^{\otimes n}_{g^{(0)},\theta_o}}
\!\big[\{\widehat g(T,\bX)-g_o(T,\bX)\}^2\big]
\;\le\;
C_{\mathrm{eq}}\,
\EE_{\PP^{\otimes n}_{g_o,\theta_o}}
\!\big[\{\widehat g(T,\bX)-g_o(T,\bX)\}^2\big].
$$
Combining this with the previous display yields
$$
\inf_{\widehat g}\ 
\sup_{g_o\in\mathcal H}
\EE_{\PP^{\otimes n}_{g_o,\theta_o}}
\!\big[\{\widehat g(T,\bX)-g_o(T,\bX)\}^2\big]
\;\ge\;
\frac{c}{C_{\mathrm{eq}}}\,\gamma_n^2.
$$
Finally, since the supremum in the theorem is also taken over $\btheta_o\in\Theta$, and the above bound holds for each fixed $\theta_o$, we obtain
\begin{equation}\label{eq:unweighted-lower}
\inf_{\widehat g}\ 
\sup_{\theta_o\in\Theta,\ g_o\in\mathcal H}
\EE_{\PP^{\otimes n}_{g_o,\theta_o}}
\!\big[\{\widehat g(T,\bX)-g_o(T,\bX)\}^2\big]
\;\ge\;
c_1\,\gamma_n^2
\end{equation}
for some $c_1>0$.
By Lemma~\ref{lem:delta-zero-implies-unweighted-zero} and Assumptions
\ref{ass:A1}, \ref{ass:A2}, \ref{ass:A4}, there exists a constant
$C_\Delta>0$ such that, for all $g_o\in\mathcal H$ and all $\widehat g$,
$$
\EE_{\PP^{\otimes n}_{g_o,\theta_o}}
\!\big[\{\widehat g(T,\bX)-g_o(T,\bX)\}^2\big]
\;\le\;
C_\Delta\,
\EE_{\PP^{\otimes n}_{g_o,\theta_o}}
\!\big[\Delta\{\widehat g(T,\bX)-g_o(T,\bX)\}^2\big] \, .
$$
Combining this with \eqref{eq:unweighted-lower} gives
$$
\inf_{\widehat g}\ 
\sup_{\theta_o\in\Theta,\ g_o\in\mathcal H}
\EE_{\PP^{\otimes n}_{g_o,\theta_o}}
\!\big[\Delta\{\widehat g(T,\bX)-g_o(T,\bX)\}^2\big]
\;\ge\;
\frac{c_1}{C_\Delta}\,\gamma_n^2 \, .
$$
Renaming the constant $c=c_1/C_\Delta$ completes the proof.
\end{proof}

\EfficientScoreAndInformationBound*

\begin{proof}[Proof of \Cref{thm:eff-score}]
Let ${\bf u}\in\mathbb{R}^p$ and let $\{g_s:s\in(-1,1)\}\in\mathcal{H}_{g_o}$
be a one-dimensional regular submodel through $g_o$, with
\[
\dot g_0 = \left.\frac{\partial g_s}{\partial s}\right|_{s=0}\in\mathcal{T}_{g_o}.
\]
Consider the semiparametric submodel
\[
\{\PP_{\theta_o+s {\bf u},\; g_s}: s\in(-1,1)\}.
\]
For one observation $(T,\Delta,\bX,\bZ)$,
\[
\ell(\btheta,g)
=
\Delta\{g(T,\bX)+\btheta^\top\bZ\}
-
\int_0^\tau Y(t)\exp\{g(t,\bX)+\btheta^\top\bZ\}\,dt.
\]
Differentiating at $s=0$ gives
\[
\left.\frac{d}{ds}\ell(\btheta_o+s {\bf u}\, ,\, g_s)\right|_{s=0}
=
{\bf u}^\top \dot\ell_{\theta_o}
+
\dot\ell_{g_o,\dot g_0},
\]
where
\[
\dot\ell_{\theta_o}
=
\int_0^\tau \bZ\,dM(t),
\qquad
\dot\ell_{g_o,g}
=
\int_0^\tau g(t,\bX)\,dM(t).
\]
Hence the nuisance score space is
\[
\mathcal{T}^{\mathrm{nuis}}_{g_o}
=
\{\dot\ell_{g_o,g}: g\in\mathcal{T}_{g_o}\},
\]
and the efficient score is
\[
\ell^*_{\theta_o}
=
\dot\ell_{\theta_o}-\Pi_{g_o}(\dot\ell_{\theta_o}),
\]
where $\Pi_{g_o}$ denotes projection onto
$\overline{\mathcal{T}^{\mathrm{nuis}}_{g_o}}$ in $L_2(\PP_0)$.

By Lemma~1 of \citet{sasieni1992information}, for any square-integrable predictable
functions $\mu_1,\mu_2$,
\[
E_0\!\left[
\left(\int_0^\tau \mu_1\,dM\right)
\left(\int_0^\tau \mu_2\,dM\right)
\right]
=
E_0\!\left[\Delta\, \mu_1(T,\bX)\mu_2(T,\bX)\right].
\]
Therefore the map
\[
J:\mu \mapsto \int_0^\tau \mu(t,\bX)\,dM(t)
\]
is an isometry from $L_2(\PP_0,\Delta)$ onto its image. Since
$\overline{\mathcal{T}}_{g_o}$ is a closed linear subspace of $L_2(\PP_0,\Delta)$
and $Z_j\in L_2(\PP_0,\Delta)$ by Assumption~\ref{ass:A2}, the Hilbert projection
theorem implies that for each $j=1,\dots,p$ there exists a unique
$g_j^*\in\overline{\mathcal{T}}_{g_o}$ such that
\[
\Pi_{g_o}(\dot\ell_{\theta_o,j})=\dot\ell_{g_o,g_j^*}.
\]
Let $\bg^*=(g_1^*,\dots,g_p^*)^\top$. Then
\[
\ell^*_{\theta_o}(T,\bX,\bZ,\Delta)
=
\int_0^\tau \{\bZ-\bg^*(t,\bX)\}\,dM(t).
\]
The orthogonality condition for the projection is
\[
E_0\!\left[
\{\dot\ell_{\theta_o,j}-\dot\ell_{g_o,g_j^*}\}\dot\ell_{g_o,g}
\right]
=0
\qquad
\forall g\in\mathcal{T}_{g_o},
\]
which, by the isometry above, is equivalent to
\[
E_0\!\left[
\Delta\{Z_j-g_j^*(T,\bX)\}g(T,\bX)
\right]
=0
\qquad
\forall g\in\mathcal{T}_{g_o}.
\]
By density, the same holds for all $g\in\overline{\mathcal{T}}_{g_o}$.
Thus $g_j^*$ is the $\Delta$-weighted $L_2(\PP_0)$ projection of $Z_j$ onto
$\overline{\mathcal{T}}_{g_o}$, and equivalently
\[
\bg^*
\in
\arg\min_{\bg\in(\overline{\mathcal{T}}_{g_o})^p}
E_0\!\left[\Delta\|\bZ-\bg(T,\bX)\|_2^2\right].
\]
Finally, by the martingale isometry,
\[
I(\btheta_o)
=
E_0\!\left[\ell_{\theta_o}^{*\otimes 2}\right]
=
E_0\!\left[\Delta\{\bZ-\bg^*(T,\bX)\}^{\otimes 2}\right].
\]
Finally, under the regularity conditions in Assumptions 
(A1)--(A4), existence and uniqueness of the weighted
projection $\bg^*$ in $\overline {\mathcal{T}}_{g_o}^p$ 
follow from the approximation theorems of \citet{stone1985additive}  
and the general semiparametric projection theory in
Appendix~A.4 of \citet{bickel1993efficient}.  Therefore, both the efficient score and the efficient information matrix
are well-defined and given by the formulas above.
\end{proof}

\ThetaHatAsmp*

\begin{proof}[Proof of \Cref{thm:asym-theta}]
By Theorem~1,
$d(\widehat\eta,\eta_o)=O_p(\tau_n)$,
$\widehat\eta=(\widehat\btheta,\widehat g)$.
Hence, by the second part of \ref{ass:Anew},
\begin{align}
\PP\{r_{\widehat\theta,\widehat g}
-r_{\theta_o,g_o}\}
&=
-I(\btheta_o)(\widehat\btheta-\btheta_o)
+O_p(\tau_n^2)
\nonumber\\
&=
-I(\btheta_o)(\widehat\btheta-\btheta_o)
+o_p(n^{-1/2}),
\label{eq:population-score-expansion}
\end{align}
where the last equality follows from
$\sqrt n\,\tau_n^2\to0$.
The localized empirical-process bound used in the proof of
Theorem~\ref{thm:rate} gives
\[
(\PP_n-\PP)\{r_{\widehat\theta,\widehat g}
-r_{\theta_o,g_o}\}
=o_p(n^{-1/2}).
\]
Therefore,
\begin{align*}
\PP_n r_{\widehat\theta,\widehat g}
={}&
\PP_n r_{\theta_o,g_o}
-I(\btheta_o)(\widehat\btheta-\btheta_o)
+o_p(n^{-1/2}).
\end{align*}
By the first part of \ref{ass:Anew},
\[
\PP_n r_{\widehat\theta,\widehat g}=o_p(n^{-1/2}),
\]
and consequently
\[
I(\btheta_o)(\widehat\btheta-\btheta_o)
=
\PP_n r_{\theta_o,g_o}
+o_p(n^{-1/2}).
\]
By Theorem~\ref{thm:eff-score},
$r_{\theta_o,g_o}=\ell_{\theta_o}^*$. Thus,
\[
\sqrt n(\widehat\btheta-\btheta_o)
=
I(\btheta_o)^{-1}
\frac1{\sqrt n}\sum_{i=1}^n
\ell_{\theta_o}^*(T_i,\bX_i,\bZ_i,\Delta_i)
+o_p(1).
\]
The efficient score has mean zero and covariance $I(\btheta_o)$.
The multivariate central limit theorem and Slutsky's theorem therefore
give
$\sqrt{n}(\widehat\btheta-\btheta_o)
\;\overset{D}{\longrightarrow}\;
N\bigl({\bf 0},I(\btheta_o)^{-1}\bigr)
$.
\end{proof}

\begin{lemma}\label{lem:numerical_error} Let
$0=t_0<t_1<\cdots<t_{m_n}=\tau$ be a grid with mesh size
$\delta_n=\max_{1\le j\le m_n}(t_j-t_{j-1})$, and suppose that the grid contains all observed times \(T_1,\dots,T_n\). For \(\eta=(\btheta,g)\),  let
$\mathcal{N}_n=\{\eta:d(\eta,\eta_o)\le M\tau_n\}$ for some fixed \(M>0\). Assume \ref{ass:A2}--\ref{ass:A3}. Also define the uniform time-modulus
$$
\omega_n(\delta)
=
\sup_{\eta\in \mathcal{N}_n}\sup_{\bX \in\mathcal X}\sup_{\substack{s,t\in[0,\tau]\\ |t-s|\le \delta}}
|g(t,\bX)-g(s,\bX)|.
$$
Then there exist constants \(C_1,C_2>0\), depending only on the bounds in \ref{ass:A2}--\ref{ass:A3}, such that
\[
\sup_{\eta\in \mathcal{N}_n}\bigl|\tilde\ell_n(\eta)-\ell_n(\eta)\bigr|
\le C_1\,\omega_n(\delta_n),
\]
and
\[
\sup_{\eta\in \mathcal{N}_n}\bigl\|\nabla_\theta \tilde\ell_n(\eta)-\nabla_\theta \ell_n(\eta)\bigr\|
\le C_2\,\omega_n(\delta_n).
\]
In particular, if \(\omega_n(\delta_n)=O(\delta_n)\), then both differences are \(O(\delta_n)\) uniformly on \(\mathcal{N}_n\).
\end{lemma}

\begin{proof}[Proof of \Cref{lem:numerical_error}]
Fix \(\eta=(\btheta,g)\in \mathcal{N}_n\).  Recall
\[
\ell_n(\eta)
=
\frac1n\sum_{i=1}^n
\left[
\Delta_i\chi_\eta(T_i,\bX_i,\bZ_i)
-
\int_0^\tau Y_i(t)e^{\chi_\eta(t,\bX_i,\bZ_i)}\,dt
\right],
\]
and
\[
\tilde\ell_n(\eta)
=
\frac1n\sum_{i=1}^n
\sum_{j:Y_i(t_{j-1})=1}
\left[
\Delta_{ij}\chi_\eta(t_j,\bX_i,\bZ_i)
-
\bigl\{\min(T_i,t_j)-t_{j-1}\bigr\}e^{\chi_\eta(t_j,\bX_i,\bZ_i)}
\right] \, ,
\]
where
$\chi_\eta(t,\bX,\bZ)=g(t,\bX)+\btheta^\top \bZ$.
Because the grid contains all observed times, for each \(i\) there exists an index
\(j(i)\) such that \(T_i=t_{j(i)}\). Hence
\[
\sum_{j:Y_i(t_{j-1})=1}\Delta_{ij}\chi_\eta(t_j,\bX_i,\bZ_i)
=
\Delta_i\chi_\eta(T_i,\bX_i,\bZ_i),
\]
so the event contribution is represented exactly. Therefore
\[
\tilde\ell_n(\eta)-\ell_n(\eta)
=
-\frac1n\sum_{i=1}^n R_i(\eta),
\]
where
\[
R_i(\eta)
=
\sum_{j:Y_i(t_{j-1})=1}
\bigl\{\min(T_i,t_j)-t_{j-1}\bigr\}e^{\chi_\eta(t_j,\bX_i,\bZ_i)}
-
\int_0^\tau Y_i(t)e^{\chi_\eta(t,\bX_i,\bZ_i)}\,dt .
\]
For each \(i\), define
$\psi_{i,\eta}(t)=Y_i(t)e^{\chi_\eta(t,\bX_i,\bZ_i)}$, $t\in[0,\tau]$.
On any interval \((t_{j-1},t_j]\) such that \(Y_i(t_{j-1})=1\), we have
\(Y_i(t)=1\) for \(t\in[t_{j-1},\min(T_i,t_j)]\), and thus
\[
\int_{t_{j-1}}^{\min(T_i,t_j)} \psi_{i,\eta}(t)\,dt
=
\int_{t_{j-1}}^{\min(T_i,t_j)} e^{\chi_\eta(t,\bX_i,\bZ_i)}\,dt.
\]
Since \(\Theta\) is compact, \(\|\bZ\|\le B_{\bZ}\) a.s. by \ref{ass:A2}, and \(\|g\|_\infty\le D_n=O(1)\) by \ref{ass:A3}, there exists a constant \(C<\infty\) such that uniformly over \(i\), \(\eta\in \mathcal{N}_n\), and \(t\in[0,\tau]\),
\[
e^{\chi_\eta(t,\bX_i,\bZ_i)}\le C.
\]
Moreover, by the mean-value theorem,
\[
\bigl|e^{\chi_\eta(t,\bX_i,\bZ_i)}-e^{\chi_\eta(s,\bX_i,\bZ_i)}\bigr|
\le
C\,|g(t,\bX_i)-g(s,\bX_i)|
\le
C\,\omega_n(\delta_n)
\]
whenever \(|t-s|\le \delta_n\). Therefore, for each interval,
\[
\left|
\int_{t_{j-1}}^{\min(T_i,t_j)} e^{\chi_\eta(t,\bX_i,\bZ_i)}\,dt
-
\bigl\{\min(T_i,t_j)-t_{j-1}\bigr\}e^{\chi_\eta(t_j,\bX_i,\bZ_i)}
\right|
\le
(t_j-t_{j-1})\,C\,\omega_n(\delta_n).
\]
Summing over \(j\) and using \(\sum_j (t_j-t_{j-1})=\tau\), we obtain
\[
|R_i(\eta)|\le C\tau\,\omega_n(\delta_n)
\]
uniformly in \(i\) and \(\eta\in \mathcal{N}_n\). Hence
\[
\sup_{\eta\in \mathcal{N}_n}|\tilde\ell_n(\eta)-\ell_n(\eta)|
\le
\frac1n\sum_{i=1}^n \sup_{\eta\in \mathcal{N}_n}|R_i(\eta)|
\le
C_1\,\omega_n(\delta_n)
\]
for some constant \(C_1>0\).

Next consider the \(\btheta\)-scores. Since \(g\) does not depend on \(\theta\),
\[
\nabla_\theta \ell_n(\eta)
=
\frac1n\sum_{i=1}^n
\left[
\Delta_i \bZ_i
-
\int_0^\tau \bZ_i Y_i(t)e^{\chi_\eta(t,\bX_i,\bZ_i)}\,dt
\right],
\]
and
\[
\nabla_\theta \tilde\ell_n(\eta)
=
\frac1n\sum_{i=1}^n
\left[
\Delta_i \bZ_i
-
\sum_{j:Y_i(t_{j-1})=1}
\bigl\{\min(T_i,t_j)-t_{j-1}\bigr\} \bZ_i e^{\chi_\eta(t_j,\bX_i,\bZ_i)}
\right].
\]
Therefore,
\[
\nabla_\theta \tilde\ell_n(\eta)-\nabla_\theta \ell_n(\eta)
=
-\frac1n\sum_{i=1}^n \bZ_i R_i(\eta).
\]
Using \(\|\bZ_i\|\le B_Z\) a.s. from \ref{ass:A2} together with the bound on \(R_i(\eta)\),
\[
\sup_{\eta\in \mathcal{N}_n}
\bigl\|\nabla_\theta \tilde\ell_n(\eta)-\nabla_\theta \ell_n(\eta)\bigr\|
\le
\frac1n\sum_{i=1}^n \|\bZ_i\|\,\sup_{\eta\in \mathcal{N}_n}|R_i(\eta)|
\le
B_{\bZ} C_1\,\omega_n(\delta_n).
\]
This proves the second bound with \(C_2=B_{\bZ} C_1\).
\end{proof}

\begin{lemma}\label{lem:info-hessian}
Let
$S_n(\eta)=\nabla_{\theta}\ell_n(\eta)$ and
$\dot S_n(\eta)=\nabla_{\theta}S_n(\eta)$ be the Hessian of the empirical
log--likelihood with respect to $\btheta$.
For a fixed $M>0$, define the shrinking neighbourhood
$\mathcal N_n=\{\eta:d(\eta,\eta_o)\le M\tau_n\}$,
where $\tau_n=\gamma_n\log^2 n$.
Assume \ref{ass:A1}--\ref{ass:A4}. Then,
$$
\sup_{\eta\in\mathcal N_n}
\bigl\|
  \dot S_n(\eta)+ I_0(\btheta_o)
\bigr\|
\;\xrightarrow{P}\;0,
$$
where
$$
I_0(\btheta_o)
=
\EE_0\!\left[
\int_0^\tau \bZ\bZ^\top
Y(t)\exp\{\btheta_o^\top\bZ+g_o(t,\bX)\}\,dt
\right]
=
\EE_0\!\left[\Delta\,\bZ\bZ^\top\right].
$$
If, in addition, $I_0(\btheta_o)$ is nonsingular, then
$$
\sup_{\eta\in\mathcal N_n}\bigl\|\dot S_n(\eta)^{-1}\bigr\| = O_p(1).
$$
\end{lemma}

\begin{proof}[Proof of \Cref{lem:info-hessian}]
For a single observation $(T,\Delta,\bX,\bZ)$, the per--observation
log--likelihood is
$$
\ell(\btheta,g)
=
\Delta\{\btheta^\top\bZ + g(T,\bX)\}
-
\int_0^\tau Y(t)\exp\{\btheta^\top\bZ + g(t,\bX)\}\,dt \, .
$$
The corresponding $\btheta$-score is
$$
s(\btheta,g)
=
\nabla_{\theta}\ell(\btheta,g)
=
\Delta\bZ
-
\int_0^\tau \bZ\,Y(t)\exp\{\btheta^\top\bZ + g(t,\bX)\}\,dt \, ,
$$
and the $\btheta$-Hessian is
$$
\dot s(\btheta,g)
=
\nabla_{\theta}s(\btheta,g)
=
-
\int_0^\tau \bZ\bZ^\top
Y(t)\exp\{\btheta^\top\bZ + g(t,\bX)\}\,dt \, .
$$
Therefore,
$$
\dot S_n(\eta)=\PP_n\dot s(\btheta,g).
$$

Define the matrix-valued per-observation integrand
$$
J(\eta;\bW)
=
\int_0^\tau \bZ\bZ^\top
Y(t)\exp\{\btheta^\top\bZ + g(t,\bX)\}\,dt,
\qquad
\bW=(T,\Delta,\bX,\bZ),
$$
so that
$$
\dot S_n(\eta)=-\,\PP_n J(\eta;\cdot),
\qquad
\PP \dot s(\btheta,g)=-\,\PP J(\eta;\cdot).
$$

By Assumptions~\ref{ass:A2}--\ref{ass:A3}, the covariates are bounded and the sieve is
uniformly bounded, so there exists a constant $C<\infty$ such that
$$
\sup_{\eta\in\mathcal N_n}\|J(\eta;\bW)\|\le C
\qquad\text{a.s.}
$$
Since $p$ is fixed, it suffices to control each matrix entry separately. The class
$$
\mathcal J_n=\{J(\eta;\cdot):\eta\in\mathcal N_n\}
$$
inherits the same bracketing entropy bound as the localized likelihood/score classes used
in Lemma~\ref{lem:unified_maximal}. Hence it is $\PP$--Glivenko--Cantelli, and
\begin{equation}\label{eq:Jn-ULLN-corrected}
\sup_{\eta\in\mathcal N_n}
\bigl\|
  \PP_n J(\eta;\cdot)-\PP J(\eta;\cdot)
\bigr\|
\;\xrightarrow{P}\;0.
\end{equation}

Next we show continuity of $\eta\mapsto \PP J(\eta;\cdot)$ in the metric $d$.
Let $\eta_1=(\btheta_1,g_1)$ and $\eta_2=(\btheta_2,g_2)$, and write
$\phi(t,\bX,\bZ)=\chi_{\eta_1}(t,\bX,\bZ)-\chi_{\eta_2}(t,\bX,\bZ)$.
Then, by boundedness of $\bZ$ and the mean-value theorem for the exponential,
there exists a constant $C_1>0$ such that
\begin{align*}
\bigl\|J(\eta_1;\bW)-J(\eta_2;\bW)\bigr\|
&\le
C_1\int_0^\tau Y(t)\,|\phi(t,\bX,\bZ)|\,dt .
\end{align*}
Taking expectations and applying Cauchy--Schwarz,
\begin{align*}
\bigl\|\PP J(\eta_1;\cdot)-\PP J(\eta_2;\cdot)\bigr\|
&\le
C_1
\EE_0\!\left[\int_0^\tau Y(t)\,|\phi(t,\bX,\bZ)|\,dt\right]
\\
&\le
C_1\tau^{1/2}
\left\{
\EE_0\!\left[\int_0^\tau Y(t)\phi(t,\bX,\bZ)^2\,dt\right]
\right\}^{1/2}.
\end{align*}
Under Assumptions~\ref{ass:A1}--\ref{ass:A4}, the true hazard
$h_o(t\mid \bX,\bZ)=\exp\{\btheta_o^\top\bZ+g_o(t,\bX)\}$ is bounded away from zero,
say $h_o\ge c_h>0$. Therefore,
\begin{align*}
\EE_0\!\left[\int_0^\tau Y(t)\phi(t,\bX,\bZ)^2\,dt\right]
&\le
c_h^{-1}
\EE_0\!\left[\int_0^\tau Y(t)h_o(t\mid\bX,\bZ)\phi(t,\bX,\bZ)^2\,dt\right]
\\
&=
c_h^{-1}\EE_0\!\left[\Delta\,\phi(T,\bX,\bZ)^2\right]
=
c_h^{-1} d(\eta_1,\eta_2)^2,
\end{align*}
where the equality follows from the compensator identity.
Hence
\[
\bigl\|\PP J(\eta_1;\cdot)-\PP J(\eta_2;\cdot)\bigr\|
\le
C_2\, d(\eta_1,\eta_2)
\]
for some constant $C_2>0$. In particular,
\begin{equation}\label{eq:PJ-cont-corrected}
\sup_{\eta\in\mathcal N_n}
\bigl\|
  \PP J(\eta;\cdot)-\PP J(\eta_o;\cdot)
\bigr\|
\le
C_2 M\tau_n
\longrightarrow 0.
\end{equation}

At the true parameter,
\begin{align*}
\PP J(\eta_o;\cdot)
&=
\EE_0\!\left[
\int_0^\tau \bZ\bZ^\top Y(t)
\exp\{\btheta_o^\top\bZ+g_o(t,\bX)\}\,dt
\right]
\\
&=
\EE_0\!\left[\Delta\,\bZ\bZ^\top\right]
=
I_0(\btheta_o),
\end{align*}
again by the compensator identity. Therefore, using
\eqref{eq:Jn-ULLN-corrected} and \eqref{eq:PJ-cont-corrected},
\begin{align*}
\sup_{\eta\in\mathcal N_n}
\bigl\|
  \dot S_n(\eta)+I_0(\btheta_o)
\bigr\|
&=
\sup_{\eta\in\mathcal N_n}
\bigl\|
  -\PP_n J(\eta;\cdot)+I_0(\btheta_o)
\bigr\|
\\
&\le
\sup_{\eta\in\mathcal N_n}
\bigl\|
  \PP_n J(\eta;\cdot)-\PP J(\eta;\cdot)
\bigr\|
+
\sup_{\eta\in\mathcal N_n}
\bigl\|
  \PP J(\eta;\cdot)-\PP J(\eta_o;\cdot)
\bigr\|
\\
&\xrightarrow{P}0.
\end{align*}

If $I_0(\btheta_o)$ is nonsingular, let
$\lambda_0=\lambda_{\min}(I_0(\btheta_o))>0$.
On the event
\[
\sup_{\eta\in\mathcal N_n}
\bigl\|
  \dot S_n(\eta)+I_0(\btheta_o)
\bigr\|
<\lambda_0/2,
\]
all matrices $-\dot S_n(\eta)$ are nonsingular and
\[
\sup_{\eta\in\mathcal N_n}
\bigl\|
  \dot S_n(\eta)^{-1}
\bigr\|
\le
2/\lambda_0.
\]
Since this event has probability tending to one, we conclude
\[
\sup_{\eta\in\mathcal N_n}
\bigl\|
  \dot S_n(\eta)^{-1}
\bigr\|
=
O_p(1).
\]
This completes the proof.
\end{proof}

\ThetaAsmp*

\begin{proof}[Proof of \Cref{thm:quad-equiv}]
Recall that $\ell_n(\eta)$ and $\widetilde\ell_n(\eta)$ denote the
exact and numerically integrated log-likelihoods, respectively.
For any fixed $M>0$, applying Lemma~11 with $\mathcal N_n=\mathcal N_n(M)$ and $\omega_n=\omega_{n,M}$, gives
\[
\sup_{\eta\in\mathcal N_n(M)}
|\widetilde\ell_n(\eta)-\ell_n(\eta)|
\le C_1\omega_{n,M}(\delta_n).
\]
Since $\eta_o\in\mathcal N_n(M)$, subtracting the values at
$\eta_o$ gives
\[
\sup_{\eta\in\mathcal N_n(M)}
\left|
\{\widetilde\ell_n(\eta)-\ell_n(\eta)\}
-
\{\widetilde\ell_n(\eta_o)-\ell_n(\eta_o)\}
\right|
=
O_p\{\omega_{n,M}(\delta_n)\}.
\tag{B.20}
\]

For every sufficiently large fixed $M$, the assumption
$\omega_{n,M}(\delta_n)=o(\tau_n^2)$ implies that the
discretization error is negligible relative to the local
quadratic curvature of the population criterion. Together with
the empirical-process bound used in the proof of Theorem~1 and
Lemma~6,
\[
\ell_0(\eta)-\ell_0(\eta_o)\asymp-d(\eta,\eta_o)^2,
\]
the same shell argument used in the proof of Theorem~1
(Theorem~3.4.1 of van der Vaart and Wellner, 1996) applies to
$\widetilde\ell_n$. In particular, on shells
$d(\eta,\eta_o)\asymp 2^j\tau_n$, the negative population
curvature dominates both the empirical-process fluctuation and
the numerical-integration error. Hence,
\[
d(\widetilde\eta,\eta_o)=O_p(\tau_n).
\]
Lemma~10 then gives
\[
\|\widetilde\btheta-\btheta_o\|=O_p(\tau_n).
\]

It remains to establish the first-order expansion of
$\widetilde\btheta$. Let $\widetilde r_{\theta,g}$ denote the
discretized analogue of $r_{\theta,g}$. By the numerical part of
\ref{ass:Anew},
\[
\PP_n\widetilde r_{\widetilde\theta,\widetilde g}
=o_p(n^{-1/2}).
\]

The score-discretization bound gives, uniformly on
$\mathcal N_n(M)$,
\[
\sup_{\eta\in\mathcal N_n(M)}
\|\PP_n\widetilde r_\eta-\PP_n r_\eta\|
=o_p(n^{-1/2}),
\]
under
$\sqrt n\,\omega_{n,M}(\delta_n)\to0$,
the uniform Lipschitz condition on $\bg^*$, and
$\sqrt n\,\delta_n\to0$.
Since
$d(\widetilde\eta,\eta_o)=O_p(\tau_n)$, it follows that
\[
\PP_n r_{\widetilde\theta,\widetilde g}
=o_p(n^{-1/2}).
\]

By the second part of \ref{ass:Anew},
\[
\PP\{r_{\widetilde\theta,\widetilde g}
-r_{\theta_o,g_o}\}
=
-I(\btheta_o)(\widetilde\btheta-\btheta_o)
+O_p(\tau_n^2)
=
-I(\btheta_o)(\widetilde\btheta-\btheta_o)
+o_p(n^{-1/2}).
\]
The same localized stochastic-equicontinuity argument as in the
proof of Theorem~\ref{thm:asym-theta} gives
\[
(\PP_n-\PP)
\{r_{\widetilde\theta,\widetilde g}-r_{\theta_o,g_o}\}
=o_p(n^{-1/2}).
\]
Consequently,
\[
I(\btheta_o)(\widetilde\btheta-\btheta_o)
=
\PP_n r_{\theta_o,g_o}
+o_p(n^{-1/2}).
\]
By Theorem~\ref{thm:eff-score},
$r_{\btheta_o,g_o}=\ell_{\btheta_o}^*$, and therefore
\[
\sqrt n(\widetilde\btheta-\btheta_o)
=
I(\btheta_o)^{-1}
\frac1{\sqrt n}\sum_{i=1}^n
\ell_{\theta_o}^*(T_i,\bX_i,\bZ_i,\Delta_i)
+o_p(1).
\]
The asserted asymptotic normality follows from the multivariate
central limit theorem and Slutsky's theorem.
\end{proof}


\section{Proofs of Section 5}\label{appB}

\Hderivative*

\begin{proof}[Proof of \Cref{lem:pathwise}]
	Along the joint submodel,
	\[
	h_\varepsilon(u\mid \bX_0,\bZ_0)
	=
	h_o(u\mid \bX_0,\bZ_0)
	\exp\!\Big\{
	\varepsilon\,\ba^\top(\bZ_0-\bg^*(u,\bX_0))
	+\varepsilon\,b(u,\bX_0)
	\Big\} \, .
	\]
	Differentiating at $\varepsilon=0$ gives
	\[
	\left.
	\frac{d}{d\varepsilon}
	h_\varepsilon(u\mid \bX_0,\bZ_0)
	\right|_{\varepsilon=0}
	=
	h_o(u\mid \bX_0,\bZ_0)
	\Big\{
	\ba^\top(\bZ_0-\bg^*(u,\bX_0))
	+
	b(u,\bX_0)
	\Big\} \, .
	\]
	Integration over $[0,t]$ yields the stated result.
\end{proof}

\EIFms*

\begin{proof}[Proof of \Cref{lem:EIF_existence_MS}]
	By \ref{ass:A7}, $Q: \overline{\mathcal{T}}_{g_o} \to L^2(\PP_{\bX,\bZ})$ is a bounded linear operator. Next, by  \ref{ass:A7}, there exists a square-integrable random function $\psi_t^0(\cdot,\cdot;\bX_0,\bZ_0)$
	such that for all $b\in \mathcal \overline{\mathcal{T}}_{g_o}$,
	\begin{equation}\label{eq:riesz-repr}
	(Qb)(\bX_0,\bZ_0)=\langle \psi_t^0(\cdot,\cdot;\bX_0,\bZ_0),\, b\rangle_0
	\quad\text{in }L^2(\PP_{\bX,\bZ}) \, .    
	\end{equation}	
	For each $b\in \overline{\mathcal{T}}_{g_o}$, the martingale isometry associated with
	\eqref{eq:inner_product} yields
	\[
	\langle \psi_t^0(\cdot,\cdot;\bX_0,\bZ_0),\,b\rangle_0
	=
	\EE_0\!\left[
	\left(\int_0^\tau \psi_t^0(u,\bX;\bX_0,\bZ_0)\,dM(u)\right)
	\left(\int_0^\tau b(u,\bX)\,dM(u)\right)
	\ \middle|\ \bX_0,\bZ_0
	\right],
	\]
	and since $\langle \psi_t^0,b\rangle_0=L_t(b)$ in $L^2(\PP_{\bX,\bZ})$, the nuisance contribution to the IF
	in the $L^2(\PP_{\bX,\bZ})$ sense is $\int_0^\tau \psi_t^0(u,\bX;\bX_0,\bZ_0)\,dM(u)$.
	
	The parametric derivative term in (\ref{eq:Gateaux_deriv_MS2}) equals $\ba^\top c_{0,t}(\bX_0,\bZ_0)$.
	Since $\ell_{\theta_o}^*$ is the efficient score and
	$I(\btheta_o)=\EE_0\{\ell_{\theta_o}^*(\bW)^{\otimes 2}\}$,
	the corresponding efficient influence function for $\btheta_o$ is $I(\btheta_o)^{-1}\ell_{\theta_o}^*(\bW)$.
	Thus the parametric contribution to the IF for $H_o(t\mid\bX_0,\bZ_0)$ is
	$$c_{0,t}(\bX_0,\bZ_0)^\top I(\btheta_o)^{-1}\ell_{\theta_o}^*(\bW) \, . $$
	Combining nuisance and parametric pieces gives \eqref{eq:EIF_MS}. Finally, using \eqref{eq:Gateaux_deriv_MS2}, (C.1), and the score decomposition \eqref{eq:score_decomp},
	one checks that
	\[
	\left.\frac{d}{d\varepsilon}H_\varepsilon(t\mid \bX_0,\bZ_0)\right|_{\varepsilon=0}
	-
	\EE_0\!\big[\phi_t^0(\bW;\bX_0,\bZ_0)\,S(\bW;\ba,b)\mid\bX_0,\bZ_0\big]
	=0
	\quad\text{in }L^2(\PP_{\bX,\bZ}),
	\]
	which proves pathwise differentiability of  $H_{\PP}(t\mid\bX_0,\bZ_0)$ as an $L^2(\PP_{\bX,\bZ})$-valued functional.
\end{proof}

\Discrete*

\begin{proof}[Proof of \Cref{prop:closed_form_discrete}]
	Since $\bX$ takes values in the finite set $\{\bx^{(1)},\ldots,\bx^{(J)}\}$, every $b\in \overline{\mathcal{T}}_{g_o}$
	admits the representation
	\[
	b(u,\bX)=\sum_{j=1}^J b_j(u)\,\mathbf 1\{\bX=\bx^{(j)}\},
	\qquad b_j\in L_2([0,\tau]).
	\]
    Although the functions $b_j$ need not vary
independently within the restricted nuisance
tangent space, the following calculation establishes
the Riesz identity for arbitrary square-integrable
functions of $(u,\bX)$, and hence, in particular,
for every $b\in\overline{\mathcal T}_{g_o}$.
	Let $j_0$ be such that $\bX_0=\bx^{(j_0)}$. Using the inner product \eqref{eq:inner_product},
	for any square-integrable $\psi$ for which the expression below is finite, we obtain
	\begin{align*}
		\langle \psi, b\rangle_0
		&=
		\EE_0\!\left[\int_0^\tau \psi(u,\bX;\bX_0,\bZ_0)\,b(u,\bX)\,Y(u)h_o(u\mid\bX,\bZ)\,du\right] \\
		&=
		\sum_{j=1}^J \EE_0\!\left[\mathbf 1\{\bX=\bx^{(j)}\}\int_0^\tau
		\psi(u,\bx^{(j)};\bX_0,\bZ_0)\,b_j(u)\,Y(u)h_o(u\mid\bX,\bZ)\,du\right] \\
		&=
		\sum_{j=1}^J \PP_0(\bX=\bx^{(j)})\int_0^\tau \psi(u,\bx^{(j)};\bX_0,\bZ_0)\,b_j(u)\,
		\EE_0\!\left[Y(u)h_o(u\mid\bX,\bZ)\mid \bX=\bx^{(j)}\right]du \\
		&=
		\sum_{j=1}^J \PP_0(\bX=\bx^{(j)})\int_0^\tau \psi(u,\bx^{(j)};\bX_0,\bZ_0)\,b_j(u)\,A_j(u)\,du .
	\end{align*}
	
	On the other hand, by definition of the evaluation functional $L_t$ (as in (5.6)),
	\[
	L_t(b)=\int_0^t h_o(u\mid\bX_0,\bZ_0)\,b(u,\bX_0)\,du
	= \int_0^t h_o(u\mid\bX_0,\bZ_0)\,b_{j_0}(u)\,du .
	\]

Now take $\psi_t^0$ to be the function in \eqref{eq:psi_closed}, with
value zero outside $\{u\le t,\bX=\bX_0\}$.
Substituting this function into the preceding expression
for the inner product gives
\[
\begin{aligned}
\langle\psi_t^0,b\rangle_0
&=
\P_0(\bX=\bX_0)
\int_0^t
\frac{h_o(u\mid\bX_0,\bZ_0)}
     {\P_0(\bX=\bX_0)A_{\bX_0}(u)}
\,b_{j_0}(u)A_{\bX_0}(u)\,du
\\
&=
\int_0^t
h_o(u\mid\bX_0,\bZ_0)b_{j_0}(u)\,du
=
L_t(b).
\end{aligned}
\]
This calculation holds for every square-integrable function
$b(u,\bX)$ for which the integrals are finite, and hence,
in particular, for every
$b\in\overline{\mathcal T}_{g_o}$.
This proves the stated Riesz identity.    
\end{proof}

The following lemma shows that in empirical sums, the oracle efficient score $\ell_{\theta_o}^*$ can be replaced by the cross-fitted plug-in score $\hat\ell_{\theta_o}^{*,(-k)}$ without changing the $n^{1/2}$-limit, provided the nuisance estimates converge at a suitable rate and cross-fitting ensures conditional independence. This is the key technical step that lets us use estimated nuisances inside the IF expansion. 

\begin{restatable}{lemma}{EmpiricalLike}
	\label{lem:plugin_ellstar}
	For $i\in \mathcal I_k$ define the plug-in martingale 
	\[
	\widehat M^{(-k)}_i(u)
	=
	N_i(u)-\int_0^u Y_i(s)\widehat h^{(-k)}(s\mid \bX_i,\bZ_i)\,ds \, ,
	\]
	and the plug-in efficient score
	\begin{equation}\label{eq:ellstar_plugin}
		\widehat\ell^{*,(-k)}_{\theta_o}(\bW_i)
		=
		\int_0^\tau \{\bZ_i-\widehat {\bg}^{*,(-k)}(u,\bX_i)\}\,d\widehat M^{(-k)}_i(u) \, .
	\end{equation}
	Assume \ref{ass:A1}--\ref{ass:A7} and \ref{ass:A9}. Then
	\begin{equation}\label{eq:plugin_ellstar_rootn}
\frac{1}{\sqrt n}
\sum_{k=1}^{K_{\mathrm{cf}}}\sum_{i\in I_k}
\left[
\left\{
\widehat\ell_{\theta_o}^{*,(-k)}(\bW_i)
-\ell_{\theta_o}^*(\bW_i)
\right\}
-
\EE_0\!\left\{
\widehat\ell_{\theta_o}^{*,(-k)}(\bW)
-\ell_{\theta_o}^*(\bW)
\mid\mathcal F_{-k}
\right\}
\right]
=o_p(1).
	\end{equation}
\end{restatable}


\begin{proof}[Proof of \Cref{lem:plugin_ellstar}]
	Fix $k\in\{1,\dots,K_{\mathrm{cf}}\}$ and write $\mathcal F_{-k}$ for the sigma-field generated by the
	training sample $\{W_j:j\in \mathcal I_k^c\}$.  By Assumption \ref{ass:A9}, conditional on $\mathcal F_{-k}$
	the functions $\widetilde g^{(-k)}$, $\widetilde \btheta^{(-k)}$ and $\widehat \bg^{*,(-k)}$ are fixed,
	and the observations $\{\bW_i:i\in \mathcal I_k\}$ are i.i.d.\ from $\PP_0$.
	
		For $i\in \mathcal I_k$ set
	\[
	R_{i,k}
	=
	\widehat\ell^{*,(-k)}_{\theta_o}(\bW_i)-\ell^*_{\theta_o}(\bW_i).
	\]
    All score differences and their remainder terms below are
$p$-vectors. We use the Euclidean norm $\|\cdot\|_2$ and,
for a vector-valued function $v$, write
\[
\|v\|_{L^2(\P_0,\Delta)}^2
=
\EE_0\!\left[\Delta\|v(T,\bX,\bZ)\|_2^2\right],
\]
omitting arguments on which $v$ does not depend.
When $v$ depends on the training sample, this expectation
is over an independent observation, holding the training
sample fixed.
	Insert and subtract the intermediate integral
	$\int_0^\tau\{\bZ_i-\widehat \bg^{*,(-k)}(u,\bX_i)\}\,dM_{i}(u)$ to obtain
	\[
	R_{i,k}
	=
	\int_0^\tau \big\{ \bg^*(u,\bX_i) - \widehat \bg^{*,(-k)}(u,\bX_i)\big\}\,dM_{i}(u)
	+
	\int_0^\tau \{\bZ_i-\widehat \bg^{*,(-k)}(u,\bX_i)\}\,d\!\big(\widehat M^{(-k)}_i-M_{i}\big)(u).
	\]

Denote the first and second terms by $R_{i,k}^{(1)}$
and $R_{i,k}^{(2)}$, respectively.
Conditional on $\mathcal F_{-k}$, the vectors
$\{R_{i,k}^{(1)}:i\in I_k\}$ are independent and centered.
Applying the martingale isometry componentwise and summing
over the $p$ coordinates gives
\[
\begin{aligned}
\EE_0\!\left[
\|R_{i,k}^{(1)}\|_2^2
\mid\mathcal F_{-k}
\right]
&=
\EE_0\!\left[
\int_0^\tau
\left\|
\widehat{\bg}^{*,(-k)}(u,\bX)-\bg^*(u,\bX)
\right\|_2^2
Y(u)h_o(u\mid\bX,\bZ)\,du
\;\middle|\;\mathcal F_{-k}
\right]
\\
&=
\left\|
\widehat{\bg}^{*,(-k)}-\bg^*
\right\|_{L^2(\P_0,\Delta)}^2
=
O_p(\tau_n^2),
\end{aligned}
\]
uniformly in $k$, by \ref{ass:A9}.
Conditional independence and zero conditional means imply
\[
\begin{aligned}
\EE_0\!\left[
\left\|
\frac{1}{\sqrt n}\sum_{i\in I_k}R_{i,k}^{(1)}
\right\|_2^2
\;\middle|\;\mathcal F_{-k}
\right]
&=
\frac{1}{n}\sum_{i\in I_k}
\EE_0\!\left[
\|R_{i,k}^{(1)}\|_2^2
\mid\mathcal F_{-k}
\right]
\\
&=O_p(\tau_n^2).
\end{aligned}
\]
Therefore, by conditional Markov's inequality,
\[
\left\|
\frac{1}{\sqrt n}\sum_{i\in I_k}R_{i,k}^{(1)}
\right\|_2
=
O_p(\tau_n)=o_p(1).
\]
Next consider $R^{(2)}_{i,k}$.  Since
	\[
	\widehat M^{(-k)}_i(u)-M_{i}(u)
	=
	-\int_0^u Y_i(s)\Big\{\widehat h^{(-k)}(s\mid \bX_i,\bZ_i)-h_o(s\mid \bX_i,\bZ_i)\Big\}\,ds,
	\]
	we have
	\begin{equation}\label{eq:R2_predictable}
		R^{(2)}_{i,k}
		=
		-\int_0^\tau
		\{\bZ_i-\widehat \bg^{*,(-k)}(u,\bX_i)\}\,
		Y_i(u)\Big\{\widehat h^{(-k)}(u\mid \bX_i,\bZ_i)-h_o(u\mid \bX_i,\bZ_i)\Big\}\,du.
	\end{equation}
	Write
	\begin{equation}
\begin{aligned}
\widehat h^{(-k)}(u\mid \bX,\bZ)-h_o(u\mid \bX,\bZ)
&= h_o(u\mid \bX,\bZ) \\
&\quad \times
\Big[
\exp\!\big\{
\widetilde g^{(-k)}(u,\bX)-g_o(u,\bX)
+\bZ^\top(\widetilde\btheta^{(-k)}-\btheta_o)
\big\}
-1
\Big].
\end{aligned}
\end{equation}
	By the mean value theorem, for each $(u,\bX,\bZ)$ there exists a random $\xi^{(-k)}(u,\bX,\bZ)$
	between $0$ and $\delta^{(-k)}(u,\bX,\bZ)=\widetilde g^{(-k)}(u,\bX)-g_o(u,\bX)+\bZ^\top(\widetilde\btheta^{(-k)}-\btheta_o)$ such that
	\[
	\widehat h^{(-k)}(u\mid \bX,\bZ)-h_o(u\mid \bX,\bZ)
	=
	h_o(u\mid \bX,\bZ)\exp\{\xi^{(-k)}(u,\bX,\bZ)\}
	\delta^{(-k)}(u,\bX,\bZ).
	\]
	By the bounded-sieve condition for $\widetilde g^{(-k)}$, boundedness of $g_o$, boundedness of $\bZ$,
	and boundedness of the parameter space for $\btheta$, there exists a deterministic constant $C<\infty$
	such that
	\[
	\sup_{u\le \tau,\,\bX,\,\bZ}\left|
	\delta^{(-k)}(u,\bX,\bZ)
	\right|\le C
	\]
	with probability tending to one. Therefore, uniformly on this event, the exponential map admits
	the second-order expansion
	\[
	e^x = 1 + x + R_C(x), \qquad |R_C(x)|\le C_1 x^2,\quad |x|\le C,
	\]
	for some finite constant $C_1$. Hence
	\[
	\widehat h^{(-k)}(u\mid \bX,\bZ)-h_o(u\mid \bX,\bZ)
	=
	h_o(u\mid \bX,\bZ)\,\delta^{(-k)}(u,\bX,\bZ)
	+
	h_o(u\mid \bX,\bZ)\,R^{(-k)}(u,\bX,\bZ),
	\]
	where $R^{(-k)}(u,\bX,\bZ)=e^{\delta^{(-k)}(u,\bX,\bZ)}-1-\delta^{(-k)}(u,\bX,\bZ)$ and
	\[
	|R^{(-k)}(u,\bX,\bZ)|\le C_1\bigl(\delta^{(-k)}(u,\bX,\bZ)\bigr)^2.
	\]
	 Substituting into \eqref{eq:R2_predictable} gives
	\begin{align*}
		R^{(2)}_{i,k}
		&=
		-\int_0^\tau
		\{\bZ_i-\widehat \bg^{*,(-k)}(u,\bX_i)\}\,
		Y_i(u)h_o(u\mid \bX_i,\bZ_i)
	\delta^{(-k)}(u,\bX_i,\bZ_i),du 
		+\; \widetilde R_{i,k},
	\end{align*}

To control the vector remainder, write
\[
e_k=\widehat{\bg}^{*,(-k)}-\bg^*,
\qquad
a_k=\|\delta^{(-k)}\|_{L^2(\PP_0,\Delta)},
\qquad
b_k=\|e_k\|_{L^2(\PP_0,\Delta)}.
\]
By A2 and A10,
\[
a_k=O_p(\tau_n+n^{-1/2}),
\qquad
b_k=O_p(\tau_n),
\]
uniformly in $k$. The exponential remainder bound gives
\[
\|\widetilde R_{i,k}\|_2
\lesssim
\int_0^\tau
\left\|
\bZ_i-\widehat{\bg}^{*,(-k)}(u,\bX_i)
\right\|_2
Y_i(u)h_o(u\mid\bX_i,\bZ_i)
|\delta^{(-k)}(u,\bX_i,\bZ_i)|^2\,du.
\]
Since $\delta^{(-k)}$ is uniformly bounded with probability
tending to one, $|\delta^{(-k)}|^4\lesssim
|\delta^{(-k)}|^2$ on this event. Using
$\bZ-\widehat{\bg}^{*,(-k)}
=(\bZ-\bg^*)-e_k$, the boundedness of $\bZ-\bg^*$,
Cauchy--Schwarz, Jensen's inequality, and the compensator
identity, we obtain
\begin{equation}
\label{eq:remainder-L2}
\EE_0\!\left[
\|\widetilde R_{i,k}\|_2^2
\mid\mathcal F_{-k}
\right]
\lesssim
a_k^2+b_k^2
=
O_p(\tau_n^2+n^{-1}),
\end{equation}
uniformly in $k$.
Now consider the leading linear term
\[
L_{i,k}
=
R_{i,k}^{(2)}-\widetilde R_{i,k},
\qquad
\mu_k
=
\EE_0[L_{i,k}\mid\mathcal F_{-k}].
\]

By Cauchy--Schwarz, the boundedness conditions, and \ref{ass:A9},
\[
\EE_0\!\left[
\|L_{i,k}\|_2^2
\mid\mathcal F_{-k}
\right]
\lesssim a_k^2+b_k^2
=
O_p(\tau_n^2+n^{-1}),
\]
uniformly in $k$.
Conditional on $\mathcal F_{-k}$, the vectors
$\{L_{i,k}-\mu_k:i\in I_k\}$ are independent and centered.
Consequently,
\[
\begin{aligned}
&\EE_0\!\left[
\left\|
\frac{1}{\sqrt n}
\sum_{i\in I_k}(L_{i,k}-\mu_k)
\right\|_2^2
\;\middle|\;\mathcal F_{-k}
\right]
\\
&\qquad=
\frac{1}{n}\sum_{i\in I_k}
\EE_0\!\left[
\|L_{i,k}-\mu_k\|_2^2
\mid\mathcal F_{-k}
\right]
\\
&\qquad\le
\frac{1}{n}\sum_{i\in I_k}
\EE_0\!\left[
\|L_{i,k}\|_2^2
\mid\mathcal F_{-k}
\right]
=
O_p(\tau_n^2+n^{-1}).
\end{aligned}
\]
Hence
\[
\frac{1}{\sqrt n}\sum_{i\in I_k}L_{i,k}
=
\frac{|I_k|}{\sqrt n}\mu_k
+
O_p(\tau_n+n^{-1/2}).
\]    

For the centered remainder contribution, conditional on
$\mathcal F_{-k}$ the vectors
\[
\widetilde R_{i,k}
-
\EE_0(\widetilde R_{i,k}\mid\mathcal F_{-k}),
\qquad i\in I_k,
\]
are independent and centered. Hence, by (C.6),
\[
\begin{aligned}
&\EE_0\!\left[
\left\|
\frac{1}{\sqrt n}\sum_{i\in I_k}
\left\{
\widetilde R_{i,k}
-\EE_0(\widetilde R_{i,k}\mid\mathcal F_{-k})
\right\}
\right\|_2^2
\;\middle|\;\mathcal F_{-k}
\right]
\le
\frac{1}{n}\sum_{i\in I_k}
\EE_0\!\left[
\|\widetilde R_{i,k}\|_2^2
\mid\mathcal F_{-k}
\right]
=
O_p(\tau_n^2+n^{-1}).
\end{aligned}
\]
Therefore,
\[
\left\|
\frac{1}{\sqrt n}\sum_{i\in I_k}
\left\{
\widetilde R_{i,k}
-\EE_0(\widetilde R_{i,k}\mid\mathcal F_{-k})
\right\}
\right\|_2
=o_p(1),
\]
uniformly in $k$.
Moreover, the preceding conditional second-moment bounds for
$R_{i,k}^{(1)}$, $L_{i,k}$, and $\widetilde R_{i,k}$ imply
\[
\EE_0\!\left[
\|R_{i,k}\|_2^2
\mid\mathcal F_{-k}
\right]
=
O_p(\tau_n^2+n^{-1}),
\]
uniformly in $k$.
Combining the bounds for $R^{(1)}_{i,k}$ and $R^{(2)}_{i,k}$ yields
\[
\frac{1}{\sqrt n}
\sum_{k=1}^{K_{\mathrm{cf}}}\sum_{i\in I_k}
\left\{
R_{i,k}
-
\EE_0[R_{i,k}\mid\mathcal F_{-k}]
\right\}
=o_p(1) \, ,
\]
and this proves \eqref{eq:plugin_ellstar_rootn} and completes the proof.
\end{proof}

The following lemma proves the out-of-fold estimator $\widehat I^{(-k)}$ of the efficient information converges to $I(\btheta_o)$, and that the inverse $\{\widehat I^{(-k)}\}^{-1}$ is well-behaved when $I(\btheta_o)$ is nonsingular. This justifies replacing $I(\btheta_o)^{-1}$ in the  parametric component of the IF by a cross-fitted plug-in inverse.  

\begin{restatable}{lemma}{InverseInfo}
	\label{lem:inverse-info-plug-in}
	Define the out-of-fold (efficient) information estimator
	\[
	\widehat I^{(-k)}
	=
	\frac{1}{|\mathcal I_k^c|}
	\sum_{j\in \mathcal I_k^c}
	\Delta_j\Big\{\bZ_j-\widehat \bg^{*,(-k)}(T_j,\bX_j)\Big\}^{\otimes 2},
	\qquad k=1,\dots,K_{\mathrm{cf}}.
	\]
	Assume \ref{ass:A1}--\ref{ass:A7} and \ref{ass:A9}--\ref{ass:A10}.
	Then,
	\begin{enumerate}
		\item[(a)] $\max_{1\le k\le K_{\mathrm{cf}}}\|\widehat I^{(-k)}-I(\btheta_o)\| = o_p(1)$.
		\item[(b)] With probability tending to one, each $\widehat I^{(-k)}$ is invertible and
		\[
		\max_{1\le k\le K_{\mathrm{cf}}}\big\| \{\widehat I^{(-k)}\}^{-1}-I(\btheta_o)^{-1}\big\| = o_p(1).
		\]
		\item[(c)] If, in addition, $E_0\{\|\ell^*_{\theta_o}(\bW)\|^2\}<\infty$, then
		\[
		\frac{1}{\sqrt n}\sum_{k=1}^{K_{\mathrm{cf}}}\sum_{i\in \mathcal I_k}
		\Big(
		\{\widehat I^{(-k)}\}^{-1}-I(\btheta_o)^{-1}
		\Big)\,
		\ell^*_{\theta_o}(\bW_i)
		=
		o_p(1).
		\]
	\end{enumerate}
\end{restatable}


\begin{proof}[Proof of \Cref{lem:inverse-info-plug-in}]
	Write $\mathbb P_{n,-k}$ for the empirical measure on the training subsample $\mathcal I_k^c$ and define
	\[
	\Psi_h(\bW)=\Delta\,\{\bZ-h(T,\bX)\}^{\otimes 2}.
	\]
	Then $\widehat I^{(-k)}=\mathbb P_{n,-k}\Psi_{\widehat \bg^{*,(-k)}}$ and $I(\theta_o)=\mathbb P_0\Psi_{g^*}$, hence
	\[
	\widehat I^{(-k)}-I(\btheta_o)
	=
	(\mathbb P_{n,-k}-\mathbb P_0)\Psi_{\widehat \bg^{*,(-k)}}
	+
	\mathbb P_0\big(\Psi_{\widehat \bg^{*,(-k)}}-\Psi_{\bg^*}\big).
	\]
	
	We first control the second (bias) term. Let $r^*(T,\bX)=\bZ-g^*(T,\bX)$ and
	$\widehat r^{(-k)}(T,\bX)=\bZ-\widehat g^{*,(-k)}(T,\bX)$.
	Expanding,
	\[
	\Psi_{\widehat \bg^{*,(-k)}}(\bW)-\Psi_{\bg^*}(\bW)
	=
	\Delta\Big[
	(\widehat r^{(-k)}-r^*)\otimes(\widehat r^{(-k)}-r^*)
	+(\widehat r^{(-k)}-r^*)\otimes r^*
	+r^*\otimes(\widehat r^{(-k)}-r^*)
	\Big].
	\]
	Under \ref{ass:A2} and \ref{ass:Anew}, $\|r^*\|$ is bounded almost surely and $\|\widehat r^{(-k)}-r^*\|=\|\widehat \bg^{*,(-k)}- \bg^*\|$.
	Therefore, for a constant $C_1<\infty$,
	\[
	\big\|\mathbb P_0(\Psi_{\widehat \bg^{*,(-k)}}-\Psi_{\bg^*})\big\|
	\le
	C_1\,\mathbb P_0\|\widehat \bg^{*,(-k)}(T,\bX)- \bg^*(T,\bX)\|
	+
	C_1\,\mathbb P_0\|\widehat \bg^{*,(-k)}(T,\bX)- \bg^*(T,\bX)\|^2.
	\]
	By Cauchy--Schwarz and the definition of $\|\cdot\|_{L^2(\PP_0,\Delta)}$, each of the two expectations above is
	$O_p(\|\widehat \bg^{*,(-k)}- \bg^*\|_{L^2(\PP_0,\Delta)})$ and $O_p(\|\widehat \bg^{*,(-k)}-\bg^*\|_{L^2(\PP_0,\Delta)}^2)$, respectively.
	Using \ref{ass:A9}, the bias term is $O_p(\tau_n)+O_p(\tau_n^2)=o_p(1)$, uniformly in $k$.

For the empirical term, no conditioning argument is needed.
Let $\mathcal G_n^*$ denote the function class used to estimate
the least-favorable projection, and define
\[
\mathcal Q_n
=
\left\{
\Psi_h(\bW)
=
\Delta\{\bZ-h(T,\bX)\}^{\otimes 2}
:
h\in\mathcal G_n^*
\right\}.
\]
By the boundedness of $\bZ$ and of the projection-function class,
$\mathcal Q_n$ has a uniformly bounded envelope. Moreover, since
$h\mapsto\Psi_h$ is Lipschitz on bounded sets, the entropy of
$\mathcal Q_n$ is of the same order as that of the neural-network
class used for the projection estimator. Under the network-complexity
conditions in \ref{ass:A3}, the corresponding uniform law of large numbers gives
\[
\sup_{h\in\mathcal G_n^*}
\left\|
(\PP_{n,-k}-\PP_0)\Psi_h
\right\|
=o_p(1),
\]
uniformly in $k=1,\ldots,K_{\mathrm{cf}}$.
Since $\widehat{\bg}^{*,(-k)}\in\mathcal G_n^*$,
\[
\max_{1\le k\le K_{\mathrm{cf}}}
\left\|
(\PP_{n,-k}-\PP_0)
\Psi_{\widehat{\bg}^{*,(-k)}}
\right\|
=o_p(1).
\]
Combining this with the preceding bias bound proves (a).
    
	For (b), fix $k$ and write $A=\widehat I^{(-k)}$ and $B=I(\theta_o)$.
	By (a), $\|A-B\|=o_p(1)$ uniformly in $k$, and by \ref{ass:A10}, $\|B^{-1}\|\le c_0^{-1}$.
	On the event $\|A-B\|\le (2\|B^{-1}\|)^{-1}$, the matrix $A$ is invertible and $\|A^{-1}\|\le 2\|B^{-1}\|$.
	Using the resolvent identity,
	\[
	A^{-1}-B^{-1}=A^{-1}(B-A)B^{-1},
	\]
	we obtain
	\[
	\|A^{-1}-B^{-1}\|
	\le
	\|A^{-1}\|\,\|A-B\|\,\|B^{-1}\|
	\le
	2\|B^{-1}\|^2\,\|A-B\|.
	\]
	Taking maxima over $k$ and using $\max_k\|A-B\|=o_p(1)$ gives (b).
	
	For (c), write
	\[
	R_n=
	\frac{1}{\sqrt n}\sum_{k=1}^{K_{\mathrm{cf}}}\sum_{i\in \mathcal I_k}
	\Big(
	\{\widehat I^{(-k)}\}^{-1}-I(\btheta_o)^{-1}
	\Big)\,
	\ell^*_{\theta_o}(\bW_i).
	\]

    By the triangle inequality,
\[
\begin{aligned}
\|R_n\|
&\le
\sum_{k=1}^{K_{\mathrm{cf}}}
\left\|
\{\widehat I^{(-k)}\}^{-1}-I(\theta_o)^{-1}
\right\|
\left\|
\frac{1}{\sqrt n}
\sum_{i\in I_k}\ell_{\theta_o}^{*}(W_i)
\right\| \\
&\le
\left(
\max_{1\le k\le K_{\mathrm{cf}}}
\left\|
\{\widehat I^{(-k)}\}^{-1}-I(\theta_o)^{-1}
\right\|
\right)
\sum_{k=1}^{K_{\mathrm{cf}}}
\left\|
\frac{1}{\sqrt n}
\sum_{i\in I_k}\ell_{\theta_o}^{*}(W_i)
\right\|.
\end{aligned}
\]
The first factor is $o_p(1)$ by part (b). For each fold $k$,
independence and the mean-zero property of the efficient scores give
\[
E_0\left\|
\frac{1}{\sqrt n}
\sum_{i\in I_k}\ell_{\theta_o}^{*}(W_i)
\right\|^2
=
\frac{|I_k|}{n}
E_0\|\ell_{\theta_o}^{*}(W)\|^2
\le
E_0\|\ell_{\theta_o}^{*}(W)\|^2
<\infty.
\]
Hence each fold-specific score sum is $O_p(1)$. Since
$K_{\mathrm{cf}}$ is fixed, their summed norms are also $O_p(1)$.
Therefore $\|R_n\|=o_p(1)$, proving part (c).
\end{proof}

Because the estimators use a time grid, the following lemma compares the true IF $\phi_t^0$,  its grid version $\phi_{t,n}^0$, and the fully plug-in cross-fitted estimate $\widehat\phi_t^{(-k)}$. It shows $\widehat\phi_t^{(-k)}$ is $L_2(\PP_0)$-close to the grid IF uniformly over folds, and since the grid error vanishes when $\sqrt n\,\delta_n\to 0$, this yields $L_2$-consistency for the true IF as well.

\begin{restatable}{lemma}{DiscreteGrid}
	\label{lem:L2phi_discrete}
	Assume \ref{ass:A1}--\ref{ass:A11}.	
	Let $\phi^0_{t,n}(\bW;\bX_0,\bZ_0)$  denote the grid version of \eqref{eq:EIF_MS} obtained by replacing
	$\int_0^\tau \psi_t^0(u,\bX;\bX_0,\bZ_0)\,dM(u)$ with
	$\sum_{j=1}^{m_n}\psi_t^0(t_j,\bX;\bX_0,\bZ_0)\Delta M_{j}$ and
	$c_{0,t}(\bX_0,\bZ_0)$ with its grid analogue.
	Then, uniformly in $k=1,\dots,K_{\mathrm{cf}}$,
	\[
	\EE_0\!\left[\Big(\widehat\phi_t^{(-k)}(\bW;\bX_0,\bZ_0)-\phi^0_{t,n}(\bW;\bX_0,\bZ_0)\Big)^2\right]
	=o_p(1),
	\]
	and therefore (since $\sqrt{n}\delta_n\to0$ implies $\phi^0_{t,n}$ and $\phi_t^0$ are asymptotically equivalent),
	\[
	\EE_0\!\left[\Big(\widehat\phi_t^{(-k)}(\bW;\bX_0,\bZ_0)-\phi^0_t(\bW;\bX_0,\bZ_0)\Big)^2\right]
	=o_p(1),
	\qquad \text{uniformly in }k.
	\]
\end{restatable}
    

\begin{proof}[Proof of \cref{lem:L2phi_discrete}]
	Write
	\[
	\widehat\phi_t^{(-k)}-\phi^0_{t,n}
	=
	\Big(\widehat\phi^{(-k)}_{t,g}-\phi^{0}_{t,g,n}\Big)
	+
	\Big(\widehat\phi^{(-k)}_{t,\theta}-\phi^{0}_{t,\theta,n}\Big),
	\]
	and control each term in $L^2(\PP_0)$. We start with the nuisance Riesz part.
	For brevity write
	\[
	\widehat\psi_j^{(-k)}=\widehat\psi^{(-k)}_t(t_j,\bX;\bX_0,\bZ_0) \, ,
	\qquad
	\psi^0_j=\psi^0_t(t_j,\bX;\bX_0,\bZ_0).
	\]
    Throughout this proof, write $Y_j=Y(t_j)=1\{T\ge t_j\}$.
	Also write
	$\Delta\widehat M^{(-k)}_j=\Delta N_j-\Delta A^{(-k)}_j$ and
	$\Delta M^0_j=\Delta N_j-\Delta A^0_j$,
	where
	$$	\Delta A^{(-k)}_j=(t_j-t_{j-1})Y_j\widehat h^{(-k)}(t_j\mid\bX,\bZ)$$
	and
	$$\Delta A^0_j=(t_j-t_{j-1})Y_jh_o(t_j\mid\bX,\bZ)\, .$$
	Then
	\begin{equation}\label{eq:phi-dif}
			\widehat\phi^{(-k)}_{t,g}-\phi^{0}_{t,g,n}
		=
		\sum_{j=1}^{m_n}\big(\widehat\psi_j^{(-k)}-\psi^0_j\big)\Delta M^0_j
		+
		\sum_{j=1}^{m_n}\widehat\psi_j^{(-k)}\big(\Delta\widehat M^{(-k)}_j-\Delta M^0_j\big).
	\end{equation}
	
	The first sum is a martingale transform with predictable integrand
	$\widehat\psi_j^{(-k)}-\psi^0_j$, hence by the discrete-time martingale isometry,
	\[
	\EE_0\!\left[\left(\sum_{j=1}^{m_n}(\widehat\psi_j^{(-k)}-\psi^0_j)\Delta M^0_j\right)^2\Bigg|\mathcal F_{-k}\right]
	=
	\EE_0\!\left[\sum_{j=1}^{m_n}(\widehat\psi_j^{(-k)}-\psi^0_j)^2\,\Delta A^0_j\Bigg|\mathcal F_{-k}\right].
	\]
	
	In the discrete finite-support case, $\psi_t^0$ has the closed form in \eqref{eq:psi_closed}, and
	$\widehat\psi_t^{(-k)}$ is obtained by replacing $h_o$ and the nuisance quantities appearing
	in \eqref{eq:psi_closed} by their foldwise plug-in versions. Write
	\[
	\pi_0=\Pr(\bX=\bX_0),\qquad
	\widehat\pi^{(-k)}=\widehat\PP^{(-k)}(\bX=\bX_0),
	\]
	and
	\[
	A_{\bX_0}(u)=E_0\!\left[Y(u)h_o(u\mid \bX,\bZ)\mid \bX=\bX_0\right],
	\qquad
	\widehat A^{(-k)}_{\bX_0}(u)
	=
	\widehat E_{-k}\!\left[Y(u)\widehat h^{(-k)}(u\mid \bX,\bZ)\mid \bX=\bX_0\right].
	\]
	Because \(\Pr(X=X_0)=\pi_0>0\) and \(X\) has finite support, we have
	\[
	\inf_k \widehat\pi^{(-k)} \ge c_\pi>0
	\]
	with probability tending to one. Next, by the bounded-sieve restriction for \(\widetilde g^{(-k)}\), boundedness of \(\bZ\), and boundedness of the parameter space for \(\btheta\), there exists a deterministic constant \(c_\lambda>0\) such that
	\[
	\inf_{k}\inf_{u\le t}\inf_{\bx,\bz}\widehat h^{(-k)}(u\mid \bx,\bz)\ge c_\lambda
	\]
	with probability tending to one. Hence, for \(u\le t\),
	\[
	\widehat A^{(-k)}_{\bX_0}(u)
	=
	\widehat E_{-k}\!\left[Y(u)\widehat h^{(-k)}(u\mid \bX,\bZ)\mid \bX=\bX_0\right]
	\ge
	c_\lambda\,\widehat E_{-k}\!\left[Y(u)\mid \bX=\bX_0\right].
	\]
	By the positivity assumptions, there exists \(c_Y>0\) such that
	\[
	\inf_{u\le t} E_0\!\left[Y(u)\mid \bX=\bX_0\right]\ge c_Y.
	\]
	Moreover, since \(Y(u)\) is monotone, the class \(\{Y(u):0\le u\le t\}\) is Glivenko--Cantelli, and therefore
	\[
	\sup_k\sup_{u\le t}
	\left|
	\widehat E_{-k}\!\left[Y(u)\mid \bX=\bX_0\right]
	-
	E_0\!\left[Y(u)\mid \bX=\bX_0\right]
	\right|
	=o_p(1).
	\]
	It follows that
	\[
	\inf_k\min_{j:t_j\le t}\widehat A^{(-k)}_{\bX_0}(t_j)\ge c_0
	\]
	with probability tending to one for some constant \(c_0>0\).

	Because $\Pr(\bX=\bX_0)\ge \pi_0>0$ and $\bX$ has finite support, $\widehat\pi^{(-k)}\to \pi_0$
	in probability uniformly in $k$. Define the event
	\[
	\mathcal E_n
	=
	\left\{
	\inf_k \widehat\pi^{(-k)} \ge c_0
	\ \text{ and }\
	\inf_k \min_{j:t_j\le t}\widehat A^{(-k)}_{\bX_0}(t_j)\ge c_0
	\right\}.
	\]
	By the above arguments, there exists \(c_0>0\) such that
	\[
	\PP(\mathcal E_n)\to 1.
	\]
	On $\mathcal E_n$, all denominator terms appearing in $\widehat\psi_t^{(-k)}$ are uniformly bounded away from zero.
	
	Now
	\[
	\widehat\psi^{(-k)}_{t}(u,\bX;\bX_0,\bZ_0)
	=
	\mathbf 1(u\le t)\mathbf 1(\bX=\bX_0)\,
	\frac{\widehat h^{(-k)}(u\mid \bX_0,\bZ_0)}
	{\widehat\pi^{(-k)}\,\widehat A^{(-k)}_{\bX_0}(u)},
	\]
	while
	\[
	\psi^{0}_{t}(u,\bX;\bX_0,\bZ_0)
	=
	\mathbf 1(u\le t)\mathbf 1(\bX=\bX_0)\,
	\frac{h_o(u\mid \bX_0,\bZ_0)}
	{\pi_0\,A_{\bX_0}(u)}.
	\]
	For readability, suppress the common arguments $(u\mid \bX_0,\bZ_0)$ and write
	\[
	\Delta_{h,j}^{(-k)}
	=
	\widehat h^{(-k)}(t_j\mid \bX_0,\bZ_0)-h_o(t_j\mid \bX_0,\bZ_0),
	\]
	\[
	\Delta_{\pi}^{(-k)}=\widehat\pi^{(-k)}-\pi_0,
	\qquad
	\Delta_{A,j}^{(-k)}=\widehat A^{(-k)}_{\bX_0}(t_j)-A_{\bX_0}(t_j).
	\]
	On $\mathcal E_n$, a first-order ratio expansion gives, for each $j$,
	\[
	\widehat\psi_j^{(-k)}-\psi_j^0
	=
	\mathbf 1(t_j\le t)\mathbf 1(\bX=\bX_0)
	\Big\{
	L^{(-k)}_{h,j}+L^{(-k)}_{\pi,j}+L^{(-k)}_{A,j}+R^{(-k)}_{\psi,j}
	\Big\},
	\]
	where
	\[
	L^{(-k)}_{h,j}
	=
	\frac{\Delta_{h,j}^{(-k)}}{\pi_0\,A_{\bX_0}(t_j)},
	\qquad
	L^{(-k)}_{\pi,j}
	=
	\frac{h_o(t_j\mid \bX_0,\bZ_0)}{A_{\bX_0}(t_j)}
	\Big(\frac{1}{\widehat\pi^{(-k)}}-\frac{1}{\pi_0}\Big),
	\]
	\[
	L^{(-k)}_{A,j}
	=
	\frac{h_o(t_j\mid \bX_0,\bZ_0)}{\widehat\pi^{(-k)}}
	\Big(\frac{1}{\widehat A^{(-k)}_{\bX_0}(t_j)}-\frac{1}{A_{\bX_0}(t_j)}\Big),
	\]
	and $R^{(-k)}_{\psi,j}$ is a finite sum of products of estimation errors, hence second order.
	Using $(a+b+c+d)^2\lesssim a^2+b^2+c^2+d^2$, we obtain on $\mathcal E_n$,
	\[
	(\widehat\psi_j^{(-k)}-\psi_j^0)^2
	\lesssim
	\mathbf 1(\bX=\bX_0)
	\Big\{
	|L^{(-k)}_{h,j}|^2
	+
	|L^{(-k)}_{\pi,j}|^2
	+
	|L^{(-k)}_{A,j}|^2
	+
	|R^{(-k)}_{\psi,j}|^2
	\Big\}.
	\]
	Therefore
	\begin{align*}
		\EE_0\!\left[\sum_{j=1}^{m_n}(\widehat\psi_j^{(-k)}-\psi_j^0)^2\,\Delta A^0_j\right]
		&\lesssim
		T^{(-k)}_{h}
		+
		T^{(-k)}_{\pi}
		+
		T^{(-k)}_{A}
		+
		T^{(-k)}_{R},
	\end{align*}
	where
	\[
	T^{(-k)}_{h}
	=
	\EE_0\!\left[
	\sum_{j=1}^{m_n}\mathbf 1(\bX=\bX_0)\,
	|\Delta_{h,j}^{(-k)}|^2\,\Delta A^0_j
	\right],
	\]
	\[
	T^{(-k)}_{\pi}
	=
	|\Delta_{\pi}^{(-k)}|^2\,
	\EE_0\!\left[\sum_{j=1}^{m_n}\mathbf 1(\bX=\bX_0)\,\Delta A^0_j\right],
	\]
	\[
	T^{(-k)}_{A}
	=
	\EE_0\!\left[
	\sum_{j=1}^{m_n}\mathbf 1(\bX=\bX_0)\,
	|\Delta_{A,j}^{(-k)}|^2\,\Delta A^0_j
	\right],
	\]
	and $T_R^{(-k)}$ is the analogous contribution of the second-order remainder $R^{(-k)}_{\psi,j}$.
	
	We bound these four terms separately.
	
	For $T^{(-k)}_{h}$, note that
	\[
	\Delta A^0_j=(t_j-t_{j-1})Y(t_{j})h_o(t_j\mid \bX,\bZ),
	\]
	so
	\[
	T^{(-k)}_{h}
	=
	\EE_0\!\left[
	\sum_{j=1}^{m_n}(t_j-t_{j-1})\,\mathbf 1(\bX=\bX_0)\,
	Y(t_{j})h_o(t_j\mid \bX,\bZ)\,
	|\Delta_{h,j}^{(-k)}|^2
	\right].
	\]
	By the same exponential-link expansion used in Lemma~13,
	$\widehat h^{(-k)}-h_o$ is linear in
	$\widetilde g^{(-k)}-g_o$ and $\widetilde \btheta^{(-k)}-\btheta_o$
	plus a quadratic remainder.
	Therefore, exactly as in Lemma~13,
	\[
	T^{(-k)}_{h}
	\lesssim
	\|\widetilde g^{(-k)}-g_o\|_{L^2(\PP_0,\Delta)}^2
	+
	\|\widetilde\btheta^{(-k)}-\btheta_o\|^2
	+
	o_p(1)
	=
O_p(\tau_n^2)+O_p(n^{-1})+o_p(1)
	=
	o_p(1),
	\]
	uniformly in $k$, by Assumption \ref{ass:A9}.

For \(T_A^{(-k)}\), write
\[
\Delta_{A,j}^{(-k)}=B_{A,j}^{(-k)}+D_{A,j}^{(-k)},
\]
where
\[
B_{A,j}^{(-k)}
=
\widehat E_{-k}\!\left\{Y(t_{j})\widehat h^{(-k)}(t_j\mid \bX,\bZ)\mid \bX=\bX_0\right\}
-
E_0\!\left\{Y(t_{j})\widehat h^{(-k)}(t_j\mid \bX,\bZ)\mid \bX=\bX_0\right\},
\]
and
\[
D_{A,j}^{(-k)}
=
E_0\!\left\{Y(t_{j})\big(\widehat h^{(-k)}(t_j\mid \bX,\bZ)-h_o(t_j\mid \bX,\bZ)\big)\mid \bX=\bX_0\right\}.
\]
Hence,
\[
T_A^{(-k)}
\lesssim
T_{A,1}^{(-k)}+T_{A,2}^{(-k)},
\]
where
\[
T_{A,1}^{(-k)}
=
E_0\!\left[
\sum_{j=1}^{m_n}\mathbf 1(\bX=\bX_0)\,
|B_{A,j}^{(-k)}|^2\,\Delta A_j^0
\right],
\]
and
\[
T_{A,2}^{(-k)}
=
E_0\!\left[
\sum_{j=1}^{m_n}\mathbf 1(\bX=\bX_0)\,
|D_{A,j}^{(-k)}|^2\,\Delta A_j^0
\right].
\]
We first bound \(T_{A,1}^{(-k)}\). Because \(\bX\) has finite support and \(\widehat h^{(-k)}\) is uniformly bounded
with probability tending to one, the variables
\[
Y(t_{j})\widehat h^{(-k)}(t_j\mid \bX,\bZ)\mathbf 1(\bX=\bX_0)
\]
are uniformly bounded over \(j\) and \(k\). Therefore, conditional on \(\mathcal F_{-k}\),
\[
E_0\!\left[\big(B_{A,j}^{(-k)}\big)^2\mid \mathcal F_{-k}\right]\lesssim \frac{1}{n}
\qquad\text{uniformly in }j,k.
\]
Thus
\begin{align*}
	E_0\!\left[T_{A,1}^{(-k)}\mid \mathcal F_{-k}\right]
	&\lesssim
	\frac{1}{n}\,
	E_0\!\left[\sum_{j=1}^{m_n}\mathbf 1(\bX=\bX_0)\,\Delta A_j^0 \,\middle|\, \mathcal F_{-k}\right] \\
	&\lesssim
	\frac{1}{n}\,
	E_0\!\left[\sum_{j=1}^{m_n}\Delta A_j^0\right]
	=
	O\!\left(\frac{1}{n}\right),
\end{align*}
since
\[
\sum_{j=1}^{m_n}\Delta A_j^0
=
\sum_{j=1}^{m_n}(t_j-t_{j-1})Y(t_{j})h_o(t_j\mid \bX,\bZ)
=O_{p}(1).
\]
Hence
\[
T_{A,1}^{(-k)}=O_p(n^{-1})=o_p(1)
\]
uniformly in \(k\).

Next, for \(T_{A,2}^{(-k)}\), Jensen's inequality gives
\[
|D_{A,j}^{(-k)}|^2
\le
E_0\!\left[
Y(t_{j})\big(\widehat h^{(-k)}(t_j\mid \bX,\bZ)-h_o(t_j\mid \bX,\bZ)\big)^2
\mid \bX=\bX_0
\right].
\]
Therefore,
\begin{align*}
	T_{A,2}^{(-k)}
	&\lesssim
	E_0\!\left[
	\sum_{j=1}^{m_n}\mathbf 1(\bX=\bX_0)\,
	E_0\!\left\{
	Y(t_{j})\big(\widehat  h^{(-k)}(t_j\mid \bX,\bZ)-h_o(t_j\mid \bX,\bZ)\big)^2
	\mid \bX=\bX_0
	\right\}
	\Delta A_j^0
	\right].
\end{align*}
Since \(\bX\) has finite support, this is bounded by a constant multiple of
\[
E_0\!\left[
\sum_{j=1}^{m_n}(t_j-t_{j-1})\,Y(t_{j})h_o(t_j\mid \bX,\bZ)\,
\big(\widehat h^{(-k)}(t_j\mid \bX,\bZ)-h_o(t_j\mid \bX,\bZ)\big)^2
\right],
\]
which is exactly the same integrated hazard-error quantity already controlled in the bound for
\(T_h^{(-k)}\). Hence, by the exponential-link expansion together with \ref{ass:A9} and the same quadratic remainder control as in Lemma~13,
\[
T_{A,2}^{(-k)}=o_p(1)
\]
uniformly in \(k\). Combining the two bounds yields
\[
T_A^{(-k)}=o_p(1)
\qquad\text{uniformly in }k.
\]

	Finally, $T_R^{(-k)}=o_p(1)$ uniformly in $k$, because each term in $R^{(-k)}_{\psi,j}$ is a product
	of first-order estimation errors and hence contributes only second-order terms after summation with
	weight $\Delta A_j^0$.
	
	Combining the four bounds yields
	\[
	\EE_0\!\left[\sum_{j=1}^{m_n}(\widehat\psi_j^{(-k)}-\psi_j^0)^2\,\Delta A^0_j\right]
	=o_p(1),
	\qquad \text{uniformly in }k.
	\]
	Consequently, the first sum is $o_p(1)$ in $L^2(\PP_0)$ uniformly in $k$.
	
	For the second sum of (\ref{eq:phi-dif}), note that
	\[
	\Delta\widehat M^{(-k)}_j-\Delta M^0_j
	=
	-(\Delta A^{(-k)}_j-\Delta A^0_j),
	\]
	so we set
	\[
	S_{2,k}
	=
	\sum_{j=1}^{m_n}\widehat\psi_j^{(-k)}\big(\Delta \widehat M_j^{(-k)}-\Delta M_j^0\big)
	=
	-\sum_{j=1}^{m_n}\widehat\psi_j^{(-k)}(t_j-t_{j-1})Y_j
	\big\{\widehat h^{(-k)}(t_j\mid \bX,\bZ)-h_o(t_j\mid \bX,\bZ)\big\}.
	\]
	On the event \(\mathcal E_n\), the discrete closed form of \(\widehat\psi_j^{(-k)}\) and boundedness of the denominators imply
	\[
	\sup_{k}\sup_{j:t_j\le t}\big|\widehat\psi_j^{(-k)}\big|=O_p(1).
	\]
	Hence, by Cauchy--Schwarz,
	\begin{align*}
		|S_{2,k}|^2
		&\lesssim
		\left(
		\sum_{j=1}^{m_n}(t_j-t_{j-1})Y_j
		\big|\widehat h^{(-k)}(t_j\mid \bX,\bZ)-h_o(t_j\mid \bX,\bZ)\big|
		\right)^2 \\
		&\le
		\left(\sum_{j=1}^{m_n}\Delta A_j^0\right)
		\left(
		\sum_{j=1}^{m_n}
		(t_j-t_{j-1})Y_j
		\frac{\big(\widehat h^{(-k)}(t_j\mid \bX,\bZ)-h_o(t_j\mid \bX,\bZ)\big)^2}
		{h_o(t_j\mid \bX,\bZ)}
		\right).
	\end{align*}
	Since \(h_o\) is bounded away from zero and infinity, and
	\[
	\sum_{j=1}^{m_n}\Delta A_j^0
	=
	\sum_{j=1}^{m_n}(t_j-t_{j-1})Y_jh_o(t_j\mid \bX,\bZ)
	=O_{p}(1),
	\]
	it follows that
	\[
	\EE_0\!\big[|S_{2,k}|^2\big]
	\lesssim
	\EE_0\!\left[
	\sum_{j=1}^{m_n}(t_j-t_{j-1})Y_j
	\big(\widehat h^{(-k)}(t_j\mid \bX,\bZ)-h_o(t_j\mid \bX,\bZ)\big)^2
	\right].
	\]
	The right-hand side is exactly the same weighted integrated hazard-error quantity already controlled above in the bound for \(T_h^{(-k)}\). Therefore, by the exponential-link expansion, Assumption \ref{ass:A9}, and the same quadratic remainder control as in Lemma~13,
	\[
	\EE_0\!\big[|S_{2,k}|^2\big]=o_p(1)
	\qquad\text{uniformly in }k.
	\]
	Thus the second sum of (\ref{eq:phi-dif}) is \(o_p(1)\) in \(L^2(\PP_0)\), uniformly in \(k\).	Therefore,
	\[
	\EE_0\Big[\big(\widehat\phi^{(-k)}_{t,g}-\phi^{0}_{t,g,n}\big)^2\Big]=o_p(1)
	\quad \text{uniformly in }k.
	\]
	For the parametric part, we write
	\[
	\widehat\phi^{(-k)}_{t,\theta}-\phi^{0}_{t,\theta,n}
	=
	\Big(\widehat c_t^{(-k)}-c_{0,t,n}\Big)^\top I(\btheta_o)^{-1}\ell^*_{\theta_o}
	+
	\widehat c_t^{(-k)\top}\Big(\{\widehat I^{(-k)}\}^{-1}-I(\btheta_o)^{-1}\Big)\ell^*_{\theta_o}
	+
	\widehat c_t^{(-k)\top}\{\widehat I^{(-k)}\}^{-1}\Big(\widehat\ell^{*,(-k)}_{\theta_o}-\ell^*_{\theta_o}\Big),
	\]
	where all quantities are evaluated at $(\bX_0,\bZ_0)$ and $\bW$ (we suppress these arguments).
	In the discrete case,
	\[
	\widehat c_t^{(-k)}(\bX_0,\bZ_0)
	=
	\sum_{j:t_j\le t}(t_j-t_{j-1})\widehat h^{(-k)}(t_j\mid\bX_0,\bZ_0)
	\Big\{\bZ_0-\widehat \bg^{*,(-k)}(t_j,\bX_0)\Big\},
	\]
	so by boundedness and the rate conditions in Assumption \ref{ass:A9}
	$\|\widehat c_t^{(-k)}-c_{0,t,n}\|=o_p(1)$ uniformly in $k$.
	Since $\EE_0\|\ell^*_{\theta_o}(\bW)\|^2<\infty$ under the standing integrability assumptions,
	the first term is $o_p(1)$ in $L^2(\PP_0)$.
	
	For the second term, Lemma~\ref{lem:inverse-info-plug-in} yields
	\[
	\max_k\|(\widehat I^{(-k)})^{-1}-I(\btheta_o)^{-1}\|=o_p(1),
	\]
	while $\widehat c_t^{(-k)}$ is bounded in probability
	and $\ell^*_{\theta_o}\in L^2(\PP_0)$, so the product is $o_p(1)$ in $L^2(\PP_0)$ uniformly in $k$.
	
By the decomposition in the proof of Lemma~13, we have
\[
\EE_0\!\left[
\left\|
\widehat\ell_{\theta_o}^{*,(-k)}
-\ell_{\theta_o}^*
\right\|_2^2
\mid\mathcal F_{-k}
\right]
=
O_p(\tau_n^2+n^{-1}),
\]
uniformly in \(k\). Therefore
\[
\widehat\ell_{\theta_o}^{*,(-k)}-\ell_{\theta_o}^* = O_p(\tau_n)
\quad\text{in }L^2(\PP_0),
\]
and the third term is \(o_p(1)\) in \(L^2(\PP_0)\).

Putting the three bounds together gives
	\[
	\EE_0\Big[\big(\widehat\phi^{(-k)}_{t,\theta}-\phi^{0}_{t,\theta,n}\big)^2\Big]=o_p(1)
	\quad \text{uniformly in }k.
	\]
	
	\medskip
	Combining the nuisance and parametric parts proves
	\[
	\EE_0\Big[(\widehat\phi^{(-k)}_t-\phi^0_{t,n})^2\Big]=o_p(1)
	\quad \text{uniformly in }k.
	\]
	Finally, $\sqrt n\,\delta_n\to 0$ implies that the grid approximation error between $\phi^0_{t,n}$ and $\phi_t^0$
	is $o(1)$ in $L^2(\PP_0)$, so the same conclusion holds with $\phi_t^0$ in place of $\phi^0_{t,n}$.
\end{proof}

\DiscreteCLT*

\begin{proof}[Proof of \Cref{thm:cf_discrete_CLT}]
	We suppress $(t,\bX_0,\bZ_0)$ from the notation when no confusion arises.
	Let $\{ \mathcal I_k\}_{k=1}^{K_{\mathrm{cf}}}$ be the fold partition of $\{1,\dots,n\}$ and, for each $k$, let $\mathcal F_{-k}$ be the $\sigma$--field generated by the training subsample $\{\bW_i:i\in \mathcal I_k^c\}$, the fold split, and $(\bX_0,\bZ_0)$.
	By construction, $\{\bW_i:i\in \mathcal I_k\}$ is independent of $\mathcal F_{-k}$. For each fold $k$, write
	\[
	\delta g^{(-k)}(u,\bX)=\widetilde g^{(-k)}(u,\bX)-g_o(u,\bX),
	\qquad
	\delta\btheta^{(-k)}=\widetilde\btheta^{(-k)}-\btheta_o,
	\]
	and  write $\Delta t_j=t_j-t_{j-1}$.
		The cross-fitted one-step estimator is
	\[
	\widehat H^{\,1,{\sf cf}}
	=
	\widehat H^{\,{\sf cf}}
	+
	\frac{1}{n}\sum_{k=1}^{K_{\mathrm{cf}}}\sum_{i\in \mathcal I_k}\widehat\phi_t^{(-k)}(\bW_i),
	\qquad
	\widehat H^{\,{\sf cf}}
	=
	\sum_{k=1}^{K_{\mathrm{cf}}} \frac{|\mathcal{I}_k|}{n} \widehat H^{(-k)},
	\]
	where the out-of-fold plug-in estimator uses the Riemann sum
	\[
	\widehat H^{(-k)}(t,\bX_0,\bZ_0)
	=
	\sum_{{j:t_j\le t}}
	\widehat h^{(-k)}(t_{j}\mid \bX_0,\bZ_0)\,\Delta t_j \, .
	\]
	Introduce also the Riemann approximation of the \emph{true} cumulative hazard,
	\[
	\widetilde H_o(t,\bX_0,\bZ_0)
	=
	\sum_{{j:t_j\le t}}
	h_o(t_{j}\mid \bX_0,\bZ_0)\,\Delta t_j,
	\qquad
	H_o(t,\bX_0,\bZ_0)=\int_0^t h_o(u\mid \bX_0,\bZ_0)\,du.
	\]
	Under the time-discretization condition \ref{ass:A11}, we have
	\begin{equation}
		\label{eq:riemann_true_error}
		\big|\widetilde H_o-H_o\big| \le C\,\delta_n,
	\end{equation}
	for a finite constant $C$ (possibly random but tight under the joint law), and thus $\sqrt n(\widetilde H_o-H_o)=o(1)$ whenever $\sqrt n\,\delta_n\to0$.
	
	Decompose
	\[
	\widehat H^{\,1,{\sf cf}}-H_o
	=
	\Big(\widehat H^{\,{\sf cf}}-\widetilde H_o\Big)
	+
	\Big(\widetilde H_o-H_o\Big)
	+
	\frac{1}{n}\sum_{k=1}^{K_{\mathrm{cf}}}\sum_{i\in \mathcal I_k}\widehat\phi_t^{(-k)}(\bW_i).
	\]
	The second term is negligible by \eqref{eq:riemann_true_error}, so it suffices to analyze
	\[
	A_n=\widehat H^{\,{\sf cf}}-\widetilde H_o,
	\qquad
	B_n=\frac{1}{n}\sum_{k=1}^{K_{\mathrm{cf}}}\sum_{i\in \mathcal I_k}\widehat\phi_t^{(-k)}(\bW_i).
	\]
	Write $A_n=\sum_{k=1}^{K_{\mathrm{cf}}} \frac{|\mathcal I_k|}{n}(\widehat H^{(-k)}-\widetilde H_o)$.
	For each $k$,
	\[
	\widehat H^{(-k)}-\widetilde H_o
	=
	\sum_{j:t_j\le t}
	\Big\{\widehat h^{(-k)}(t_{j}\mid\bX_0,\bZ_0)-h_o(t_{j}\mid\bX_0,\bZ_0)\Big\}\Delta t_j.
	\]
	Define the fold-wise log-perturbation on the grid
	\[
	r_j^{(-k)} =
	\delta g^{(-k)}(t_{j},\bX_0) + (\delta\btheta^{(-k)})^\top \bZ_0,
	\qquad
	j=1,\dots,m_n.
	\]
	Then
	\[
	\widehat h^{(-k)}(t_{j}\mid\bX_0,\bZ_0)
	=
	h_o(t_{j}\mid\bX_0,\bZ_0)\exp\{r_j^{(-k)}\}.
	\]
	Using $\exp(r)-1=r+\frac12 r^2\exp(\tilde r)$ for some $\tilde r$ between $0$ and $r$ yields
	\begin{equation}
		\label{eq:plugin_riemann_expand}
		\widehat H^{(-k)}-\widetilde H_o
		=
		\sum_{j:t_j\le t}
		h_o(t_{j}\mid\bX_0,\bZ_0)\,r_j^{(-k)}\,\Delta t_j
		+
		R_{n,1}^{(-k)},
	\end{equation}
	with remainder bounded by
	\begin{equation}
		\label{eq:Rn1_riemann_bound}
		\big|R_{n,1}^{(-k)}\big|
		\le
		\frac12\sum_{j:t_j\le t}
		h_o(t_{j}\mid\bX_0,\bZ_0)\,(r_j^{(-k)})^2\exp\{|r_j^{(-k)}|\}\,\Delta t_j.
	\end{equation}
	The linear term in \eqref{eq:plugin_riemann_expand} decomposes into a $g$-part and a $\btheta$-part:
	\[
	\sum_{j:t_j\le t}
	h_o(t_{j}\mid\bX_0,\bZ_0)\,r_j^{(-k)}\,\Delta t_j
	=
	L_{t,n}\!\big(\delta g^{(-k)}\big)
	+
	\big(\bZ_0\,\widetilde H_o\big)^\top\delta\btheta^{(-k)},
	\]
	where the discrete functional is
	\[
	L_{t,n}(b) =
	\sum_{j:t_j\le t}
	h_o(t_{j}\mid\bX_0,\bZ_0)\,b(t_{j},\bX_0)\,\Delta t_j.
	\]
	Since $\widetilde H_o=H_o+O(\delta_n)$, we may replace $\bZ_0\widetilde H_o$ by $\bZ_0 H_o$ at cost
	$O(\|\bZ_0\|\,\delta_n\,\|\delta\btheta^{(-k)}\|)$, which is $o_p(n^{-1/2})$ under
	$\sqrt n\,\delta_n\to0$ and $\|\delta\btheta^{(-k)}\|=O_p(n^{-1/2})$.
	
	Next decompose the one-step correction into conditional mean plus fluctuation. Define
	\[
	B_{n,0}
	=
	\frac{1}{n}\sum_{k=1}^{K_{\mathrm{cf}}}\sum_{i\in \mathcal I_k}
	\Big\{\widehat\phi_t^{(-k)}(\bW_i) - \EE_0\big[\widehat\phi_t^{(-k)}(\bW)\mid \mathcal F_{-k}\big]\Big\},
	\qquad
	B_{n,1}
	=
	\sum_{k=1}^{K_{\mathrm{cf}}} \frac{|\mathcal I_k|}{n}
	\EE_0\big[\widehat\phi_t^{(-k)}(\bW)\mid \mathcal F_{-k}\big],
	\]
	so that $B_n=B_{n,0}+B_{n,1}$.
	In the discrete finite-support setting, the Riesz representer has a closed form and the foldwise estimator
	$\widehat\psi_{t}^{(-k)}$ is $\mathcal F_{-k}$-measurable. 
	For each fold $k$, conditional on $\mathcal F_{-k}$ the  estimators
	$\widetilde g^{(-k)}$, $\widetilde \btheta^{(-k)}$, and
	$\widehat\psi_{t}^{(-k)}$ are fixed functions,
	and the observations $\{\bW_i:i\in \mathcal I_k\}$ are independent of
	$\mathcal F_{-k}$.
	Since, for each $i$, the process
	\[
	M_{0,i}(t)
	=
	N_i(t)
	-
	\int_0^t Y_i(u)h_o(u\mid\bX_i,\bZ_i)\,du
	\]
	is a square-integrable martingale with respect to its natural filtration,
	it follows that for any $\mathcal F_{-k}$-measurable predictable process
	$h(u,\bX_i)$,
	\[
	\EE_0\!\left[
	\int_0^\tau h(u,\bX_i)\,dM_{0,i}(u)
	\ \middle|\ \mathcal F_{-k}
	\right]
	=0.
	\]
	This identity justifies the foldwise conditional mean calculations below.

Importantly, the cancellation of the first-order nuisance term does not require $\delta g^{(-k)}=\widetilde g^{(-k)}-g_o$ to belong to the nuisance tangent space. In the finite-support setting, the explicit function $\psi_t^0$ in (5.10) satisfies the Riesz identity for every square-integrable function $b(u,\bX)$ for which the two sides are finite, as shown in the proof of Lemma 3. Hence the identity applies directly to $b=\delta g^{(-k)}$.

In particular, the foldwise conditional mean satisfies
	\begin{equation}
		\label{eq:mean_correction_discrete_riemann}
		\EE_0\big[\widehat\phi_t^{(-k)}(\bW)\mid \mathcal F_{-k}\big]
		=
		-\,L_{t,n}\!\big(\delta g^{(-k)}\big)
		-\big(\bZ_0 H_o(t\mid\bX_0,\bZ_0)\big)^\top\delta\btheta^{(-k)}
		+R_{n,2}^{(-k)},
	\end{equation}
	where $R_{n,2}^{(-k)}$ is second order and satisfies
	\begin{equation}
		\label{eq:Rn2_bound}
		\big|R_{n,2}^{(-k)}\big|
		\lesssim
		\|\widetilde g^{(-k)}-g_o\|_{L^2(\PP_0,\Delta)}^2+\|\widetilde\btheta^{(-k)}-\btheta_o\|^2,
	\end{equation}
	using the finite-support identity established in Lemma~3, together with  corresponding foldwise analogues
	and the uniform boundedness of $h_o$ ensured by Assumptions A1--A4.
	
	Combine \eqref{eq:plugin_riemann_expand} with \eqref{eq:mean_correction_discrete_riemann} and the foldwise expansions
	and use the replacement $\widetilde H_o=H_o+O(\delta_n)$ discussed above to obtain
	\[
	A_n+B_{n,1}
	=
	\sum_{k=1}^{K_{\mathrm{cf}}}  \frac{|\mathcal I_k|}{n}  \Big(R_{n,1}^{(-k)}+R_{n,2}^{(-k)}\Big)
	+o_p(n^{-1/2})
	\]
	under the joint law.
	It remains to bound the remainders.
	The bound \eqref{eq:Rn1_riemann_bound} together with $\exp\{|r|\}\le 1+|r|\exp\{|r|\}$ implies that, on events where
	$\max_{j\le m_n}|r_j^{(-k)}|$ is bounded (which holds with probability tending to one by the uniform boundedness assumptions),
	\[
	|R_{n,1}^{(-k)}|
	\lesssim
	\sum_{j:t_j\le t} h_o(t_{j}\mid\bX_0,\bZ_0)\,(r_j^{(-k)})^2\,\Delta t_j
	\lesssim
	\big(L_{t,n}(|\delta g^{(-k)}|)\big)^2 + \|\delta\btheta^{(-k)}\|^2 \, .
	\]
	We now bound the discrete linear functional $L_{t,n}(b)$.
	By Cauchy--Schwarz and boundedness of $h_o$,
	\[
	|L_{t,n}(b)|
	\lesssim
	\left(
	\sum_{j:t_j\le t}
	h_o(t_{j}\mid\bX_0,\bZ_0)
	b(t_{j},\bX_0)^2\Delta t_j
	\right)^{1/2}.
	\]
	The discrete sum is a Riemann approximation of
	$\int_0^t h_o(u\mid\bX_0,\bZ_0)b(u,\bX_0)^2du$
	with error $O(\delta_n\|b\|_\infty^2)$.
	Under the boundedness of the sieve and $\sqrt n\,\delta_n\to0$,
	this error is negligible.
	Because $\bX$ takes values in a finite set and $\PP(\bX=x_0)\ge c_0>0$,
	we may write
	\[
	\int_0^t h_o(u\mid\bX_0,\bZ_0)b(u,\bX_0)^2du
	=
	\frac{1}{\PP(\bX=x_0)}
	\EE_0\!\left[
	\int_0^t h_o(u\mid\bX,\bZ)b(u,\bX)^2du
	\,\mathbf 1\{\bX=x_0\}
	\right].
	\]
	Under the overlap condition
	$\EE_0\{Y(u)\mid \bX,\bZ\}\ge \kappa>0$
	and boundedness of $h_o$,
	the right-hand side is bounded by a constant multiple of
	\[
	\EE_0\!\left[
	\int_0^\tau Y(u)h_o(u\mid\bX,\bZ)b(u,\bX)^2du
	\right]
	=
	\|b\|_{L^2(\PP_0,\Delta)}^2.
	\]
	Therefore
	\[
	|L_{t,n}(b)|\lesssim \|b\|_{L^2(\PP_0,\Delta)} + o(n^{-1/2}).
	\]
	Consequently, $|R_{n,1}^{(-k)}|\lesssim \|\delta g^{(-k)}\|_{L^2(\PP_0,\Delta)}^2+\|\delta\btheta^{(-k)}\|^2$ in probability, uniformly in $k$.
	Together with \eqref{eq:Rn2_bound} we obtain
	\begin{equation}
		\label{eq:AnBn1_bound_final}
		|A_n+B_{n,1}|
		\lesssim
		\max_{k}\|\widetilde g^{(-k)}-g_o\|_{L^2(\PP_0,\Delta)}^2
		+
		\max_{k}\|\widetilde \btheta^{(-k)}-\btheta_o\|^2
		+
		o_p(n^{-1/2}).
	\end{equation}
	Under the assumption $\max_k\|\widetilde g^{(-k)}-g_o\|_{L^2(\PP_0,\Delta)} = O_p(\gamma_n\log^2 n)$ with
	$\gamma_n\log^2 n=o(n^{-1/4})$, we have $\sqrt n\,\max_k\|\widetilde g^{(-k)}-g_o\|_{L^2(\PP_0,\Delta)}^2=o_p(1)$.
	Under $\max_k\|\widetilde \btheta^{(-k)}-\btheta_o\|=O_p(n^{-1/2})$, we also have
	$$\sqrt n\,\max_k\|\widetilde\btheta^{(-k)}-\btheta_o\|^2=o_p(1) \, . $$
	Therefore \eqref{eq:AnBn1_bound_final} implies
	\[
	\sqrt n\,(A_n+B_{n,1})=o_p(1).
	\]
	
	We now analyze \(B_{n,0}\) fold by fold.
	For each \(k\), define
	\[
	B_{n,0}^{(k)}
	=
	\frac{1}{|\mathcal I_k|}\sum_{i\in\mathcal I_k}
	\Big\{
	\widehat\phi_t^{(-k)}(\bW_i)
	-
	\EE_0\big[\widehat\phi_t^{(-k)}(\bW)\mid \mathcal F_{-k} \big]
	\Big\},
	\]
	so that
	$
	B_{n,0}
	=
	\sum_{k=1}^{K_{\mathrm{cf}}} \frac{|\mathcal I_k|}{n}\,B_{n,0}^{(k)}$.
   Fix \(k\). Conditional on \(\mathcal F_{-k}\), the variables
	\(\{\bW_i:i\in\mathcal I_k\}\) are i.i.d., and
	\(\widehat\phi_t^{(-k)}\) is fixed.
	Hence the summands in \(B_{n,0}^{(k)}\) are conditionally i.i.d.\ and mean zero.
	Now insert and subtract the true IF. For \(i\in\mathcal I_k\), write
	\[
	\Delta_{i,k}
	=
	\widehat\phi_t^{(-k)}(\bW_i;\bX_0,\bZ_0)-\phi_t^0(\bW_i;\bX_0,\bZ_0),
	\qquad
	\Delta_k(\bW)
	=
	\widehat\phi_t^{(-k)}(\bW;\bX_0,\bZ_0)-\phi_t^0(\bW;\bX_0,\bZ_0).
	\]
	Then
	\begin{align*}
		B_{n,0}^{(k)}
		&=
		\frac{1}{|\mathcal I_k|}\sum_{i\in\mathcal I_k}
		\Big\{
		\phi_t^0(\bW_i;\bX_0,\bZ_0)
		-
		\EE_0[\phi_t^0(\bW;\bX_0,\bZ_0)\mid \mathcal F_{-k}]
		\Big\} \\
		&\qquad
		+
		\frac{1}{|\mathcal I_k|}\sum_{i\in\mathcal I_k}
		\Big\{
		\Delta_{i,k}
		-
		\EE_0[\Delta_k(\bW)\mid \mathcal F_{-k}]
		\Big\}.
	\end{align*}
	Because \(\bW\) is independent of the training subsample generating \(\mathcal F_{-k}\), and \((\bX_0,\bZ_0)\in\mathcal F_{-k}\), we have
	\[
	\EE_0[\phi_t^0(\bW;\bX_0,\bZ_0)\mid \mathcal F_{-k}]
	=
	\EE_0[\phi_t^0(\bW;\bX_0,\bZ_0)\mid \bX_0,\bZ_0].
	\]
	Since \(\phi_t^0\) is a mean-zero IF,
	\[
	\EE_0[\phi_t^0(\bW;\bX_0,\bZ_0)\mid \bX_0,\bZ_0]=0 \, ,
	\]
	and
	\[
	\EE_0[\phi_t^0(\bW;\bX_0,\bZ_0)\mid \mathcal F_{-k}]=0.
	\]
	Therefore
	\[
	B_{n,0}^{(k)}
	=
	\frac{1}{|\mathcal I_k|}\sum_{i\in\mathcal I_k}\phi_t^0(\bW_i;\bX_0,\bZ_0)
	+
	R_{n,3}^{(k)},
	\]
	where
	\[
	R_{n,3}^{(k)}
	=
	\frac{1}{|\mathcal I_k|}\sum_{i\in\mathcal I_k}
	\Big\{
	\Delta_{i,k}
	-
	\EE_0[\Delta_k(\bW)\mid \mathcal F_{-k}]
	\Big\}.
	\]
	We next show that the remainder is negligible.
	Conditional on \(\mathcal F_{-k}\), the summands in \(R_{n,3}^{(k)}\) are centered and independent, so
	\[
	\EE_0\!\left[\left(R_{n,3}^{(k)}\right)^2\middle|\mathcal F_{-k}\right]
	\le
	\frac{1}{|\mathcal I_k|}
	\EE_0\!\left[\Delta_k(\bW)^2\middle|\mathcal F_{-k} \right].
	\]
	Based on Lemma~\ref{lem:L2phi_discrete},
	\[
	\max_{1\le k\le K}
	\EE_0\!\left[\Delta_k(\bW)^2 \right]
	=
	o_p(1),
	\]
	and therefore, by taking expectations,
	\[
	\EE_0\!\left[\left(R_{n,3}^{(k)}\right)^2\right]
	\le
	\frac{1}{|\mathcal I_k|}
	\EE_0\!\left[\Delta_k(\bW)^2\right]
	=
	o\!\left(\frac1n\right),
	\]
	uniformly in $k$, because $|\mathcal I_k|\asymp n$ and $K$ is fixed. Hence
$$
	\sum_{k=1}^K \frac{|\mathcal I_k|}{n}R_{n,3}^{(k)}=o_p(n^{-1/2}).
$$	
	Consequently,
	\begin{align*}
		B_{n,0}
		=
		\sum_{k=1}^{K_{\mathrm{cf}}} \frac{|\mathcal I_k|}{n}
		\left\{
		\frac{1}{|\mathcal I_k|}\sum_{i\in\mathcal I_k}\phi_t^0(\bW_i;\bX_0,\bZ_0)
		+
		R_{n,3}^{(k)}
		\right\} 
		=
		\frac{1}{n}\sum_{i=1}^n \phi_t^0(\bW_i;\bX_0,\bZ_0)
		+
		o_p(n^{-1/2}),
	\end{align*}
	and hence
	\[
	\sqrt n\,B_{n,0}
	=
	\frac{1}{\sqrt n}\sum_{i=1}^n \phi_t^0(\bW_i;\bX_0,\bZ_0)
	+
	o_p(1).
	\]
	Combining this with $\sqrt n(A_n+B_{n,1})=o_p(1)$ and $\sqrt n(\widetilde H_o-H_o)=o(1)$ yields
	\[
	\sqrt n\big(\widehat H^{\,1,{\sf cf}}-H_o\big)
	=
	\frac{1}{\sqrt n}\sum_{i=1}^n \phi_t^0(\bW_i;\bX_0,\bZ_0)+o_p(1).
	\]
	Conditional on $(\bX_0,\bZ_0)$, the summands $\phi_t^0(\bW_i;\bX_0,\bZ_0)$ are i.i.d.\ with mean $0$ and variance
	$\sigma_H^2(t;\bX_0,\bZ_0)=\Var(\phi_t^0(\bW;\bX_0,\bZ_0)\mid\bX_0,\bZ_0)$, and have finite second moment by Lemma~\ref{lem:EIF_existence_MS}.
	The (conditional) Lindeberg--Feller CLT therefore gives
	\[
	\frac{1}{\sqrt n}\sum_{i=1}^n \phi_t^0(\bW_i;\bX_0,\bZ_0)
	\Rightarrow
	\mathcal N\!\big(0,\sigma_H^2(t;\bX_0,\bZ_0)\big),
	\]
	proving conditional asymptotic normality of
	\(\widehat H^{\,1,{\sf cf}}(t\mid \bX_0,\bZ_0)\)
	given \((\bX_0,\bZ_0)\).
	
	For $\widehat S^{\,1,{\sf cf}}=\exp\{-\widehat H^{\,1,{\sf cf}}\}$, the delta method applies since $x\mapsto e^{-x}$ is continuously differentiable, giving
	\[
	\sqrt n\big(\widehat S^{\,1,{\sf cf}}-S_o\big)
	=
	-\,S_o\,
	\sqrt n\big(\widehat H^{\,1,{\sf cf}}-H_o\big)+o_p(1),
	\]
	and thus
	\[
	\sqrt n\big(\widehat S^{\,1,{\sf cf}}-S_o\big)
	\rightarrow^{\mathcal D}
	\mathcal N\!\big(0, S_o(t\mid\bX_0,\bZ_0)^2\,\sigma_H^2(t;\bX_0,\bZ_0)\big).
	\]
	
	Finally, consider the variance estimator
	\[
	\widehat\sigma_H^2(t;\bX_0,\bZ_0)
	=
	\frac{1}{n}\sum_{k=1}^{K_{\mathrm{cf}}}\sum_{i\in \mathcal I_k}
	\Big\{\widehat\phi_t^{(-k)}(\bW_i;\bX_0,\bZ_0)\Big\}^2.
	\]
	For each $i$, let $k(i)$ denote the fold index such that $i\in\mathcal I_{k(i)}$, and write
	\[
	\widehat\phi_i
	=
	\widehat\phi_t^{(-k(i))}(\bW_i;\bX_0,\bZ_0),
	\qquad
	\phi_i
	=
	\phi_t^0(\bW_i;\bX_0,\bZ_0),
	\qquad
	\mathbb P_n f=\frac{1}{n}\sum_{i=1}^n f(\bW_i).
	\]
	Then
	\[
	\widehat\sigma_H^2(t;\bX_0,\bZ_0)-\sigma_H^2(t;\bX_0,\bZ_0)
	=
	\big(\mathbb P_n-\PP_0\big)\phi^2
	+
	\mathbb P_n\big(\widehat\phi^2-\phi^2\big).
	\]
	
	Conditionally on $(\bX_0,\bZ_0)$, $\{\phi_i^2\}_{i=1}^n$ are i.i.d.\ with
	$\EE_0(\phi^2\mid\bX_0,\bZ_0)<\infty$, hence the law of large numbers yields
	$\big(\mathbb P_n-\PP_0\big)\phi^2=o_p(1)$.
	For the second term, use $\widehat\phi^2-\phi^2=(\widehat\phi-\phi)(\widehat\phi+\phi)$ and Cauchy--Schwarz, and get
	\[
	\Big|\mathbb P_n(\widehat\phi^2-\phi^2)\Big|
	\le
	\Big\{\mathbb P_n(\widehat\phi-\phi)^2\Big\}^{1/2}
	\Big\{\mathbb P_n(\widehat\phi+\phi)^2\Big\}^{1/2}.
	\]
	By the $L^2(\PP_0)$-consistency established above,
	$\EE_0[(\widehat\phi-\phi)^2]\to0$ (uniformly in folds), which implies
	$\mathbb P_n(\widehat\phi-\phi)^2=o_p(1)$.
	Moreover,
	\[
	\mathbb P_n(\widehat\phi+\phi)^2
	\le
	2\mathbb P_n\widehat\phi^2+2\mathbb P_n\phi^2
	=
	O_p(1),
	\]
	since $\sup_n \EE_0(\widehat\phi^2)<\infty$ and $\EE_0(\phi^2)<\infty$.
	Therefore $\mathbb P_n(\widehat\phi^2-\phi^2)=o_p(1)$, and hence
	\[
	\widehat\sigma_H^2(t;\bX_0,\bZ_0)\xrightarrow{p}\sigma_H^2(t;\bX_0,\bZ_0).
	\]
	Finally, by the delta method,
	\[
	\sqrt n\Big\{\widehat S^{\,1,{\sf cf}}(t\mid\bX_0,\bZ_0)-S_o(t\mid\bX_0,\bZ_0)\Big\}
	\Rightarrow
	\mathcal N\!\left(0,\ \sigma_S^2(t;\bX_0,\bZ_0)\right) \,,
	\]
	and consistency of
	$\widehat\sigma_S^2(t;\bX_0,\bZ_0)=\{\widehat S^{\,1,{\sf cf}}(t\mid\bX_0,\bZ_0)\}^2\widehat\sigma_H^2(t;\bX_0,\bZ_0)$
	follows by Slutsky’s theorem together with $\widehat S^{\,1,{\sf cf}}(t\mid\bX_0,\bZ_0)\xrightarrow{p}S_o(t\mid\bX_0,\bZ_0)$.	
\end{proof}


\section{General covariate distributions of {\bX}: sieve approximation of the
Riesz representer}\label{appC} 

When $\bX$ has a general distribution, no closed-form expression is available and
$\psi^0_{t}(\cdot,\cdot;\bX_0,\bZ_0)$ must be approximated in a finite-dimensional sieve. In this section we construct a spline-based approximation
$\widehat\psi^{(-k)}_{t}(\cdot,\cdot;\bX_0,\bZ_0)$ and define a cross-fitted one-step estimator based on the same Riemann grid
$\{t_j\}_{j=0}^{m_n}$ used throughout (with $\Delta t_j=t_j-t_{j-1}$), and we state the corresponding $\sqrt{n}$ asymptotic normality
under additional sieve conditions.

Fix $t\in[0,\tau]$ and a random evaluation point $(\bX_0,\bZ_0)$ drawn independently from the target covariate distribution.
By conditions, $L_t$ is bounded on $(\mathcal H,\|\cdot\|_{L^2(\PP_0,\Delta)})$, $\mathcal{H}=\overline{\mathcal{T}}_{g_o}$, and there exists a unique
$\psi^0_{t}(\cdot,\cdot;\bX_0,\bZ_0)\in \mathcal H$ solving the Riesz equation
\[
\langle \psi^0_{t}(\cdot,\cdot;\bX_0,\bZ_0),\, b\rangle_0 = L_t(b),\qquad \forall\, b\in \mathcal H.
\]
Let $\{\mu_r(u,\bX):r=1,\ldots,d_n\}$ be a spline basis on $[0,\tau]\times\mathcal X$ (e.g., tensor-product $B$-splines in $u$
and each coordinate of $\bX$ after rescaling $\bX$ to $[0,1]^d$), and define the sieve space
\[
\mathcal H_{d_n}=\mathrm{span}\{\mu_1,\ldots,\mu_{d_n}\}\subset \mathcal H.
\]
Define the population sieve-Riesz projection $\psi^{0,(d_n)}_{t}(\cdot,\cdot;\bX_0,\bZ_0)\in \mathcal H_{d_n}$ as the unique solution of
\[
\langle \psi^{0,(d_n)}_{t}(\cdot,\cdot;\bX_0,\bZ_0),\, b\rangle_0 = L_t(b),\qquad \forall\, b\in \mathcal H_{d_n}.
\]
The key additional requirement for $\sqrt{n}$ inference is that the sieve approximation bias is negligible, namely,
\begin{equation}\label{eq:sieve_bias_generalX}
	\big\|\psi^{0,(d_n)}_{t}(\cdot,\cdot;\bX_0,\bZ_0)-\psi^0_{t}(\cdot,\cdot;\bX_0,\bZ_0)\big\|_{L^2(\PP_0,\Delta)}=o(n^{-1/2}) \, .
\end{equation}
It might be tempting to assume that for each $t\in[0,\tau]$, the Riesz representer $\psi^0_{t}(\cdot,\cdot;\bX_0,\bZ_0)$ (viewed as a function of
$(u,\bX)\in[0,\tau]\times\mathcal X$) belongs to a Sobolev (or H\"older) ball of smoothness $s$, uniformly over
$t\in[0,\tau]$ (and uniformly over $(\bX_0,\bZ_0)$ in its support), so that the spline sieve $\mathcal H_{d_n}$ satisfies the
approximation property \eqref{eq:sieve_bias_generalX}. Clearly, this condition is separate from the composite H\"older assumption on $g_o$, which is used to guarantee
$\|\widehat g-g_o\|_{L^2(\PP_0,\Delta)} = O_p(\gamma_n\log^2 n)$.

In the absence of additional structural restrictions, approximation rates in the norm $\|\cdot\|_{L^2(\PP_0,\Delta)}$
deteriorate rapidly with the dimension of $\bX$ (curse of dimensionality), so that achieving
$o(n^{-1/2})$ accuracy  typically requires sieve dimension $d_n$ that grows so quickly with $n$ that the required $n$ for a given accuracy becomes impractically large.
Throughout the remainder of this section we assume that
\eqref{eq:sieve_bias_generalX} holds.

Now, we provide an estimator of $\psi^{0,(d_n)}_{t}(\cdot,\cdot;\bX_0,\bZ_0)$ under the cross-fitted setting.
For each fold $k$ and functions $f,h\in \mathcal H$, define the fold-wise empirical inner product (computed on $\mathcal I_k^c$) by the Riemann sum
\begin{equation}\label{eq:inner_nk_riemann}
	\langle f,h\rangle_{n,-k}
	=
	\frac{1}{|\mathcal I_k^c|}
	\sum_{i\in \mathcal I_k^c}\sum_{j=1}^{m_n}
	f(t_j,\bX_i)\,h(t_j,\bX_i)\,
	Y_i(t_{j})\,\widehat h^{(-k)}(t_j\mid \bX_i,\bZ_i)\,\Delta t_j.
\end{equation}
Similarly, define the fold-wise empirical linear functional by
\begin{equation}\label{eq:Lhat_nk_riemann}
	\widehat L^{(-k)}_t(b)
	=
	\sum_{j:\,t_j\le t}
	\widehat h^{(-k)}(t_{j}\mid \bX_0,\bZ_0)\, b(t_{j},\bX_0)\,\Delta t_j.
\end{equation}
We estimate the Riesz representer in the sieve space $\mathcal H_{d_n}$ by solving the ridge-stabilized empirical Riesz equation.
Namely, we find $\widehat\psi^{(-k)}_{t}(\cdot,\cdot;\bX_0,\bZ_0)\in \mathcal H_{d_n}$ such that for all $b\in \mathcal H_{d_n}$,
\begin{equation}\label{eq:emp_riesz_generalX_riemann}
	\langle \widehat\psi^{(-k)}_{t}(\cdot,\cdot;\bX_0,\bZ_0),\, b\rangle_{n,-k}
	+\rho_n\,\langle \widehat\psi^{(-k)}_{t}(\cdot,\cdot;\bX_0,\bZ_0),\, b\rangle_{\mathrm{coef}}
	=
	\widehat L^{(-k)}_t(b),
\end{equation}
where $\rho_n\downarrow 0$, $\sqrt{n}\rho_n\to 0$, and $\langle\cdot,\cdot\rangle_{\mathrm{coef}}$ is the Euclidean inner product of spline coefficients. Here $\rho_n>0$ is a vanishing ridge parameter used only to stabilize the inversion of the empirical Gram matrix. Specifically, write $\widehat\psi^{(-k)}_{t}(u,\bX;\bX_0,\bZ_0)=\sum_{r=1}^{d_n}\widehat\alpha^{(-k)}_r\,\mu_r(u,\bX)$ and
$b(u,\bX)=\sum_{s=1}^{d_n}\beta_s\,\mu_s(u,\bX)$.
Let $\widehat G_{-k}\in\mathbb R^{d_n\times d_n}$ and $\widehat\ell_{-k}\in\mathbb R^{d_n}$ be defined such that their components are given by
\[
(\widehat G_{-k})_{rs}=\langle \mu_r,\mu_s\rangle_{n,-k}, \qquad r,s=1,\ldots,d_n.
\]
\[
(\widehat\ell_{-k})_s
=
\widehat L_t^{(-k)}(\mu_s)
=
\sum_{j:\,t_j\le t}
\widehat h^{(-k)}(t_{j}\mid \bX_0,\bZ_0)\,
\mu_s(t_{j},\bX_0)\,
\Delta t_j,
\qquad s=1,\ldots,d_n.
\]
Let
$\widehat\alpha^{(-k)} = \big(\widehat\alpha^{(-k)}_1,\ldots,\widehat\alpha^{(-k)}_{d_n}\big)^\top \in \mathbb R^{d_n}$.
Then \eqref{eq:emp_riesz_generalX_riemann} is equivalent to
\[
(\widehat G_{-k}+\rho_n I_{d_n})\,\widehat\alpha^{(-k)}=\widehat\ell_{-k} \, ,
\]
so that $\widehat\alpha^{(-k)}=(\widehat G_{-k}+\rho_n I_{d_n})^{-1}\widehat\ell_{-k}$,
where $I_{d_n}$ denotes the $d_n\times d_n$ identity matrix.

Finally, we are ready to define the cross-fitted one-step estimator based on the above  spline approximation.
Using the fold-wise spline weight $\widehat\psi^{(-k)}_{t}(\cdot,\cdot;\bX_0,\bZ_0)$, define the fold-wise nuisance influence contribution
in discrete (Riemann-sum) form
\[
\widehat\phi^{(-k)}_{t,g}(W_i;\bX_0,\bZ_0)
=
\sum_{j=1}^{m_n}
\widehat\psi^{(-k)}_{t}(t_{j},\bX_i;\bX_0,\bZ_0)\,\Delta \widehat M^{(-k)}_i(t_j).
\]
Let $\widehat\phi^{(-k)}_{t,\theta}(W_i;\bX_0,\bZ_0)$ denote the parametric IF piece, constructed exactly as in the discrete-support
section.
Set $\widehat\phi^{(-k)}_{t}=\widehat\phi^{(-k)}_{t,g}+\widehat\phi^{(-k)}_{t,\theta}$. The cross-fitted one-step estimator is then
\[
\widehat H^{1,\mathrm{cf}}(t\mid \bX_0,\bZ_0)
=
\widehat H^{\mathrm{cf}}(t\mid \bX_0,\bZ_0)
+
\frac1n\sum_{k=1}^{K_{\mathrm{cf}}}\sum_{i\in \mathcal I_k}\widehat\phi^{(-k)}_{t}(W_i;\bX_0,\bZ_0),
\]
and
$
\widehat S^{1,\mathrm{cf}}(t\mid \bX_0,\bZ_0)=\exp\{-\widehat H^{1,\mathrm{cf}}(t\mid \bX_0,\bZ_0)\}
$.

Relative to the discrete case, the general-$\bX$ setting introduces three additional
sources of error that must be controlled to obtain a $\sqrt n$ expansion:
\begin{enumerate}
	\item \emph{Sieve approximation (bias):} the population sieve-Riesz representer
	$\psi^{0,(d_n)}_t\in\mathcal H_{d_n}$ must approximate $\psi_t^0$ in the $\|\cdot\|_{L^2(\PP_0,\Delta)}$ norm,
	i.e. $\|\psi^{0,(d_n)}_t-\psi_t^0\|_{L^2(\PP_0,\Delta)}$ (cf.\ \eqref{eq:sieve_bias_generalX}).
	\item \emph{Estimation of the sieve-Riesz weights:} the data-driven solution
	$\widehat\psi_t$ must be close to $\psi^{0,(d_n)}_t$, i.e. $\|\widehat\psi_t-\psi^{0,(d_n)}_t\|_{L^2(\PP_0,\Delta)}=o_p(1)$.
	\item \emph{Stability:} the (regularized) Gram matrix defining $\widehat\psi_t$ must be well-conditioned,
	and the effective sieve complexity must grow slowly enough with $n$ so that the associated empirical-process
	remainder terms are $o_p(n^{-1/2})$.
\end{enumerate}

The general-$\bX$ setting is substantially more delicate than the finite-support case. In particular,
a fully rigorous analysis of the sieve-Riesz step would require additional assumptions on the
covariate distribution and on the regularity and approximability of the population Riesz
representer, together with quantitative control of the sieve approximation bias, the
ridge-stabilized Gram system, and the empirical-process terms arising from estimation of the
Riesz coefficients. A complete investigation of these issues for the specific sieve-Riesz
implementation considered here would take us beyond the intended scope of the present paper.
Accordingly, the following theorem is best viewed as a high-level conditional asymptotic result for the
general-$\bX$ setting. Even in this form, however, the result remains valuable, as it helps
clarify the ingredients that would be needed to extend the one-step asymptotic theory beyond the
finite-support case. From a practical point of view, Theorem~\ref{thm:generalX-CLT} is already highly relevant, since when $\bX$ is handled through the DNN nuisance-estimation step, discretization or grouping of
covariates is often not expected to materially compromise prediction accuracy.


\vspace{0.5cm}
\begin{restatable}{theorem}{generalX-CLT}
	\label{thm:generalX-CLT}
	Fix $t\in[0,\tau]$ and let $(\bX_0,\bZ_0)$ be an independent draw from the covariate distribution,
	independent of the training sample. Assume Assumptions A1--A4  hold, so that the hazard is uniformly bounded.
	Let $\{t_j\}_{j=0}^{m_n}$ be a deterministic grid with $0=t_0<\cdots<t_{m_n}=t$ and mesh
	$\delta_n=\max_j(t_j-t_{j-1})$. Assume $\sqrt{n}\,\delta_n\to0$ and the time-grid Riemann approximation error is $O(\delta_n)$ uniformly over a neighborhood
	of $\eta_o=(\btheta_o,g_o)$.  
	Suppose the conditions of Lemma~\ref{lem:EIF_existence_MS} hold and $\EE_0\{\phi_t^0(\bW;\bX_0,\bZ_0)^2\}<\infty$.
	Suppose moreover that for each fold $k$,
	\begin{enumerate}
		\item[(i)]  $\|\widetilde g^{(-k)}-g_o\|_{L^2(\PP_0,\Delta)}=o_p(n^{-1/4})$ (up to polylogarithmic factors) and, uniformly in $k$,
			\[
		\sqrt{|\mathcal I_k^c|}\,(\widetilde \btheta^{(-k)}-\btheta_o)
		=
		I(\btheta_o)^{-1}\frac{1}{\sqrt{|\mathcal I_k^c|}}\sum_{i\in \mathcal I_k^c}\ell_{\theta_o}^*(\bW_i)
		+o_p(1);
		\]
		\item[(ii)] the sieve bias condition \eqref{eq:sieve_bias_generalX} holds;
		\item[(iii)]  $\|\widehat\psi^{(-k)}_{t}(\cdot,\cdot;\bX_0,\bZ_0)-\psi^{0,(d_n)}_{t}(\cdot,\cdot;\bX_0,\bZ_0)\|_{L^2(\PP_0,\Delta)}=o_p(1)$;
		\item[(iv)] the ridge-stabilized Gram matrix $\widehat G_{-k}+\rho_n I_{d_n}$ is well-conditioned and
		$d_n$ grows slowly enough so that the additional empirical-process terms due to estimating $\widehat\psi$ are $o_p(n^{-1/2})$. The projection and grid-approximation condition
\eqref{eq:projection-grid-control} is also assumed.
	\end{enumerate}
	Then,  conditional on the independent draw $(\bX_0,\bZ_0)$,
	\[
	\sqrt{n}\Big\{\widehat H^{1,\mathrm{cf}}(t\mid \bX_0,\bZ_0)-H_o(t\mid \bX_0,\bZ_0)\Big\} \mid (\bX_0,\bZ_0)
	\Rightarrow
	N\!\Big(0,\sigma_H^2(t;\bX_0,\bZ_0)\Big),
	\]
	where $\sigma_H^2(t;\bX_0,\bZ_0)=\Var_0\!\big(\phi_t^0(W;\bX_0,\bZ_0)\mid \bX_0,\bZ_0\big)$.
	Consequently,
	\[
	\sqrt{n}\Big\{\widehat S^{1,\mathrm{cf}}(t\mid \bX_0,\bZ_0)-S_o(t\mid \bX_0,\bZ_0)\Big\} \mid (\bX_0,\bZ_0)
	\Rightarrow
	N\!\Big(0,\sigma_S^2(t;\bX_0,\bZ_0)\Big),
	\]
	and
	$
	\sigma_S^2(t;\bX_0,\bZ_0)=S_o(t\mid \bX_0,\bZ_0)^2\,\sigma_H^2(t;\bX_0,\bZ_0)
	$.
\end{restatable}

Theorem~\ref{thm:generalX-CLT} should be read as a conditional asymptotic-distribution result under assumptions (i)–(iv), where assumption (iv) collects the high-level control needed for the sieve-Riesz estimation step, including the associated empirical-process and projection-approximation errors.

\begin{proof}[Proof of \Cref{thm:generalX-CLT}]
	We suppress $(t,\bX_0,\bZ_0)$ from the notation when no confusion arises.
	Let $\{\mathcal I_k\}_{k=1}^K$ be the fixed $K$--fold partition and let $\mathcal F_{-k}$ denote the $\sigma$--field generated by
	the training subsample $\{\bW_j:j\in \mathcal I_k^c\}$, the fold split, and $(\bX_0,\bZ_0)$.
	By construction, conditional on $\mathcal F_{-k}$, the observations $\{\bW_i:i\in \mathcal I_k\}$ are i.i.d.\ from $\PP_0$.
	
	Throughout this proof, $L_{t,n}$ and $c_{0,t,n}$ denote the grid-based analogues of $L_t$ and
	$c_{0,t}$, respectively. Under $\sqrt n\,\delta_n\to0$, the discrepancy between the continuous and
	grid-based versions is $o(n^{-1/2})$ and is absorbed into the remainder notation.
	
	Let $\widetilde H_o$ be the Riemann approximation of $H_o$ on the grid, so that
	$\|\widetilde H_o-H_o\|\lesssim \delta_n$ and hence $\sqrt n(\widetilde H_o-H_o)=o(1)$ whenever $\sqrt n\,\delta_n\to 0$.
	Write
	\[
	\widehat H^{1,{\sf cf}} - H_o
	=
	\underbrace{\big(\widehat H^{{\sf cf}}-\widetilde H_o\big)}_{A_n}
	+
	\underbrace{\big(\widetilde H_o-H_o\big)}_{o(n^{-1/2})}
	+
	\underbrace{\mathbb P_n\widehat\phi_t}_{B_n},
	\]
	where $\widehat H^{{\sf cf}}=\sum_{k=1}^{K_{\mathrm{cf}}} \frac{|\mathcal I_k|}{n} \widehat H^{(-k)}$ and
	$\mathbb P_n\widehat\phi_t=n^{-1}\sum_{k=1}^{K_{\mathrm{cf}}}\sum_{i\in \mathcal I_k}\widehat\phi_t^{(-k)}(\bW_i)$.
	
	As in Theorem~\ref{thm:cf_discrete_CLT}, a first-order exponential expansion yields
	\[
	A_n
	=
	\sum_{k=1}^{K_{\mathrm{cf}}} \frac{|\mathcal I_k|}{n} \sum_{j:t_j\le t} \Delta t_j\,h_o(t_j\mid\bX_0,\bZ_0)
	\Big\{ \delta g^{(-k)}(t_j,\bX_0) + (\delta\btheta^{(-k)})^\top \bZ_0 \Big\}
	+ R_{n,1},
	\]
	with $\sqrt n\,R_{n,1}=o_p(1)$ under assumption~(i), the same second-order remainder control as in
	Theorem~\ref{thm:cf_discrete_CLT}.
	
	Let $\phi_t^0(\bW)$ denote the population IF from Lemma~\ref{lem:EIF_existence_MS}.
	Since $\phi_t^0$ is an influence function, $\PP_0\!\big(\phi_t^0(\bW)\mid \bX_0,\bZ_0\big)=0$, hence
	$\mathbb P_n\phi_t^0=(\mathbb P_n-\PP_0)\phi_t^0$.
	Insert and subtract $\mathbb P_n\phi_t^0$ to obtain
	\[
	\sqrt n\big(\widehat H^{1,{\sf cf}}-H_o\big)
	=
	\underbrace{\sqrt n\,(\mathbb P_n-\PP_0)\phi_t^0}_{\text{CLT term}}
	+
	\sqrt n\Big\{ A_n + \mathbb P_n\big(\widehat\phi_t-\phi_t^0\big)\Big\}
	+ o_p(1).
	\]
	By conditional Lindeberg--Feller (given $(\bX_0,\bZ_0)$) and $\EE_0[(\phi_t^0)^2\mid\bX_0,\bZ_0]<\infty$,
	\[
	\sqrt n\,(\mathbb P_n-\PP_0)\phi_t^0 \xrightarrow{d}
	N\!\left(0,\sigma_H^2(t;\bX_0,\bZ_0)\right),
	\qquad
	\sigma_H^2(t;\bX_0,\bZ_0)=\Var_0\!\big(\phi_t^0(\bW)\mid \bX_0,\bZ_0\big).
	\]
	Thus it remains to show
	\[
	\sqrt n\Big\{ A_n + \mathbb P_n\big(\widehat\phi_t-\phi_t^0\big)\Big\}=o_p(1).
	\]
	
Decomposing $\widehat\phi_t-\phi_t^0$ into the nuisance-Riesz part and the $\btheta$-part gives
	\[
	\widehat\phi_t^{(-k)}-\phi_t^0
	=
	\big(\widehat\phi_{t,g}^{(-k)}-\phi_{t,g}^0\big)
	+
	\big(\widehat\phi_{t,\theta}^{(-k)}-\phi_{t,\theta}^0\big).
	\]
By the same first-order expansion used in the discrete setting  for the parametric IF component,
\[
\sum_{k=1}^{K_{\mathrm{cf}}} \frac{|\mathcal I_k|}{n}
\Big\{
c_{0,t,n}(\bX_0,\bZ_0)^\top\delta\btheta^{(-k)}
-
T_{n,1,\theta}^{(-k)}
\Big\}
+
\mathbb P_n\!\big(\hat\phi_{t,\theta}-\phi_{t,\theta}^0\big)
=
o_p(n^{-1/2}),
\]
by Lemmas~\ref{lem:plugin_ellstar}--\ref{lem:inverse-info-plug-in} and assumption~(i).

Now we focus on the nuisance-Riesz part, where the new ingredient is the sieve estimation of $\psi_t^0$.
Write
\begin{eqnarray}\label{eq:phig-dif}
	\widehat\phi_{t,g}^{(-k)}-\phi_{t,g}^0
	=
	\sum_{j=1}^{m_n}\big(\widehat\psi_j^{(-k)}-\psi_j^0\big)\,\Delta M_{0,j}
	\;+\;
	\sum_{j=1}^{m_n}\widehat\psi_j^{(-k)}\big(\Delta \widehat M_j^{(-k)}-\Delta M_{0,j}\big) \, .
\end{eqnarray}
We start with the second sum of the right-hand side of (\ref{eq:phig-dif}).
Set
\[
B_{n,1}^{(g)}
=
\frac{1}{n}\sum_{k=1}^{K_{\mathrm{cf}}}\sum_{i\in \mathcal I_k}
\sum_{j=1}^{m_n}
\widehat\psi_t^{(-k)}(t_j,\bX_i;\bX_0,\bZ_0)\,
\big(\Delta \widehat M_{ij}^{(-k)}-\Delta M_{0,ij}\big),
\]
so that
\[
\mathbb P_n\big(\widehat\phi_{t,g}-\phi_{t,g}^0\big)
=
B_{n,0}^{(g)}+B_{n,1}^{(g)},
\]
where
\[
B_{n,0}^{(g)}
=
\sum_{k=1}^{K_{\mathrm{cf}}} \frac{|\mathcal I_k|}{n}
\left\{ \frac{1}{|\mathcal I_k|}
\sum_{i\in \mathcal I_k}
\sum_{j=1}^{m_n}
\big(\widehat\psi_t^{(-k)}(t_j,\bX_i;\bX_0,\bZ_0)-\psi_t^0(t_j,\bX_i;\bX_0,\bZ_0)\big)\,\Delta M_{0,ij}
\right\}.
\]
We first analyze \(B_{n,1}^{(g)}\). Since
$
\Delta \widehat M_{ij}^{(-k)}-\Delta M_{0,ij}
=
-\big(\Delta \widehat A_{ij}^{(-k)}-\Delta A_{0,ij}\big)$,
we have
\[
B_{n,1}^{(g)}
=
-\frac{1}{n}\sum_{k=1}^{K_{\mathrm{cf}}}\sum_{i\in \mathcal I_k}\sum_{j=1}^{m_n}
\widehat\psi_t^{(-k)}(t_j,\bX_i;\bX_0,\bZ_0)\,
Y_i(t_{j})
\big\{\widehat h^{(-k)}(t_j\mid \bX_i,\bZ_i)-h_o(t_j\mid \bX_i,\bZ_i)\big\}\,\Delta t_j,
\]
and 
$
B_{n,1}^{(g)}
= \sum_{k=1}^{K_{\mathrm{cf}}} \frac{|\mathcal I_k|}{n} \ B_{n,1}^{(g,k)}
$ where
$$
B_{n,1}^{(g,k)}
=
\frac{1}{|\mathcal I_k|}\sum_{i\in \mathcal I_k}\sum_{j=1}^{m_n}
\widehat\psi_t^{(-k)}(t_j,\bX_i;\bX_0,\bZ_0)\,
\big(\Delta \widehat M_{ij}^{(-k)}-\Delta M_{0,ij}\big) \, .
$$

Conditional on \(\mathcal F_{-k}\), the observations \(\{\bW_i:i\in \mathcal I_k\}\) are i.i.d., hence
\[
\EE_0\!\big(B_{n,1}^{(g,k)}\mid \mathcal F_{-k}\big)
=
-
\sum_{j=1}^{m_n}
\EE_0\!\left[
\widehat\psi_t^{(-k)}(t_j,\bX;\bX_0,\bZ_0)\,
Y(t_{j})
\big\{\widehat h^{(-k)}(t_j\mid \bX,\bZ)-h_o(t_j\mid \bX,\bZ)\big\}
\Bigm| \mathcal F_{-k}
\right]\Delta t_j.
\]

Now write
\[
\widehat h^{(-k)}(t_j\mid \bX,\bZ)-h_o(t_j\mid \bX,\bZ)
=
h_o(t_j\mid \bX,\bZ)\,r_j^{(-k)}(\bX,\bZ)
+
R_{\lambda,j}^{(-k)}(\bX,\bZ),
\]
where
\[
r_j^{(-k)}(\bX,\bZ)
=
\delta g^{(-k)}(t_j,\bX)+(\delta\btheta^{(-k)})^\top \bZ,
\]
and the exponential-link remainder satisfies
\[
|R_{\lambda,j}^{(-k)}(\bX,\bZ)|
\lesssim
h_o(t_j\mid \bX,\bZ)\,
\big(r_j^{(-k)}(\bX,\bZ)\big)^2
\]
with probability tending to one, uniformly in \(j\) and \(k\).
Let \(\bg^*(u,\bX)\) denote the \(L^2(\PP_0,\Delta)\)-projection of \(\bZ\) onto \(\mathcal H\), as in Lemma~\ref{lem:EIF_existence_MS}, and define
\[
b_j^{(-k)}(\bX)
=
\delta g^{(-k)}(t_j,\bX)+(\delta\btheta^{(-k)})^\top \bg^*(t_j,\bX).
\]
Then
\[
r_j^{(-k)}(\bX,\bZ)
=
b_j^{(-k)}(\bX)
+
(\delta\btheta^{(-k)})^\top\!\big\{\bZ-\bg^*(t_j,\bX)\big\}.
\]
Accordingly,
\[
\EE_0\!\big(B_{n,1}^{(g,k)}\mid \mathcal F_{-k}\big)
=
-
\Big\{
T_{n,1,g}^{(-k)}+T_{n,1,\theta}^{(-k)}+R_{n,2}^{(-k)}
\Big\},
\]
where
\[
T_{n,1,g}^{(-k)}
=
\sum_{j=1}^{m_n}
\EE_0\!\left[
\widehat\psi_t^{(-k)}(t_j,\bX;\bX_0,\bZ_0)\,
Y(t_{j})h_o(t_j\mid \bX,\bZ)\,
b_j^{(-k)}(\bX)
\Bigm| \mathcal F_{-k}
\right]\Delta t_j,
\]
and
\[
T_{n,1,\theta}^{(-k)}
=
\sum_{j=1}^{m_n}
\EE_0\!\left[
\widehat\psi_t^{(-k)}(t_j,\bX;\bX_0,\bZ_0)\,
Y(t_{j})h_o(t_j\mid \bX,\bZ)\,
(\delta\btheta^{(-k)})^\top\!\big\{\bZ-\bg^*(t_j,\bX)\big\}
\Bigm| \mathcal F_{-k}
\right]\Delta t_j,
\]
and
\[
|R_{n,2}^{(-k)}|
\lesssim
\sum_{j=1}^{m_n}
\EE_0\!\left[
\big|\widehat\psi_t^{(-k)}(t_j,\bX;\bX_0,\bZ_0)\big|\,
Y(t_{j})h_o(t_j\mid \bX,\bZ)\,
\big(r_j^{(-k)}(\bX,\bZ)\big)^2
\Bigm| \mathcal F_{-k}
\right]\Delta t_j .
\]

The term \(T_{n,1,g}^{(-k)}\) is the nuisance-space part, since \(b^{(-k)}\) is a function of \((u,\bX)\) only.
The term \(T_{n,1,\theta}^{(-k)}\) corresponds to the residual parametric direction generated by
\(\bZ-\bg^*(u,\bX)\); this will match the first-order parametric contribution in
\(\mathbb P_n(\widehat\phi_{t,\theta}-\phi_{t,\theta}^0)\).
For functions $f$ and $f'$, define the grid-based population
inner product
\[
\langle f,f'\rangle_{0,n}
=
\EE_0\!\left[
\sum_{j=1}^{m_n}
f(t_j,\bX)f'(t_j,\bX)
Y(t_{j})h_o(t_j\mid\bX,\bZ)\,\Delta t_j
\right].
\]
For training-dependent functions, the expectation is taken
over an independent observation $\bW\sim\PP_0$, holding
$\mathcal F_{-k}$ fixed.
Here $b^{(-k)}(t_j,\bX)=b_j^{(-k)}(\bX)$.
Let $\Pi_{d_n}b^{(-k)}$ denote the orthogonal projection
onto $\mathcal H_{d_n}$ with respect to
$\langle\cdot,\cdot\rangle_{0,n}$, characterized by
\[
\left\langle
b^{(-k)}-\Pi_{d_n}b^{(-k)},v
\right\rangle_{0,n}
=0,
\qquad v\in\mathcal H_{d_n}.
\]
We next identify the leading nuisance term $T_{n,1,g}^{(-k)}$ with the evaluation functional
applied to $b^{(-k)}$. Write
\[
T_{n,1,g}^{(-k)}
=
\Big\langle \hat\psi_t^{(-k)},\,\Pi_{d_n} b^{(-k)} \Big\rangle_{0,n}
+
R_{n,3}^{(-k)},
\]
where $R_{n,3}^{(-k)}=0$ by the defining orthogonality
of $\Pi_{d_n}$ and the inclusion
$\widehat\psi_t^{(-k)}\in\mathcal H_{d_n}$.
Next,
\[
\Big\langle \hat\psi_t^{(-k)},\,\Pi_{d_n} b^{(-k)} \Big\rangle_{0,n}
=
\Big\langle \psi_t^{0,(d_n)},\,\Pi_{d_n} b^{(-k)} \Big\rangle_0
+
R_{n,4}^{(-k)},
\]
where $R_{n,4}^{(-k)}=o_p(n^{-1/2})$ uniformly in $k$ by assumption~(iv), which absorbs the
additional empirical-process terms arising from estimating $\hat\psi_t^{(-k)}$ and replacing the
empirical inner product by its population version.
Finally, by the defining property of the population sieve-Riesz representer,
\[
\Big\langle \psi_t^{0,(d_n)},\,\Pi_{d_n} b^{(-k)} \Big\rangle_0
=
L_{t}\!\big(\Pi_{d_n} b^{(-k)}\big).
\]
Hence
\[
T_{n,1,g}^{(-k)}
=
L_{t}\!\big(\Pi_{d_n} b^{(-k)}\big)
+
R_{n,3}^{(-k)}+R_{n,4}^{(-k)}.
\]

It remains to remove the sieve projection and pass to the
grid-based evaluation functional. The high-level projection
and discretization control required here is
\begin{equation}
\label{eq:projection-grid-control}
\max_{1\le k\le K}
\left|
L_t\!\left(\Pi_{d_n}b^{(-k)}\right)
-
L_{t,n}\!\left(b^{(-k)}\right)
\right|
=o_p(n^{-1/2}).
\end{equation}
Therefore
\[
T_{n,1,g}^{(-k)} = L_{t,n}\!\big(b^{(-k)}\big)+o_p(n^{-1/2}) \, ,
\]
uniformly in \(k\).

On the other hand, by the definition of
\[
c_{0,t}(\bX_0,\bZ_0)
=
\bZ_0 H_o(t\mid \bX_0,\bZ_0)-L_t(\bg^*),
\]
and its Riemann approximation \(c_{0,t,n}(\bX_0,\bZ_0)\), we have
\[
\sum_{j:t_j\le t}
h_o(t_j\mid \bX_0,\bZ_0)\,
(\delta\btheta^{(-k)})^\top\!\big\{\bZ_0-\bg^*(t_j,\bX_0)\big\}\,\Delta t_j
=
\big(c_{0,t,n}(\bX_0,\bZ_0)\big)^\top \delta\btheta^{(-k)},
\]
with $c_{0,t,n}(\bX_0,\bZ_0)-c_{0,t}(\bX_0,\bZ_0)=O(\delta_n)=o(n^{-1/2})$.
Thus the first-order expansion of \(A_n\) can be written as a nuisance component
\(L_{t,n}(b^{(-k)})\) plus the parametric component
\(c_{0,t,n}(\bX_0,\bZ_0)^\top\delta\btheta^{(-k)}\).
The latter is canceled by the corresponding first-order contribution from
\(\mathbb P_n(\widehat\phi_{t,\theta}-\phi_{t,\theta}^0)\), exactly as in
the discrete setting.

Now rewrite the first-order expansion of \(A_n\) as
\[
A_n
=
\sum_{k=1}^{K_{\mathrm{cf}}} \frac{|\mathcal I_k|}{n}
\left[
L_{t,n}\!\big(b^{(-k)}\big)
+
c_{0,t,n}(\bX_0,\bZ_0)^\top \delta\btheta^{(-k)}
\right]
+R_{n,1},
\]
where
\[
b_j^{(-k)}(\bX)
=
\delta g^{(-k)}(t_j,\bX)+(\delta\btheta^{(-k)})^\top \bg^*(t_j,\bX),
\]
and
\[
c_{0,t,n}(\bX_0,\bZ_0)
=
\sum_{j:t_j\le t}
h_o(t_j\mid \bX_0,\bZ_0)\,
\big\{\bZ_0-\bg^*(t_j,\bX_0)\big\}\,\Delta t_j,
\qquad
c_{0,t,n}(\bX_0,\bZ_0)=c_{0,t}(\bX_0,\bZ_0)+o(1).
\]

Combining the previous display with
\[
T_{n,1,g}^{(-k)} = L_{t,n}\!\big(b^{(-k)}\big)+o_p(n^{-1/2}),
\]
we see that the nuisance-space contribution in
\[
A_n
=
\sum_{k=1}^{K_{\mathrm{cf}}} \frac{|\mathcal I_k|}{n}
\Big\{
L_{t,n}\!\big(b^{(-k)}\big)
+
c_{0,t,n}(\bX_0,\bZ_0)^\top\delta\btheta^{(-k)}
\Big\}
+
R_{n,1}
\]
cancels fold-by-fold with the term
\[
-\sum_{k=1}^{K_{\mathrm{cf}}} \frac{|\mathcal I_k|}{n}T_{n,1,g}^{(-k)}
\]
appearing in $\sum_k \frac{|\mathcal I_k|}{n}\EE_0(B_{n,1}^{(g,k)}\mid\mathcal F_{-k})$.
Therefore,
\[
A_n+\sum_{k=1}^{K_{\mathrm{cf}}} \frac{|\mathcal I_k|}{n}\EE_0\!\big(B_{n,1}^{(g,k)}\mid\mathcal F_{-k}\big)
=
\sum_{k=1}^{K_{\mathrm{cf}}} \frac{|\mathcal I_k|}{n}
\Big\{
c_{0,t,n}(\bX_0,\bZ_0)^\top\delta\btheta^{(-k)}
-
T_{n,1,\theta}^{(-k)}
\Big\}
+
R_{n,1}
+
\sum_{k=1}^{K_{\mathrm{cf}}} \frac{|\mathcal I_k|}{n}\widetilde R_{n,2}^{(-k)},
\]
where $\widetilde R_{n,2}^{(-k)}=o_p(n^{-1/2})$ uniformly in $k$.

By the same first-order expansion used earlier for the parametric IF component,
the weighted sum
\[
\sum_{k=1}^{K_{\mathrm{cf}}} \frac{|\mathcal I_k|}{n}
\Big\{
c_{0,t,n}(\bX_0,\bZ_0)^\top\delta\btheta^{(-k)}
-
T_{n,1,\theta}^{(-k)}
\Big\}
+
\mathbb P_n\big(\widehat\phi_{t,\theta}-\phi_{t,\theta}^0\big)
=
o_p(n^{-1/2}) \, .
\]
 Consequently,
\[
A_n
+
\sum_{k=1}^{K_{\mathrm{cf}}} \frac{|\mathcal I_k|}{n}\EE_0\!\big(B_{n,1}^{(g,k)}\mid\mathcal F_{-k}\big)
+
\mathbb P_n\big(\widehat\phi_{t,\theta}-\phi_{t,\theta}^0\big)
=
o_p(n^{-1/2}),
\]
because assumption~(i) also implies
\[
\sqrt n\,\max_k\|\delta g^{(-k)}\|_{L^2(\PP_0,\Delta)}^2=o_p(1),
\qquad
\sqrt n\,\max_k\|\delta\btheta^{(-k)}\|^2=o_p(1).
\]

It remains to control the centered empirical fluctuation
\[
\sum_{k=1}^{K_{\mathrm{cf}}} \frac{|\mathcal I_k|}{n}
\left\{
B_{n,1}^{(g,k)}-\EE_0\!\big(B_{n,1}^{(g,k)}\mid \mathcal F_{-k}\big)
\right\},
\]
which is \(o_p(n^{-1/2})\) by conditional independence across \(i\in \mathcal I_k\), bounded conditional second moments, and assumption~(iv).  Here the dependence on $\widehat\psi_t^{(-k)}$ is handled entirely through assumption (iv), conditional on $\mathcal F_{-k}$.
Hence
\[
A_n+B_{n,1}^{(g)}
=
-\mathbb P_n\big(\widehat\phi_{t,\theta}-\phi_{t,\theta}^0\big)
+o_p(n^{-1/2}).
\]

It remains to control \(B_{n,0}^{(g)}\), which reflects sieve-Riesz approximation and estimation error.
Let $\psi_t^{0,(d_n)}\in\mathcal H_{d_n}$ be the population sieve-Riesz representer and write
	\[
	\widehat\psi_t^{(-k)}-\psi_t^0
	=
	\underbrace{\big(\psi_t^{0,(d_n)}-\psi_t^0\big)}_{\text{sieve bias}}
	+
	\underbrace{\big(\widehat\psi_t^{(-k)}-\psi_t^{0,(d_n)}\big)}_{\text{estimation error}}.
	\]
Using the triangle inequality,
\[
\|\widehat\psi_t^{(-k)}-\psi_t^0\|_{L^2(\PP_0,\Delta)}
\le
\|\psi_t^{0,(d_n)}-\psi_t^0\|_{L^2(\PP_0,\Delta)}
+
\|\widehat\psi_t^{(-k)}-\psi_t^{0,(d_n)}\|_{L^2(\PP_0,\Delta)}.
\]
By assumptions~(ii) and~(iii),
\[
\max_{1\le k\le K}\|\widehat\psi_t^{(-k)}-\psi_t^0\|_{L^2(\PP_0,\Delta)}=o_p(1).
\]	
Now, conditional on $\mathcal F_{-k}$ we can apply the discrete-time martingale isometry to obtain
	\begin{align*}
		&\EE_0\!\left[
		\left(
		\sum_{j=1}^{m_n}\big(\widehat\psi_t^{(-k)}(t_j,\bX)-\psi_t^0(t_j,\bX)\big)\,\Delta M_{0,j}
		\right)^2
		\Bigm|\mathcal F_{-k}
		\right] \\
		&\qquad=
		\EE_0\!\left[
		\sum_{j=1}^{m_n}\big(\widehat\psi_t^{(-k)}(t_j,\bX)-\psi_t^0(t_j,\bX)\big)^2\,\Delta A_{0,j}
		\Bigm|\mathcal F_{-k}
		\right].
	\end{align*}
	Conditional on \(\mathcal F_{-k}\), the summands
	\[
	\sum_{j=1}^{m_n}
	(\widehat\psi_t^{(-k)}(t_j,\bX_i)-\psi_t^0(t_j,\bX_i))\,\Delta M_{0,ij},
	\qquad i\in \mathcal I_k,
	\]
	are independent and centered. Hence, by the discrete-time martingale isometry,
	\[
	\Var_0\!\left(
	\frac{1}{|\mathcal I_k|}\sum_{i\in \mathcal I_k}\sum_{j=1}^{m_n}
	(\widehat\psi_t^{(-k)}(t_j,\bX_i)-\psi_t^0(t_j,\bX_i))\,\Delta M_{0,ij}
	\Bigm|\mathcal F_{-k}
	\right)
	\lesssim
	\frac{1}{|\mathcal I_k|}\|\widehat\psi_t^{(-k)}-\psi_t^0\|_{L^2(\PP_0,\Delta)}^2.
	\]
	Therefore,
	\[
	\frac{1}{|\mathcal I_k|}\sum_{i\in \mathcal I_k}\sum_{j=1}^{m_n}
	(\widehat\psi_t^{(-k)}(t_j,\bX_i)-\psi_t^0(t_j,\bX_i))\,\Delta M_{0,ij}
	=
	O_p\!\left(
	|\mathcal I_k|^{-1/2}\|\widehat\psi_t^{(-k)}-\psi_t^0\|_{L^2(\PP_0,\Delta)}
	\right).
	\]
	Since \(K\) is fixed and \(|\mathcal I_k|\asymp n\),
	\[
	B_{n,0}^{(g)}
	=
	O_p\!\left(
	n^{-1/2}\max_{1\le k\le K}\|\widehat\psi_t^{(-k)}-\psi_t^0\|_{L^2(\PP_0,\Delta)}
	\right)
	\]	
	and hence
	\[
	B_{n,0}^{(g)}=o_p(n^{-1/2}).
	\]
	
	Combining this with the previously established bound
	\[
	A_n+B_{n,1}^{(g)}+\mathbb P_n\big(\widehat\phi_{t,\theta}-\phi_{t,\theta}^0\big)=o_p(n^{-1/2}),
	\]
	and with
	\[
	B_{n,0}^{(g)}=o_p(n^{-1/2}),
	\]
	we obtain
	\[
	A_n+\mathbb P_n\big(\widehat\phi_t-\phi_t^0\big)=o_p(n^{-1/2}).
	\]
	and therefore
	\[
	\sqrt n\big(\widehat H^{1,{\sf cf}}(t\mid \bX_0,\bZ_0)-H_o(t\mid \bX_0,\bZ_0)\big)
	\xrightarrow{d}
	N\!\left(0,\sigma_H^2(t;\bX_0,\bZ_0)\right).
	\]
	The result for $\widehat S^{1,{\sf cf}}=\exp(-\widehat H^{1,{\sf cf}})$ follows by the delta method with derivative $-S_o$.
\end{proof}

\end{appendix}

\end{spacing}
\end{document}